\documentclass[english, 11pt]{article}
\usepackage{mathpazo}
\usepackage[T1]{fontenc}
\usepackage[latin9]{inputenc}
\usepackage{geometry}
\usepackage{babel}
\usepackage{float}

\usepackage{amsmath}
\usepackage{amsthm}
\usepackage{natbib}
\usepackage{amssymb}
\usepackage{amsfonts}
\usepackage{url}
\usepackage{algorithm}
\usepackage[noend]{algorithmic}
\usepackage{color}
\usepackage{booktabs} 
\usepackage{multirow} 
\usepackage{makecell}

\usepackage[unicode=true,pdfusetitle, bookmarks=true,bookmarksnumbered=false,bookmarksopen=false,breaklinks=true,pdfborder={0 0 1},backref=false,colorlinks=true,citecolor=blue]{hyperref}

\newcommand{\blue}[1]{\textcolor{blue}{{#1}}}

\global\long\def\mS{\mathcal{S}}%
\global\long\def\mA{\mathcal{A}}%
\global\long\def\mG{\mathcal{G}}%
\global\long\def\mE{\mathcal{E}}%
\global\long\def\VR{V^{\mathrm{ref}}}
\global\long\def\E{\mathbb{E}}%
\global\long\def\R{\mathbb{R}}%
\global\long\def\otilde{\widetilde{O}}%
\global\long\def\VR{V^{\mathrm{ref}}}

\newcommand{\remind}[1]{{\color{blue}#1}}

\newcommand{\mix}{t_{\text{mix}}}
\newcommand{\hit}{t_{\text{hit}}}
\DeclareMathOperator{\spn}{sp}

\newtheorem{proposition}{\protect\propositionname}
\newtheorem{lemma}{\protect\lemmaname}
\newtheorem{theorem}{\protect\theoremname}

\usepackage{appendix,breakurl}
\PassOptionsToPackage{hyphens}{url}

\makeatother

\providecommand{\lemmaname}{Lemma}
\providecommand{\propositionname}{Proposition}
\providecommand{\remarkname}{Remark}
\providecommand{\theoremname}{Theorem}
\providecommand{\corollaryname}{Corollary}

\begin{document}

\title{Sharper 
Model-free Reinforcement Learning  for Average-reward Markov Decision Processes}

\author{Zihan Zhang$^\dagger$, Qiaomin Xie$^\mathsection$ 
\footnote{Accepted for presentation at the Conference on Learning Theory (COLT) 2023. Emails: \texttt{zsubfunc@outlook.com}, \texttt{qiaomin.xie@wisc.edu}}\\ ~\\
	\normalsize $^\dagger$ Princeton University\\
	\normalsize $^\mathsection$ University of Wisconsin-Madison
}

\date{}

\maketitle

\begin{abstract}%

   We develop several provably efficient model-free reinforcement learning (RL) algorithms for  infinite-horizon average-reward Markov Decision Processes (MDPs). We consider both online setting and the setting with access to a simulator. In the online setting, we propose model-free RL algorithms based on reference-advantage decomposition. Our algorithm achieves $\widetilde{O}(S^5A^2\mathrm{sp}(h^*)\sqrt{T})$ regret after $T$ steps, where $S\times A$ is the size of state-action space, and 
   $\mathrm{sp}(h^*)$ the span of the optimal bias function. Our results are the first to achieve optimal dependence in $T$ for weakly communicating MDPs. 
   In the simulator setting, we propose a model-free RL algorithm that finds an $\epsilon$-optimal policy using $\widetilde{O} \left(\frac{SA\mathrm{sp}^2(h^*)}{\epsilon^2}+\frac{S^2A\mathrm{sp}(h^*)}{\epsilon} \right)$ samples, whereas the minimax lower bound is $\Omega\left(\frac{SA\mathrm{sp}(h^*)}{\epsilon^2}\right)$. 
   Our results are based on two new techniques that are unique in the average-reward setting: 1) better discounted approximation by value-difference estimation; 2) efficient construction of confidence region for the optimal bias function with space complexity $O(SA)$.

\end{abstract}

\section{Introduction}

Reinforcement learning (RL) has emerged as a paradigm for solving challenging sequential decision-making problems and recently led to remarkable breakthroughs across various application domains. 
RL algorithms are typically developed under the Markov Decision Process (MDP) framework and 
can be broadly classified into two categories. \emph{Model-based} methods find a policy by first estimating the MDP model using data. These algorithms store the estimated probability transition kernel, resulting in high space complexity. \emph{Model-free} methods directly estimate the value function or policy, which are more memory-efficient and can readily incorporate function approximation.
Moreover, for both episodic MDPs  and infinite-horizon discounted-reward MDPs, various model-free algorithms (e.g., optimistic Q-learning, \citealt{jin2018q, dong2019q}) have been shown to achieve near-optimal statistical efficiency in terms of regret or sample complexity.

\emph{Infinite-horizon average-reward} MDP formulation is more appropriate for applications that involve continuing
operations not divided into episodes; examples include resource allocation in data centers and congestion control in Internet. In contrast to episodic/discounted MDPs, theoretical understanding of model-free algorithm is relatively inadequate for the average-reward setting.
Most existing work focuses on model-based methods, such as UCRL2~\citep{jaksch2010near}, PSRL~\citep{ouyang2017learning}, SCAL~\citep{fruit2018efficient} and EULER~\citep{zanette2019tighter}. 

In this paper, we study model-free algorithms for average-reward MDPs, whose reward function and transition kernel are unknown to the agent. We consider both \emph{online} and \emph{simulator} settings. In online setting, the agent interacts with the unknown system sequentially and learns from the observed trajectory, with the goal of minimizing the cumulative regret during the learning process. A fundamental challenge in online RL is to balance exploration and exploitation.
In the simulator setting, the agent has  access to a generative model/simulator and can sample the next state and reward for any state-action pair~\citep{kakade2003thesis,kearns2002sparse}. The goal is to find an approximate optimal policy using a small number of samples. 

In the online setting, existing model-based algorithms achieve a regret that nearly matches the minimax lower bound $\Omega(\sqrt{DSAT})$ \citep{jaksch2010near}, where $D$ is the diameter of the MDP, $S$ and $A$ are the size of state and spaces, respectively, and $T$ is the number of steps. On the other hand, the model-free algorithm POLITEX~\citep{abbasi2019exploration,abbasi2019politex} only achieves $\widetilde{O}(T^{3/4})$ regret for the restricted class of ergodic MDPs, where all policies mix at the same speed.\footnote{Throughout this paper, we use the notation $\widetilde{O}(\cdot)$ to suppress constant and logarithmic factors.} Using an optimistic mirror descent algorithm,  \cite{wei2020model_free} achieve an improved regret $\widetilde{O}(\sqrt{T})$ for ergodic MDPs. Under the less restricted assumption of weakly communicating MDPs---which is necessary for low regret~\citep{bartlett2009regal}---the best existing regret bound of model-free algorithms is $\widetilde{O}(T^{2/3})$, achieved by optimistic Q-learning~\citep{wei2020model_free}. It remains unknown whether a model-free algorithm can achieve the optimal $\widetilde{O}(\sqrt{T})$ regret for this general setting.

For the simulator setting, existing work mainly focuses on ergodic MDPs and the results typically depend on $\mix,$ the \emph{worst-case} mixing time induced by a policy \citep{wang2017primal,jin2020efficiently,jin2021towards}. For weakly communicating MDP, the recent work by \cite{wang2022average} proves a minimax lower bound $\Omega(SA\spn(h^*)\epsilon^{-2})$ on the sample complexity of finding an $\epsilon$-optimal policy. Here $\spn(h^*)$ is the span of the optimal bias function and is always a lower bound of the diameter $D$ and the mixing time $\mix.$ They also establish an upper bound $\otilde(SA\spn(h^*)\epsilon^{-3})$. Their approach is based on reduction to a discounted MDP and using an existing model-based discounted MDP algorithm~\citep{li2020breaking}. For model-free algorithms, it is  unclear what guarantees can be achieved.

With the above gaps in existing literature, a fundamental question is left open: Can we design computationally efficient model-free algorithms that achieve the optimal regret and sample complexity for weakly communicating MDPs?

\paragraph{Contributions.} 
We make progress towards answering the above question. For both the online and simulator settings, we develop model-free algorithms that enjoy improved guarantees compared with existing results. Our algorithms are truly model-free, in the sense of only incurring $O(SA)$ space complexity. 
It is well-acknowledged that the average-reward setting poses several fundamental challenges that are absent in episodic/discounted MDPs. Addressing these challenges requires several new algorithmic and analytical ideas, which we summarize below.
\begin{itemize}
    \setlength{\itemsep}{1pt}%
    \setlength{\parskip}{0pt}%
    \item In the online setting, we design a model-free algorithm, $\mathtt{UCB-AVG}$, based on an optimistic variant of variance-reduced Q-learning. We show that $\mathtt{UCB-AVG}$ achieves a regret bound $\otilde(S^5A^2\spn(h^*)\sqrt{ T})$, which is optimal in $T$ (up to log factors) in view of the minimax lower bound in~\cite{jaksch2010near}. To the best of our knowledge, our algorithm is the first computationally efficient model-free method with $\otilde(\sqrt{T})$ regret for weakly communicating MDPs.
    \item In the simulator setting, we adapt the idea of $\mathtt{UCB-AVG}$ to develop a model-free algorithm that finds an $\epsilon$-optimal policy with sample complexity $\otilde\left( {SA\spn^2(h^*)\epsilon^{-2}} + S^2A\spn(h^*)\epsilon^{-1} \right).$ This sample complexity is near-optimal for weakly communicating MDPs, in view of the minimax lower bound established in~\cite{wang2022average}.
    \item On the technical side, our approach integrates two key ideas: learning an $\gamma$-discounted MDP as an approximation, and leveraging reference-advantage decomposition for variance reduction~\citep{zhang2020almost} in optimistic Q-learning. As recognized in prior work, a naive approximation by discounted MDPs results in suboptimal guarantees. A distinguishing feature of our method is maintaining estimates of \emph{value-difference} between state pairs to provide a sharper bound on the variance of reference advantage. We also crucially use a careful choice of the discounted factor $\gamma$ to balance approximation error due to discounting and the statistical learning error, and we are able to maintain a good-quality reference value function using a small space complexity. We discuss our techniques in greater details in Section~\ref{sec:technique}.
\end{itemize}

\subsection{Additional Related Work} 
In the discussion below, we primarily focus on RL algorithms in the tabular setting with finite state and action spaces, which are the closest to our work.

\paragraph{Regret analysis of model-based RL for average-reward MDPs.} For the online setting, \cite{bartlett2009regal} and \cite{jaksch2010near} are the first works that achieve sub-linear regret with model-based algorithms.
In the last few years, many model-based algorithms have been proposed to further improve the regret guarantee, including the works of~\cite{ouyang2017learning,ortner2020regret, zhang2019bias,fruit2018efficient,fruit2019improved}. These algorithms first estimate the unknown model using observed data, and then perform planning based on the fitted model. All these algorithms achieve $\otilde(\sqrt{T})$ regret, but they often suffer from sub-optimal dependence on other parameters or computational inefficiency. See Table~\ref{tab:regret_comparison_all}
for a detailed comparison.

\paragraph{Regret analysis of model-free RL for average-reward MDPs.} For the tabular setting, 
\cite{wei2020model_free} provide two model-free algorithms: Optimistic Q-learning achieves $\otilde(T^{2/3})$ regret for weakly communicating MDPs; MDP-OOMD, which is based on the idea of optimistic online mirror descent, achieves improved regret bound of $\otilde(\sqrt{T})$ for the more restricted class of ergodic MDPs. 
For the setting with function approximation, the first model-free algorithm is the POLITEX algorithm~\citep{abbasi2019politex}, which is a variant of regularized policy iteration. POLITEX is shown to achieve $\otilde(T^{3/4})$ regret in ergodic MDPs with high probability. The work by~\cite{hao2021adaptive} presents an adaptive approximate policy iteration (AAPI) algorithm, which improves these results to $\otilde(T^{2/3})$.  
\cite{lazic2021improved} provide an improved analysis of POLITEX, showing that it attains $O(\sqrt{T})$ high-probability regret bound. 
The work by~\cite{wei2021linear} considers linear function approximation and proposes a model-free algorithm MDP-EXP2 that achieves $\otilde(\sqrt{T})$ expected regret for ergodic MDPs. It is noteworthy that they also provide two model-based algorithms: FOPO achieves $\otilde(\sqrt{T})$ regret but is computationally inefficient; another computationally efficient variant OLSVI.FH obtains $\otilde(T^{3/4})$ regret. 

\begin{table}[t]\label{tab:regret_comparison_all}
\caption{Comparisons of existing regret bounds for RL algorithms in infinite-horizon average-reward MDPs with $S$ states, $A$ actions, and $T$ steps. $D$ is the diameter of the MDP, $\spn(h^*)\leq D$ is the span of the optimal value function, 
$\mix$ is the \emph{worst-case} mixing time,
$\hit$ is the hitting time (i.e., maximum inverse stationary state visitation probability under any policy), and $\eta \leq \hit$ is some distribution mismatch coefficient. 
}
\begin{center}
\renewcommand{\arraystretch}{1.3}
\footnotesize
\begin{tabular}{ c|c|c|c }
\hline
 & \textbf{Algorithm} & \textbf{Regret} & \textbf{Comment} \\ \hline
\multirow{9}{*}{Model-based} & REGAL {\small\citep{bartlett2009regal}} & $\otilde(\spn(h^*)\sqrt{SAT})$ & no efficient implementation \\
 & UCRL2 {\small\citep{jaksch2010near}} & $\otilde(DS\sqrt{AT})$ & - \\
 & PSRL {\small\citep{ouyang2017learning}}& $\otilde(\spn(h^*)S\sqrt{AT})$ & Bayesian regret \\
& OSP {\small\citep{ortner2020regret}}& $\otilde(\sqrt{\mix SAT})$ & \makecell[l]{ergodic assumption and \\ no efficient implementation } \\
 & SCAL {\small\citep{fruit2018efficient}} & $\otilde(\spn(h^*)S\sqrt{AT})$ & -\\
 & UCRL2B {\small\citep{fruit2019improved}} & $\otilde(S\sqrt{DAT})$ & -\\
 &  EBF {\small\citep{zhang2019bias}}& $\otilde(\sqrt{DSAT})$ & no efficient implementation
\\   & FOPO {\small\citep{wei2021linear}}& $\otilde(\spn(h^*)\sqrt{S^3A^3 T})$ & no efficient implementation \\
   & OLSVI.FH {\small\citep{wei2021linear}}& $\otilde(\sqrt{\spn(h^*)}(SAT)^{\frac{3}{4}})$ & -\\
   \hline
\multirow{4}{*}{Model-free} 
 & Politex {\small\citep{abbasi2019politex}} & $\mix^3 \hit \sqrt{SA}T^\frac{3}{4}$ & uniform mixing \\
 & Optimistic Q-learning {\small\citep{wei2020model_free}} & $\otilde(\spn(h^*)(SA)^{\frac{1}{3}}T^{\frac{2}{3}})$ &- \\
 & MDP-OOMD {\small\citep{wei2020model_free}}& $\otilde(\sqrt{\mix^3 \eta AT})$ & ergodic assumption \\
  & AAPI {\small\citep{hao2021adaptive}}& $\otilde(\mix^2(\eta SA)^{\frac{1}{3}}T^{\frac{2}{3}})$ & uniform mixing \\
    & MDP-EXP2 {\small\citep{wei2021linear}}& $\otilde(SA\sqrt{\mix^3 T})$ & uniform mixing \\
   & \textbf{UCB-AVG} {(\small \textbf{this work})} & $\tilde{O}(S^5A^2\mathrm{sp}(h^*)\sqrt{T})$& - \\
   \hline
 & lower bound {\small\citep{jaksch2010near}} & $\Omega(\sqrt{DSAT})$ & - \\ \hline
\end{tabular}
\end{center}
\end{table}

\paragraph{PAC bounds with access to a simulator.} When it comes to the case with access to a generative model, there have been extensive developments on the finite-sample analysis and probably approximately correct (PAC) bounds in the last several years. For $\gamma$-discounted infinite-horizon MDPs, a sequence of work establishes the sample complexity of vanilla Q-learning~\citep{even2003learning,wainwright2019stochastic,chen2020finite}. The work by~\cite{li2021q} provides a tight analysis of Q-learning, showing that $\otilde(\frac{SA}{(1-\gamma)^4\epsilon^2})$ suffices for Q-learning to yield an entrywise $\epsilon$-approximation of the optimal Q-function. On the other hand, the minimax lower limit for discounted MDP scale on the order of $\Omega(\frac{SA}{(1-\gamma)^3\epsilon^2})$~\citep{azar2013minimax}. 
For average-reward MDPs, existing work only concerns a restricted class of ergodic MDPs. \cite{wang2017primal} develops a model-free method that makes primal-dual updates to the policy and value vectors. The proposed algorithm finds an $\epsilon$-optimal policy with sample complexity $\otilde(\tau^2 SA \mix^2\epsilon^{-2}),$ where $\mix$ is the mixing time for the ergodic MDP, $\tau$ is an ergodicity parameter that requires an additional ergodic assumption. The follow-up work by~\cite{jin2020efficiently} improves the sample complexity to $\otilde(SA \mix^2\epsilon^{-2})$, by developing a similar primal-dual stochastic mirror descent algorithm. They also establish a lower bound $\Omega(SA \mix \epsilon^{-2})$, and provide a model-based algorithm with sample complexity $\otilde(SA\mix \epsilon^{-3})$ \citep{jin2021towards}. The recent work \citep{li2022stochastic} develops a variant of stochastic policy mirror descent algorithm that achieves $\otilde(SA\mix^3\epsilon^{-2})$ overall sample complexity for uniform mixing MDPs. 
See Table~\ref{tab:sample_complexity_comparison} for a summary. We remark that the diameter $D$ and $\mix$ are both lower bounded by $\spn(h^*)$, and $\mix$ can be arbitrarily large for certain MDPs \citep{wang2022average}.

\begin{table}[h]
\caption{Comparison of  sample complexity bounds of RL algorithms finding an $\epsilon$-optimal policy $\pi$ with $\rho^\pi \geq \rho^*-\epsilon$ for average-reward MDP with $S$ states, $A$ actions and a simulator. 
}\label{tab:sample_complexity_comparison}
\begin{center}
\footnotesize
\begin{tabular}{ c|c|c }
\hline
 \textbf{Algorithm} & \textbf{Sample complexity} & \textbf{Comment} \\ \hline
 Primal-dual method {\small\citep{wang2017primal}} & $\otilde(\tau^2 SA\mix^2 \epsilon^{-2})$ & ergodic assumption \\
 primal-dual SMD {\small\citep{jin2020efficiently}} & $\otilde(SA\mix^2 \epsilon^{-2})$ & uniform mixing \\
  Discounted MDP {\small\citep{jin2021towards}} & $\otilde(SA\mix \epsilon^{-3})$ & uniform mixing \\
      Policy mirror descent {\small\citep{li2022stochastic}} & $\otilde(SA\mix^3 \epsilon^{-2})$ & uniform mixing \\
  Discounted MDP {\small\citep{wang2022average}} & $\otilde(SA\spn(h^*)\epsilon^{-3})$ & - \\ 
   \textbf{Monotone Q-learning} {(\small \textbf{this work})} & $\otilde(S^2A\spn^2(h^*)\epsilon^{-2})$ & - \\
      \textbf{Refined Q-learning} {(\small \textbf{this work})} & $\otilde(SA\spn^2(h^*)\epsilon^{-2})$ & - \\
 \hline
  Lower bound {\small\citep{jin2021towards}} & $\Omega(SA\mix \epsilon^{-2})$ & uniform mixing \\ \hline
   Lower bound {\small\citep{wang2022average}} & $\Omega(SA \spn(h^*) \epsilon^{-2})$ & -\\ \hline
\end{tabular}
\end{center}
\end{table}

\section{Preliminaries} \label{sec:preliminaries}

We consider an \emph{infinite-horizon average-reward} MDP $(\mS,\mA,P,r)$, where  $\mS$ is the
state space, $\mA$ is the action space, $P:\mS\times\mA\times\mS\rightarrow[0,1]$
is the transition kernel, and $r:\mS\times\mA\rightarrow [0,1]$
is the reward function. We assume that $\mS$ and $\mA$ are finite with $S=|\mS|$ and $A=|\mA|$. We write the state space as $\mathcal{S}=\{s^1,s^2,\ldots, s^{S}\}.$ At each time step $t$,  the agent observes a state $s_{t}\in\mS$, takes an action $a_{t}\in\mA$ and receives a reward $r(s_{t},a_{t})$. 
The system then transitions to a new state $s_{t+1}\sim P(\cdot|s_t,a_t)$.
A stationary policy $\pi:\mS\rightarrow\Delta(\mA)$ maps each state
to a distribution over actions. 
The \emph{average reward} of $\pi$ starting from state $s_{0}$ is defined as 
\begin{align*}
 \rho^{\pi}(s_0):=\lim_{T\rightarrow\infty}\frac{1}{T}\E_{\pi}\bigg[\sum_{t=0}^{T-1}r(s_{t},a_{t})\mid s_{0}\bigg],  
\vspace{-0.15in}
\end{align*}
where the expectation $\E_{\pi}$ is with respect to $a_{t}\sim\pi(\cdot| s_{t})$ and $s_{t+1}\sim P(\cdot|s_{t},a_{t})$.
When the Markov chain induced by $\pi$
is ergodic, the average reward $\rho^{\pi}(s_0)$ is independent of initial state $s_0$ and denoted by $\rho^{\pi}$. 
The (relative) value function $h^{\pi}$ (also called bias function)
is defined as 
$h^{\pi}(s):=\E_{\pi}\left[\sum_{t=0}^{\infty}\left(r(s_{t},a_{t})-\rho^{\pi}\right)\mid s_{0}=s\right],$ 
which satisfies the Bellman equation
\begin{align}
\vspace{-0.1in}
\rho^{\pi}+h^{\pi}(s)=\E_{a\sim\pi(\cdot|s),s'\sim P(\cdot|s,a)}\left[r(s,a)+h^{\pi}(s')\right],  \quad \forall s \in \mS.
\label{eq:Bellman}
\end{align} 

It is known that weakly communicating assumption is necessary for learning average-reward MDPs with low
regret \citep{bartlett2009regal}. Under this assumption, there exists an optimal average
reward $\rho^{*}=\max_{\pi}\rho^{\pi}$ independent of the
initial state \citep{puterman2014markov}.  Let $\pi^{*}$ be the corresponding optimal policy and  $h^*$ be the bias function under $\pi^{*}.$ Define the span of $h^*$ as $\spn(h^*)=\max_s h^*(s)-\min_s h^*(s)$, which is bounded for weakly communicating MDPs~\citep{lattimore2018bandit}. 
For simplicity, we assume that $\spn(h^*)$ is known.

We consider two settings. In the online setting, the agent executes the MDP sequentially for $T$ steps, collecting $T$ samples in total. 
A common performance metric is the \emph{cumulative regret}:
	\begin{align}
	\mathrm{R}(T):= \sum_{t=1}^T \big[ \rho^*-r(s_t,a_t) \big],\label{eq:regret}
	\end{align}
which measures the difference between the total rewards of the optimal policy and that of an online learning algorithm. 
In the setting with a generative model/simulator, one can obtain independent samples of the reward and transition for any given state-action pair. We say a policy $\pi$ is $\epsilon$-optimal if it satisfies $\rho^{\pi}\geq \rho^*-\epsilon $. A standard performance metric is the sample complexity, which is the number of samples needed for finding an $\epsilon$-optimal policy.

\paragraph{Notation.}  Define $\mathbb{V}(p,x):=p^{\top}x^2 - (p^{\top}x)^2$ for two vectors $p,x$ of the same dimension, where $x^2$ is the entry-wise square of $x.$ That is, $\mathbb{V}(p,x)$ is the variance of a random variable $X$ with distribution $p$ on the support $x.$ Let $P_{s,a}=P(\cdot|s,a)\in \R^{1\times S}$ denote the transition probability vector given the state-action pair $(s,a),$ and let $P_{s,a,s'}=P(s'|s,a).$ Define $\spn(V)=\max_{s}V(s)-\min_{s}V(s)$ as the span of a vector $V$. We use $xy$ as a shorthand of the inner product $x^{\top}y$. In particular, $PV$ denotes $P^{\top}V$ for $P\in \Delta^{S}$ and $V\in \mathbb{R}^{S}$ through out the paper, where $\Delta^{S}$ is the $S$-dimensional simplex. 
We also use define $\textbf{1}$ to denote the $S$-dimensional all-one vector.

\section{Overview of Techniques} \label{sec:technique}

In this section, we provide an overview of our approach and technical contributions.

\smallskip
\noindent{\textbf{Reduction to a Proper Discounted MDP.}} 
Consider an $\gamma$-discounted MDP (DMDP) with the same model $(\mS,\mA,P,r)$ as the average-reward MDP (AMDP). Since an AMDP is the limit of DMDP as $\gamma \rightarrow 1$, one may learn a DMDP and use it as an approximation of the AMDP.  
We denote the optimal value and Q functions of the $\gamma$-discounted MDP by $V^*$ and $Q^*$, which satisfy the Bellman equations:
\begin{align*}
    Q^*(s,a) &= r(s,a) +\gamma \E_{s'\sim P_{sa}} [V^*(s')], \forall s\in \mS,\forall a\in\mA; \mbox{ and }
    V^*(s)= \max_{a\in \mA} Q^*(s,a), \forall s\in \mS.
\end{align*}
The lemma below bounds the difference between $V^*$ of the DMDP and $\rho^*$ of the AMDP.
\vspace{-0.05in}
\begin{lemma} [Lemma 2 of \cite{wei2020model_free} ]\label{lemma:discount_approximate}
For any $\gamma \in [0,1)$, the optimal value function $V^*$ of the $\gamma$-discounted MDP satisfies:
(\romannumeral1) $\spn(V^*)\leq 2\spn(h^*)$;  (\romannumeral2) $\forall s\in \mS,$ $\big|V^*(s)-\frac{\rho^*}{1-\gamma}\big|\leq \spn(h^*)$.
\end{lemma}
\vspace{-0.05in}

The choice of $\gamma$ is crucial for approximation quality. In the online setting, the regret due to $\gamma$-discounted approximation is $\Theta(T(1-\gamma)\mathrm{sp}(h^*))$, since the discounted approximation is roughly equivalent to moving $(1-\gamma)$-portion of transition probability to a null state. To obtain $\otilde(\sqrt{T})$ regret, one must choose a large $\gamma$ such that $1/{(1-\gamma)} = \Omega(\sqrt{T})$. 
However, the regret bounds of existing model-free DMDP algorithms have a high-order dependence on $1/(1-\gamma)$, which prevents one to choose $1/{(1-\gamma)} = \Omega(\sqrt{T})$. For example, \cite{wei2020model_free} derived a regret bound $\otilde\big( \sqrt{{T}/{(1-\gamma)}}+{1}/{(1-\gamma)}+ T(1-\gamma) \big)$ by applying an existing discounted Q-learning algorithm; optimally choosing  
 ${1}/({1-\gamma}) = \Theta(T^{1/3})$ leads to an $\otilde(T^{2/3})$ regret.
To achieve the optimal regret of $\otilde(\sqrt{T})$, one must use a large $1/(1-\gamma)$, which in turn requires developing novel algorithms and analysis for $\gamma$-discounted MDP. Similar arguments apply to the simulator setting.

\smallskip
\noindent{\textbf{Variance Reduction by Reference-advantage Decomposition.}}
Our algorithms are based on variance-reduced optimistic $Q$-learning algorithms. In particular, we carefully maintain a reference value function $V^{\mathrm{ref}}$ that approximates the optimal discounted value function $V^*$, and use $V^{\mathrm{ref}}$ for variance reduction. 
We divide the learning process into epochs and keep $V^{\mathrm{ref}}$ fixed in each epoch. 
Let $V$ be the value function induced by the estimated $Q$ function. In $Q$-learning, one uses estimate of $P_{s,a}V$ to update $Q$-function. We decompose $P_{s,a}V$ into $P_{s,a}V^{\mathrm{ref}}$ (reference) and $P_{s,a}(V-V^{\mathrm{ref}})$ (advantage). Since $V^{\mathrm{ref}}$ is fixed during the current epoch, we can estimate the reference $P_{s,a}V^{\mathrm{ref}}$ using all samples of $(s,a)$ in this epoch. On the other hand, to estimate the advantage $P_{s,a}(V - V^{\mathrm{ref}})$, only a smaller number of samples is available, leading to a large exploration bonus. We overcome this by designing variance-aware exploration bonus and developing sharper control of the variance $\{\mathbb{V}(P_{s,a},V-V^{\mathrm{ref}})\}_{(s,a)}$, as elaborated below.

As mentioned above, a bottleneck for applying existing algorithms for finite-horizon or discounted MDPs (e.g., those in ~\citep{sidford2018near,zhang2020almost,li2021q,zhang2020model}) is the
high-order dependence on $H:={1}/{(1-\gamma)}$. 
In these algorithms, $\mathbb{V}(P_{s,a},V-V^{\mathrm{ref}})$ is typically bounded by $\sum_{s'}P_{s,a,s'}\left(V(s')-V^{\mathrm{ref}}(s')\right)^2=O(\|V^{\mathrm{ref}}-V\|^2_{\infty})$; for example, \cite{zhang2020almost} and \cite{paper:model_free_breaking} have the bound $\|V^{\mathrm{ref}}-V^*\|_{\infty}=O(1)$.
However, as we hope to choose $H = \Theta(\sqrt{T})$, canceling the $H$ factor in the regret bound requires  $\|V^{\mathrm{ref}}-V\|_\infty = o(1)$, which is impossible in general. 
To see this, simply consider an MDP with 1 state,  1 action,  
and a Bernoulli$(\mu)$ reward distribution, for which the value function is $V(s) = H{\mu}$. By standard minimax arguments, after $T$ steps, the best possible estimation error for $V(s)$ is $\tilde{\Theta}(H{\sqrt{\mu/T}})=\tilde{\Theta}(1)$.

We achieve a sharper control of the variance $\mathbb{V}(P_{s,a},V-V^{\mathrm{ref}})$ by a novel choice of  $V^{\mathrm{ref}}$  using the idea of value-difference estimation. Below we explain our approach for the two settings.

\subsection{Online Setting}

Our key observation is that only the bias differences $(h^*(s)-h^*(s'))$ matter for finding the optimal average-reward policy. Therefore, instead of bounding the absolute error $\| V^{\mathrm{ref}} - V\|_\infty$, we control the differences of the gap $V^{\mathrm{ref}}-V^*$ among difference states.
To this end, we maintain a reference function
$V^{\mathrm{ref}}$ and a group of error thresholds $\mathcal{W}:=\{w_{(s,s')}\}$ satisfying 
\begin{align}
\vspace{-0.15in}
\big| V^{\mathrm{ref}}(s)- V^{\mathrm{ref}}(s')   - ( V^*(s)-V^*(s')) \big|\leq w_{(s,s')},\quad\forall (s,s')\in \mathcal{S}\times \mathcal{S}.\label{eq:lst2}
\vspace{-0.15in}
\end{align} 
Accordingly, we can construct a confidence region for the optimal value function $V^*$:
\begin{align*}
\vspace{-0.15in}
\mathcal{C}:&= \big\{v\in \mathbb{R}^S | \mathrm{sp}(v)\leq 2\mathrm{sp}(h^*) ,\big| V^{\mathrm{ref}}(s)- V^{\mathrm{ref}}(s')   - ( v(s)-v(s')) \big|\leq w_{(s,s')},\forall (s,s')\in \mathcal{S}\times \mathcal{S} \big\}
\vspace{-0.15in}
\end{align*}
The result below shows that we can efficiently project to $\mathcal{C}$.
\begin{proposition}
\label{pro:proj} 
Given $\mathcal{C}$ defined above, there exists a projection operator $\mathrm{Proj}_{\mathcal{C}}:\mathbb{R}^{S}\to \mathcal{C}$ satisfying that: (1) $\mathrm{Proj}_{\mathcal{C}}(v)=v$ for $v\in \mathcal{C}$; (2)
for any $v_1\geq v_2$, $\mathrm{Proj}_{\mathcal{C}}(v_1)\geq \mathrm{Proj}_{\mathcal{C}}(v_2)$.  Moreover, the computation cost of $\mathrm{Proj}_{\mathcal{C}}$ is $\otilde(S^2)$.
\end{proposition}
See Appendix~\ref{sec:proj} for the details of  $\mathrm{Proj}_{\mathcal{C}}$.
By carefully designing the update rule, we ensure that the estimate value function $V$ is always in $\mathcal{C}$.
With $V\in \mathcal{C}$, we have
\begin{align}
\mathbb{V}(P_{s,a},V-V^{\mathrm{ref}})\leq 4\sum_{s'}P_{s,a,s'}\omega^2_{(s,s')}.\nonumber
\end{align}

To design the reference value function $\VR$ with 
 well-bounded $\{w_{(s,s')}\}$, we are inspired by the method developed in \cite{zhang2019bias} for estimating the bias differences $(h^*(s)-h^*(s'))$ of average-reward MDP from a trajectory. 
Using similar idea, we can design an algorithm to estimate $(V^*(s)-V^*(s'))$.
We name such an estimator as $\emph{value-difference}$ estimator for $(s,s')$. 

\smallskip
\noindent{\textbf{Space Efficiency via Graph Update.}} Suppose we have $O(S^2)$ space to maintain the $\emph{value-difference}$ estimator for each $(s,s')$. One can use induction argument to show that regret in each epoch is bounded by  $\otilde(\sqrt{T}).$
To make the algorithm truly \emph{model-free} with reduced space complexity $O(SA)$ (assuming $S\gg A$), we further note that a group of \emph{value-difference} estimators could be represented as a graph $\mathcal{G}:=(\mathcal{V},\mathcal{E})$ with $\mathcal{V}=\mathcal{S}$, $\mathcal{E}$ being the set of state-state pairs with estimators, and the error bounds $\{\omega_{(s,s')}\}_{(s,s')\in \mathcal{E}}$ being edge weights. We have the simple observation by graph theory:
For any connected $\mathcal{G}$ with non-negative weights, there exists a tree $\mathcal{T}$ be a subgraph of $\mathcal{G}$, such that $d_{(s,s')}(\mathcal{T})$ is at most $d_{(s,s')}(\mathcal{G})$ for any $(s,s')$, where $d_{(s,s')}(\mathcal{G}')$ denotes the distance from $s$ to $s'$ in the graph $\mathcal{G}'$. Following this idea, we can re-construct $V^{\mathrm{ref}}$ and error bounds $\{\tilde{w}_{(s,s')}\}_{(s,s')\in \mathcal{S}\times \mathcal{S}}$ using only $O(S)$ \emph{value-difference} estimators such that $\tilde{w}_{(s,s')}= O(Sw_{(s,s')})$
 for any $(s,s')\in \mathcal{S}\times \mathcal{S}$. Indeed, one can construct $\mathcal{T}$ by recursively finding a circle and deleting the edge with maximal weight in this circle.

How to choose the \emph{value-difference} estimators? 
Obviously, we should select the estimator for  each  $(s,s')$ in a uniform way to avoid missing some $(s,s')$ with large visitation count and small weight $\omega_{(s,s')}$. Based on this idea, we maintain the \emph{value-difference} estimator of  $(s,u_{\mathrm{target}}(s,a))$ for each $(s,a)$, and let $u_{\mathrm{target}}(s,a)$ go over the state space recurrently as the count of $(s,a)$ increases.

\subsection{Simulator Setting}
For the setting with a generative model,
we bound the variance of gap $\mathbb{V}(P_{s,a},V-V^{\mathrm{ref}})$ for a group of well-conditioned value functions. Suppose that $V_1,V_2\in \mathbb{R}^{S}$ satisfy that
\begin{align}
 & \max_{a} \left(r(s,a) + \gamma P_{s,a}V_1 \right)\leq V_1(s)\leq \max_{a}\left( r(s,a) + \gamma P_{s,a}V_1 + \beta_1(s,a) \right), \nonumber
 \\ & \max_{a} \left(r(s,a) + \gamma P_{s,a}V_2 \right)\leq V_2(s)\leq \max_{a}\left( r(s,a) + \gamma P_{s,a}V_2 + \beta_2(s,a) \right) ,\nonumber
 \\ & V_1(s)\geq  V_2(s)\geq V^*(s) , \forall s,\label{eq:21022x}
\end{align}
where $\beta_i(s,a)\geq 0$ represents the Bellman error of $V_i$ for $i=1,2$. Now we can bound the discounted total variance of the gap $V_1-V_2$ following policy $\pi$:
\begin{align}
 \mathbb{E}_{\pi}\left[\sum_{t=0}^{\infty} \gamma^{t}\mathbb{V}(P_{s_t,a_t}, V_1 - V_2)\right] 
 & \leq \|V_1-V^*\|_{\infty}^2 + 2\left(V_1(s_0)- V^{\pi}(s_0)\right).\nonumber
\end{align}
Compared with the naive bound $\mathbb{E}_{\pi}\left[\sum_{t=0}^{\infty} \gamma^{t}\mathbb{V}(P_{s_t,a_t}, V_1 - V_2)\right]\leq {\|V_1-V_2\|^2_{\infty}}/{(1-\gamma)}$, we roughly improve a factor of ${1}/{(1-\gamma)}$ when $\|V_1-V_2\|_{\infty}\geq {1}/{\sqrt{1-\gamma}}$. Moreover, in the case $V_2 = V^*$, we save a factor of ${1}/{(1-\gamma)}$ when $\pi$ is closed to an optimal policy. 

On the other hand, we note that the condition 
\begin{align}
\max_{a}\left(r(s,a) + \gamma P_{s,a}V \right)\leq V(s)\leq \max_{a}\left(r(s,a)+ \gamma P_{s,a}V + \beta(s,a)\right) ,\forall (s,a)\label{eq:21022x1}
\end{align}
holds for the value function estimate $V$ under the optimistic  $Q$-learning framework with $\{\beta(s,a)\}$ being the exploration bonus. This allows us to obtain a better sample complexity.

\section{Main Results for the Online Setting}\label{sec:online}

In this section, we present our algorithm and theoretical guarantees for the online setting.

\subsection{Algorithm}

We present our model-free algorithm, $\mathtt{UCB-AVG}$ in Algorithm~\ref{alg:main}. As discussed in Section~\ref{sec:technique}, the overall idea is to learn a $\gamma$-discounted MDP to approximately solve the undiscounted problem. Our algorithm uses a variant of Q-learning with UCB exploration \citep{jin2018q} and reference-advantage decomposition for variance reduction \citep{zhang2019bias}. Specifically, we maintain a reference value function $\VR$, which forms a reasonable estimate of the discounted optimal value function $V^*$. The key novelty of our method is to construct $\VR$ by maintaining \emph{value-difference} estimators for $V^*(s)-V^*(s')$ of selected state pairs $(s,s')$. Moreover, to achieve space efficiency, Algorithm~\ref{alg:main} maintains a graph $\mG=\{\mathcal{S},\mathcal{E}\}$ with associated estimates $\{\Delta_e,\omega_e\}_{e\in \mE}$. With a slight abuse of notation, we embed the estimates into $\mathcal{G}$ by writing $\mathcal{G} = \{\mathcal{S},\mathcal{E}, \{\Delta_e,\omega_e\}_{e\in \mE} \}$.
For each edge $e=(s,s')\in\mE$, we maintain an estimate $\Delta_e$ for the \emph{value difference} $V^*(s)-V^*(s')$ and a bound $\omega_e$ on the estimation error. Given $\mG$, we can construct a reference function $\VR$ accordingly.

\begin{algorithm}[t]
\caption{$\mathtt{UCB-AVG}$}
\begin{algorithmic}[1]\label{alg:main}
\STATE{\textbf{Input:} $S,A,\delta, T$}
\STATE{\textbf{Initialization:} $\iota = \ln(\frac{2}{\delta})$; $H = \sqrt{\frac{T}{300S^6A^2\log_2(T)}\iota}$; $\gamma = \frac{H}{H-1}$; $C=400$; $R = 3600C^2S^6A^2\ln(T)\mathrm{sp}(h^*)H$;}
\STATE{${I}(s,a)\leftarrow 1, {J}(s,a)\leftarrow 0, \forall (s,a)$, $\mathcal{G} = \{\mathcal{S},\mathcal{E}, \{\Delta_e,\omega_e\}_{e\in \mathcal{E}}\}$;}
\FOR{$k=1,2,\ldots$}
\STATE{$\{(\tilde{s},\tilde{a}),\check{\Delta},\check{\omega},\check{N}_2\}\leftarrow \mathtt{UCB-REF}\left(\mathcal{G}, \mathcal{I}=\{I(s,a)\}_{s,a} ,\mathcal{H}= \{2^{J(s,a)}\}_{s,a} \right)$ (Algorithm~\ref{alg:ucbref});\label{line:outer_epoch} }
\STATE{$\mathcal{L}\leftarrow $ the trajectory in Line~\ref{line:outer_epoch};  }
\FOR{$(s,a)\in \mathcal{S}\times \mathcal{A}$\label{line:updatee_0}}
\STATE{  $s'\leftarrow s^{{I}(s,a)}$; \label{line:choose_s}$\{\Delta(s,s'),\omega(s,s'),\chi\}\leftarrow \{\check{\Delta},\check{\omega},\check{N}_2\}$;}
\STATE{$\omega(s,s')\leftarrow \omega(s,s')+\frac{2R}{\chi}$;\label{line:omega}}
\STATE{Update $\mathcal{G}:$ $\mathcal{G}\leftarrow \mathtt{UpdateG}(\mathcal{G},(s,s'),\Delta(s,s'), \omega(s,s'), )$;\label{line:updatee}}
\ENDFOR
\STATE{${I}(\tilde{s},\tilde{a})\leftarrow {I}(\tilde{s},\tilde{a})\text{ mod }S  + 1$;}
\STATE{${J}(\tilde{s},\tilde{a})\leftarrow {J}(\tilde{s},\tilde{a})+1  $ \textbf{if}  ${I}(\tilde{s},\tilde{a})=1$;}
\ENDFOR
\end{algorithmic}
\end{algorithm}

At a high level, the learning process is divided into several epochs, where each epoch is stopped when the number of samples for some state-action pair $(\tilde{s},\tilde{a})$ satisfies certain condition.
In each epoch $k,$ given the current graph  $\mG = \{\mathcal{S},\mathcal{E}\}$ and the \emph{value-difference} estimation $\{\Delta_e,\omega_e\}_{e\in \mE}$,  we invoke a variance-reduced $Q$-learning algorithm $\mathtt{UCB-REF}$. This step is described in Line~\ref{line:outer_epoch} and $\mathtt{UCB-REF}$ is presented in Algorithm~\ref{alg:ucbref}. The algorithm then proceeds to update the graph and corresponding value-difference estimate (lines~\ref{line:updatee_0}-\ref{line:updatee}), based on the bias differences estimation from the $\mathtt{UCB-REF}$ subroutine. We remark that we define and collect $\mathcal{L}$ for ease of presenting the algorithm. The graph update procedure can be implemented efficiently along with the $\mathtt{UCB-REF}$ subroutine. Without loss of generality, we assume the $S$ states in the state space is $\mathcal{S}=\{s^1,s^2,\ldots, s^{S}\}$.
The algorithm also maintains two sets of counters $\mathcal{I}:=\{I(s,a)\}_{s,a}$ and $\mathcal{J}:=\{J(s,a)\}_{s,a}.$ Here ${I}(s,a)\in [S]$ is used to determine the states pair $(s,s^{I(s,a)})$ for estimating their {value difference}; $\mathcal{J}$ are used to track the stopping condition for each epoch. 

Below we provide a detailed description of two key subroutines: Q-learning with UCB exploration and reference advantage ($\mathtt{UCB-REF}$) and the graph update ($\mathtt{UpdateG}$ functions).

\paragraph{Q-learning with UCB exploration and reference advantage ($\mathtt{UCB-REF}$).} The $\mathtt{UCB-REF}$ algorithm (Algorithm~\ref{alg:ucbref}) first constructs the confidence region $\mathcal{C}$ and computes $\VR$ from the input graph $\mG$ 
(Lines~\ref{line:C}-\ref{line:vref}). The reference function $\VR$ then remains fixed. Here the algorithm maintains a running estimate $Q$ (respectively $V$) of the optimal Q-function $Q^*$ (respectively $V^*$). At each step $\tau$, the algorithm takes a greedy action w.r.t.~the current Q estimate (line \ref{line:greedy_Q}). Upon observing a transition $(s,a,r,s'),$ it performs the following update (line \ref{line:online_update_Q}) based on reference-advantage decomposition: 
\begin{align}
     & Q(s,a)\leftarrow 
(1-\alpha_n)Q(s,a)+ \alpha_{n}\Big(  r(s,a)+ \gamma \big(V(s')-V^{\mathrm{ref}}(s')\big) + \gamma {\mu(s,a)}/{n} + b_n(s,a) \Big),
\end{align} 
where $n$ is the number of visits to the current state-action pair $(s,a)$ up to time $\tau,$ and $\alpha_n$ is a carefully chosen stepsize (line \ref{line:alpha_n}). The term ${\mu(s,a)}/{n}$  serves as an estimate of the expected one-step look-ahead value $P_{s,a}\VR$ (line~\ref{line:PVR}), which tends to have a small variance due to the use of batch data. The additional bonus term $b_n$  is used to encourage exploration (line~\ref{line:bonus})---it is typically designed to be an upper confidence bound for the collected terms in the update, so as to ensure that the estimate $Q$ forms an optimistic view of the optimal Q-function $Q^*.$

\begin{algorithm}[ht]
\caption{$\mathtt{UCB-REF}$}
\begin{algorithmic}[1]\label{alg:ucbref}
\STATE{\textbf{Input:} $\mathcal{G}=\{\mathcal{S},\mathcal{E},\{\Delta_e, \omega_{e}\}_{e\in \mathcal{E}}\}$, target index $\mathcal{I}=\{I(s,a)\}_{s,a}$ , stop condition $\mathcal{H}=\{\mathrm{stop}(s,a)\}$;}
\STATE{\textbf{Initialization:} $Q(s,a)\leftarrow  \frac{1}{1-\gamma}$, $V(s)\leftarrow \frac{1}{1-\gamma}, N(s,a)\leftarrow 0,  \forall (s,a)$, $H \leftarrow \frac{1}{1-\gamma}$; }
\STATE{$\mathcal{C}:= \left\{ v\in \mathbb{R}^{S}\mid \mathrm{sp}(v)\leq 2\mathrm{sp}(h^*), |v(s)-v(s')-\Delta_{(s,s')}|\leq \omega_{(s,s')},\forall (s,s')\in \mathcal{E}\right\}$;\label{line:C}}
\STATE{$V^{\mathrm{ref}}\leftarrow \mathrm{Proj}_{\mathcal{C}}\left({(1-\gamma)^{-1}}\mathbf{1}\right)$;\label{line:vref}}
\STATE{$\check{I}(s,a), W(s,a),B(s,a),\check{N}_1(s,a),\check{N}_2(s,a),\check{\Delta}(s,a),\check{\omega}(s,a)\leftarrow 0, \forall (s,a)$;}
\FOR{$\tau=1,2,\ldots$}
\STATE{Observe $s_{\tau}$, take action $a_{\tau}=\arg\max_{a}Q(s_{\tau},a)$, receive $r_{\tau}$ and the next state $s_{\tau+1}$; \label{line:greedy_Q}}
\STATE{//\emph{Update the value functions}}
\STATE{$(s,a,s')\leftarrow (s_{\tau},a_{\tau},s_{\tau+1})$;\quad  $N(s,a)\leftarrow N(s,a)+1$; $n \leftarrow N(s,a)$;}
\STATE{Update the mean estimator $\mu(s,a) \leftarrow  \mu(s,a)+ V^{\mathrm{ref}}(s');$ \label{line:PVR}} 
\STATE{Compute $\tilde{d}_{(s,s')}:= \min_{\mathcal{C}\subset \mathcal{E}, \mathcal{C} \text{ connects } (s,s')}\sum_{e\in \mathcal{C}}\omega_e$;}
\STATE{Update the  estimator $\kappa^2(s,a) \leftarrow 
 \kappa^2(s,a)+ \tilde{d}^2_{(s,s')}$ };
\STATE{$\alpha_{n}\leftarrow \frac{H+1}{H+n};$\label{line:alpha_n}}
\STATE{$b(s,a)\leftarrow 36\sqrt{\frac{H\kappa^2(s,a)\iota}{N^2(s,a)}} +  6\sqrt{\frac{\mathrm{sp}^2(h^*)\iota}{N(s,a)}}  + 38\frac{H\spn(h^*)\iota}{N(s,a)};$ \label{line:bonus}}
\STATE{
     $Q(s,a)\leftarrow (1-\alpha_n)Q(s,a)+ \alpha_{n}\Big(  r(s,a)+ \gamma \big(V(s')-V^{\mathrm{ref}}(s')\big) +\gamma \frac{\mu(s,a)}{N(s,a)} +b(s,a)  \Big)$
\label{line:online_update_Q}
}
\STATE{Update $V(s)\leftarrow  \min\{\max_{a}Q(s,a), V(s)\}$;}
\STATE{Update $V\leftarrow \mathrm{Proj}_{\mathcal{C}}(V)$;}
\STATE{//\emph{Update the accumulators for bias difference estimator} \label{line:ebf}}
\FOR{$(\tilde{s},\tilde{a})\in \mathcal{S}\times \mathcal{A}$}
\STATE{$\tilde{s}'\leftarrow s^{I(\tilde{s},\tilde{a})}$;}
\IF{$\check{I}(\tilde{s},\tilde{a})=0$ and $s=\tilde{s}$}
\STATE{$\check{I}(\tilde{s},\tilde{a})\leftarrow 1$, $\check{N}_2(\tilde{s},\tilde{a})\leftarrow \check{N}_2(\tilde{s},\tilde{a})+1$;}
\ENDIF
\IF{$\check{I}(\tilde{s},\tilde{a})=1$ and $s=\tilde{s'}$}
\STATE{$\check{I}(\tilde{s},\tilde{a})\leftarrow 0$}
\ENDIF
\STATE{$W(\tilde{s},\tilde{a})\leftarrow W(\tilde{s},\tilde{a})+r(s,a)$; $B(\tilde{s},\tilde{a})\leftarrow B(\tilde{s},\tilde{a})+\check{I}(\tilde{s},\tilde{a})\cdot r(s,a)$;}
\STATE{$\check{N}_1(\tilde{s},\tilde{a})\leftarrow \check{N}_1(\tilde{s},\tilde{a})+\check{I}(\tilde{s},\tilde{a})$; $\check{\Delta}(\tilde{s},\tilde{a})\leftarrow  \frac{B(\tilde{s},\tilde{a})-W(\tilde{s},\tilde{a})\check{N}_1(\tilde{s},\tilde{a})/\tau }{\check{N}_2(\tilde{s},\tilde{a})}$;}
\STATE{$\check{\omega}(\tilde{s},\tilde{a})\leftarrow \frac{10\mathrm{sp}(h^*)\sqrt{\tau \iota}+ 4\tau(1-\gamma)\mathrm{sp}(h^*)}{\check{N}_2(\tilde{s},\tilde{a})}$;\label{line:omega1}}
\ENDFOR
\STATE{\textbf{break} if  $N(s,a)=\mathrm{stop}(s,a)$, \textbf{return:} $\{(s,a),\check{\Delta}(s,a),\check{\omega}(s,a), \check{N}_2(s,a)\}$;}
\ENDFOR
\end{algorithmic}
\end{algorithm}

\paragraph{Graph Update.} 
Recall that we use a  graph $\mG=\{\mathcal{S},\mathcal{E},\{\Delta_e,\omega_e\}_{e\in \mE}\}$ 
to maintain estimate of \emph{value-difference} $V^*(s)-V^*(s')$ for certain $\{s,s'\}$ pairs/edges. In particular, during each epoch, for each state-action pair $(s,a)$, we choose one state $s'$ (line \ref{line:choose_s}) and maintain an estimate $\Delta(s,s')$ for $V^*(s)-V^*(s'),$ as well as an upper bound $\omega(s,s')$ on the estimate error. This step is described in lines \ref{line:ebf}-\ref{line:omega1} in Algorithm~\ref{alg:ucbref}. Here we follow the idea developed in~\cite{zhang2019bias}. 
Given the obtained estimate $\{\Delta(s,s'),\omega(s,s')\}$, the algorithm performs an update on the graph $\mG$ by calling $\mathtt{UpdateG}$ subroutine (Algorithm~\ref{alg:udtg}): if $(s,s')$ is on the graph, we simply update its associated values; Otherwise, a new edge corresponding to $(s,s')$ is added to $\mE.$ To maintain the tree structure of $\mG,$ we identify a cycle on the graph and remove the edge $\bar{e}$ with the maximum upper bound on the estimation error.

\begin{algorithm}[ht]
\caption{$\mathtt{UpdateG}$}
\begin{algorithmic} [1]\label{alg:udtg}
\STATE{\textbf{Input:} $\mathcal{G} = \{\mathcal{S},\mathcal{E},\{ \Delta_{e},\omega_e \}_{e\in \mathcal{E}}$, $(s,s')\in \mathcal{S}\times \mathcal{S}$,   $\Delta,\omega\in \mathbb{R}$. }
\STATE{\textbf{Initialize:} $e' \leftarrow (s,s')$;}
\IF{$e'\in \mathcal{E}$ and $\omega<\omega_{e'}$}
\STATE{$\Delta_{e'}\leftarrow \Delta$, $\omega_{e'}\leftarrow \omega$;\label{line:mark1}}
\ELSIF{$e'\notin \mathcal{E}$}
\STATE{$\mathcal{E}\leftarrow \mathcal{E}\cup e'$; \quad $(\Delta_{e'},\omega_{e'})=(\Delta,\omega)$;}
\STATE{$\mathcal{C}\leftarrow$ a cycle in $\mathcal{E}$ (if there exists),  (i.e., $\mathcal{C} = \{ (s_1,s_2), (s_2,s_3),\ldots, (s_k,s_1) \}\subset \mathcal{E}$;}
\STATE{$\bar{e}\leftarrow \arg\max_{e\in \mathcal{C}}\omega_{e}$;  \quad $\mathcal{E}\leftarrow \mathcal{E}/\bar{e}$;}
\ENDIF
\STATE{\textbf{return:} \{$\mathcal{S},\mathcal{E}, \{ \Delta_e,\omega_e\}_{e\in \mathcal{E}}$\}; }
\end{algorithmic}
\end{algorithm}

\subsection{Theoretical Guarantee}

The $\mathtt{UCB-AVG}$ algorithm achieves near-optimal regret, as stated by the following theorem. 

\begin{theorem}\label{thm:2} If the MDP is weakly communicating,
with space complexity $O(SA)$, Algorithm~\ref{alg:main} achieves the following regret bound with probability at least $1-\delta$:
$$\mathrm{R}(T)=O\bigg(S^5A^2\mathrm{sp}(h^*)\sqrt{T}\ln\frac{S^4A^2T^2}{\delta}\bigg).$$ 
\end{theorem}

We defer the proof to Section \ref{sec:onlinepf}. Below we provide discussion and remarks on this theorem.

\paragraph{Regret optimality.} Theorem~\ref{thm:2} provides a regret bound scaling with $\tilde{O}(\sqrt{T}).$ This sublinear regret bound is optimal in $T$ up to logarithmic factor, in view of the existing information-theoretic lower bounds for average-reward MDPs \citep{jaksch2010near}. To the best of our knowledge, our result provides the first model-free algorithm  that achieves $\tilde{O}(\sqrt{T})$ regret, under only the weekly communicating condition. In contrast, existing results either are suboptimal in $T$ or require strong assumption on uniformly mixing condition \remind{(cf.~Table~\ref{tab:regret_comparison_all})}.

\paragraph{Memory efficiency.} One key feature of our algorithm $\mathtt{UCB-AVG}$ is that it is model-free in nature, with a low space complexity $O(SA).$ In particular, in Algorithm \ref{alg:main} we maintain estimates of the Q-value $Q$ and counter set $\{\mathcal{I},\mathcal{J}\}$ for each state-action,   $O(SA)$ \emph{value-difference} estimators, as well as
a graph $\mG$ with $S$ edges to estimate the reference value function. Note that it takes $O(SA)$ space to store the optimal Q function, thus the space complexity cannot be further improved. 

\paragraph{Computational efficiency.} Encouragingly, our algorithm can be implemented efficiently, with time cost $\tilde{O}(SAT+S^2T + S^4A)$.  A  more detailed discussion is presented in Appendix~\ref{sec:computation_cost}.

\section{Main Results for the Simulator Setting} \label{sec:simulator}

In this section, we consider the simulator setting. To illustrate the key idea of iterative reference-advantage decomposition, we  first present a variant of our algorithm in Section~\ref{sec:warm_up_simulator} as a warm-up. We then provide a refined algorithm with improved sample complexity in Section~\ref{sec:refined_simulator}.

\subsection{The Warm-Up Algorithm}\label{sec:warm_up_simulator}

 The warm-up algorithm is presented in Algorithm~\ref{alg:sc}. The key ingredients of the algorithm are similar to the online case: (i) it learns to solve a $\gamma$-discounted MDP with a carefully designed $\gamma$; (ii) it uses reference-advantage decomposition to reduce variance. The learning process is also divided into several epochs. 
 In each epoch, we maintain an optimistic value function and update the value functions using reference-advantage decomposition. More precisely, given a reference value function $V^{k,\mathrm{ref}}$ at the beginning of the $k$-th epoch, for each state-action pair $(s,a)$, we use $T$ samples to construct estimates $u^k(s,a)$ and $\sigma^k(s,a)$ for the mean and 2nd moment of the reference $P_{s,a}V^{k,\mathrm{ref}}$ (lines \ref{line:ref_sample}-\ref{line:ref_estimate}). In the $l$-th round of the inner iteration, given current value function estimate $v^{k,l}$, we use $T_1$ samples to compute $\zeta^{k,l}(s,a)$ and $\xi^{k,l}(s,a)$ as estimates for the mean and 2nd moment of the advantage $P_{s,a}(v^{k,l}-V^{k,\mathrm{ref}})$ (lines \ref{line:v_sample}-\ref{line:v_estimate_2}). Given the collected estimates above, we update the $Q$-function according to the following equation (line \ref{line:updateq_warm_up}):
\begin{align}
q^{k,l}(s,a) &  = r(s,a) + \gamma \big(u^k(s,a) + \zeta^{k,l}(s,a)  \big) + b^{k,l}(s,a), \label{eq:updateq}
\end{align}
where $b^{k,l}(s,a)= \sqrt{\frac{12 \left(\sigma^k(s,a)-( u^k(s,a))^2\right) \iota }{T}} +\sqrt{\frac{12\left( \xi^{k,l}(s,a)-(\zeta^{k,l}(s,a))^2 \right) \iota}{T_1}} + \frac{5 \spn(V^{k,\mathrm{ref}})\iota}{T} + \frac{5\spn(V^{k,\mathrm{ref}}- v^{k,l} )\iota}{T_1}.$
The bonus term $b^{k,l}$ results in an optimistic estimate $q$ for the optimal discounted Q function $Q^*$.

Given $q^{k,l}$, we then update $v^{k,l+1}$. In order to maintain an optimistic value function that decreases monotonically, one can update $v^{k,l+1}(s)$ only when $ v^{k,l}(s) \geq \max_{a}q^{k,l}(s,a)$. To control the number of updates, we set a threshold for the update: we  update $v^{k,l+1}(s)$ as $\max_{a} q^{k,l}(s,a)$ iff $v^{k,l}(s)\geq \max_{a}q^{k,l}(s,a)+\epsilon_k,$ where $\epsilon_k=\epsilon_{k-1}/2$ is a pre-specified threshold for the $k$-th epoch (lines \ref{line:v_update_warmup}). At the end of $k$-th epoch, we update $V^k$ as the current value function estimate (line \ref{line:defv}).

 We establish the following result on the sample complexity achieved by Algorithm~\ref{alg:sc}. The proof of this result is provided in Section~\ref{sec:pf_simu_warmup}.

\begin{theorem}\label{thm:simu_warm_up}
Consider a weakly communicating average MDP. There exist constants $c_1,c_2,c_3>0$ such that: for each $\delta \in (0,1)$ and $\epsilon>0$, setting $\iota:=\ln(2/\delta),$ $T=\frac{c_1\spn^2(h^*)\iota}{\epsilon^2}$, $1-\gamma = c_2\sqrt{\frac{\iota}{T}},$  $T_1=c_3\sqrt{T\iota},$ and $ K= \left\lfloor \log_2\left( \min\left\{\frac{1}{(1-\gamma)\spn(h^*)}, \frac{T}{\spn(h^*)}, \sqrt{\frac{T}{64\spn^2(h^*)}}\right\} \right) \right \rfloor$ in Algorithm~\ref{alg:sc}, then with probability at least $1-\delta,$ Algorithm~\ref{alg:sc} finds an $O(\epsilon)$-optimal policy with the number of samples bounded by $\tilde{O}\left( \frac{S^2A\spn^2(h^*)\iota}{\epsilon^2}\right)$. 
\end{theorem}

A few remarks are in order. Most prior results on the generative model require the strong assumption of uniform mixing condition, and the resulting sample complexity bounds depend on the \emph{worst-case} mixing time $t_{\min}$ for a policy. In contrast, our result applies to weakly communicating MDP and the sample complexity scales with $\tilde{O}(\spn^2(h^*)).$ We remark that $\spn(h^*)$ is a lower bound on $t_{\min},$ which can be arbitrarily large for certain MDPs. See Table~\ref{tab:sample_complexity_comparison} for a more complete comparison. 
Turning to dependence on $S$ and $A$, we recall that the known lower bound on the sample complexity of finding $\epsilon$-optimal policy for ergodic MDPs scales with $\Omega(SA\spn(h^*)\epsilon^{-2})$ \citep{wang2022average}. The sample complexity bound of  Algorithm~\ref{alg:sc} is off by a factor of $S$. In the next section, we present a refined variant of our algorithm that achieves the optimal dependence on $S$, $A$ and $\epsilon.$

\begin{algorithm}[t]
\caption{Monotone $Q$-learning}
\begin{algorithmic}[1]\label{alg:sc}
\STATE{\textbf{Input}: $S,A$, $\delta$, $\epsilon$}
\STATE{\textbf{Initialization:} $\iota \leftarrow \ln(2/\delta)$, $T\leftarrow \frac{5\cdot 10^{7}\cdot\mathrm{sp}^2(h^*)\iota}{\epsilon^2}$ $\gamma \leftarrow  1-\frac{14000\sqrt{\iota}}{\sqrt{T}}$; $T_1\leftarrow 37\sqrt{T\iota}$; $V^{0}(s)\leftarrow \frac{1}{1-\gamma}, \forall s\in \mathcal{S}$; $K:= \left\lfloor \log_2\left( \min\left\{\frac{1}{(1-\gamma)\spn(h^*)}, \frac{T}{\spn(h^*)}, \sqrt{\frac{T}{64\spn^2(h^*)\iota}}\right\} \right) \right \rfloor$;}
\FOR{$k=1,2,\ldots, K$}
\STATE{$\epsilon_k\leftarrow 2^{-k+1}$; \quad $V^{k,\mathrm{ref}}\leftarrow V^{k-1}$;\quad  $v^{k,1} \leftarrow V^{k-1}$;}

\STATE{Sample $(s,a)$ for $T$ times and collect the next states $\{s_i\}_{i=1}^{T}$; \label{line:ref_sample}}
\STATE{$u^k(s,a)\leftarrow \frac{1}{T}\sum_{i=1}^{T}V^{k,\mathrm{ref}}(s_i)$; $\sigma^k(s,a)\leftarrow \frac{1}{T}\sum_{i=1}^{T}(V^{k,\mathrm{ref}}(s_i))^2$;  \label{line:ref_estimate}}

\FOR{$l =1 ,2,\ldots$}
\STATE{$\mathrm{Trigger}\leftarrow \mathrm{FALSE}$;}

\FOR{each $(s,a)\in \mathcal{S}\times \mathcal{A}$}
\STATE{Sample $(s,a)$ for $T_{1}$ times to collect the next states $\{s_i\}_{i=1}^{T_{1}};$\label{line:v_sample}}
\STATE{{$\zeta^{k,l}(s,a)\leftarrow \frac{1}{T_{1}}\sum_{i=1}^{T_{1}}(v^{k,l}(s_i) - V^{k,\mathrm{ref}}(s_i));$} \label{line:v_estimate_1}}
\STATE{
$\xi^{k,l}(s,a)\leftarrow \frac{1}{T_{1}}\sum_{i=1}^{T_{1}}\left(v^{k,l}(s_i)-V^{k,\mathrm{ref}}(s_i)\right)^2$;\label{line:v_estimate_2}}
\STATE{ Update $q^{k,l}(s,a)$ according to \eqref{eq:updateq};\label{line:updateq_warm_up}}
\ENDFOR \label{line:b1}
\FOR{$s\in \mathcal{S}$}
\IF{$v^{k,l}(s)\geq \max_{a}q^{k,l}(s,a)+\epsilon_k$}
\STATE{ $v^{k,l+1}(s)\leftarrow \max_{a}q^{k,l}(s,a)$; \quad $\mathrm{Trigger}=\mathrm{TRUE}$;\label{line:v_update_warmup}}
\ELSE
\STATE{$v^{k,l+1}(s)\leftarrow v^{k,l}(s)$;\label{line:updatev}}
\ENDIF
\ENDFOR

\IF{$!\mathrm{Trigger}$}
\STATE{$V^{k}\leftarrow v^{k,l+1}$;\label{line:defv}}
\STATE{\textbf{break};}
\ENDIF

\ENDFOR

\ENDFOR

\end{algorithmic}
\end{algorithm}

\subsection{Refined Algorithm}\label{sec:refined_simulator}

In Algorithm~\ref{alg:sc}, in each inner round, we compute $q^{k,l}(s,a)$ for all $(s,a)\in \mathcal{S}\times\mathcal{A}$. Note that we set $v^{k,l+1}(s)=v^{k,l}(s)$ when $q^{k,l}(s,a)+\epsilon_k \geq v^{k,l}(s)$. To reduce the sample complexity, if we are guaranteed that $q^{k,l}(s,a)+\epsilon_k \geq v^{k,l}(s)$, there is no need to sample to compute $q^{k,l}(s,a)$. However this procedure turns to be a circle: to judge whether $q^{k,l}(s,a)+\epsilon_k \geq v^{k,l}(s)$, we have to collect samples to compute $q^{k,l}(s,a)$. To address this issue, we turn to maintain a proper upper bound $\overline{q}^{k,l}(s,a)$ for $q^{k,l}(s,a)$ and only choose to compute $q^{k,l}(s,a)$ when $\overline{q}^{k,l}(s,a)+\epsilon_k < v^{k,l}(s)$. We design $\overline{q}^{k,l}(s,a)$ such that the number of samples needed for computing $\overline{q}^{k,l}(s,a)$ is comparably small. Due to space limit, we provide a detailed description of the refined algorithm (Algorithm~\ref{alg:scre}) in Section~\ref{sec:pf_simu_refine}.

We establish the following guarantee on the sample complexity achieved by the refined Algorithm~\ref{alg:scre}. The proof is provided in Section~\ref{sec:pf_simu_refine}.

\begin{theorem}\label{theorem:resc}
Consider a weakly communicating average MDP. There exist constants $c_1,c_2,c_3>0$ such that: for each $\delta \in (0,1)$ and $\epsilon>0$, setting $\iota:=\ln(2/\delta),$ $T=\frac{c_1\spn^2(h^*)\iota}{\epsilon^2}$, $1-\gamma = c_2\sqrt{\frac{\iota}{T}},$  $T_1=c_3\sqrt{T\iota}$, $T_2 = 10\iota$ and $ K= \left\lfloor \log_2\left( \min\left\{\frac{1}{(1-\gamma)\spn(h^*)}, \frac{T}{\spn(h^*)}, \sqrt{\frac{T}{64\spn^2(h^*)}}\right\} \right) \right \rfloor$ in Algorithm~\ref{alg:scre}, then with probability at least $1-\delta,$ Algorithm~\ref{alg:scre} finds an $O(\epsilon)$-optimal policy with the number of samples bounded by $\tilde{O}\left( \frac{SA\spn^2(h^*)\iota }{\epsilon^2} + \frac{S^2A\spn(h^*)\iota}{\epsilon} \right).$
\end{theorem}

We remark that for a small target error $\epsilon,$ the leading term of the sample complexity achieved by Algorithm~\ref{alg:scre} scales with $\otilde(SA\spn^2(h^*)\epsilon^{-2}).$ Compared with the result of Algorithm~\ref{alg:sc} (cf.~Theorem~\ref{thm:simu_warm_up}), the refined model-free algorithm improves the sample complexity by a factor of $S,$ achieving the optimal dependence on $S$ and $A.$ The improvement originates from the adaptive sampling, thanks to the proper optimistic estimate of the optimal Q function. In addition, similar to the guarantee of the warm-up algorithm, the performance guarantee of the refined algorithm only requires weakly communicating condition. Compared with the lower bound $\Omega(SA\spn(h^*)\epsilon^{-2})$ \citep{wang2022average}, our bound is off by a factor of $\spn(h^*).$ It will be interesting to investigate whether such a lower bound is achievable; we leave this as a future research direction.

\section{Conclusions}

In this paper, we develop provably near-optimal model-free reinforcement learning algorithms for infinite-horizon average-reward MDPs. We consider both the online setting and the simulator settings. Our algorithms build on the ideas of reducing average-reward problems to discounted-MDPs and variance reduction via reference-advantage decomposition, featuring novel construction of reference value function with low space complexity. Our results also highlight challenges in the average-reward setting. We hope that our work may serve as a first step towards a better understanding of optimal RL algorithms for long-run average-reward problems.

\section*{Acknowledgements}
The authors thank Kaiqing Zhang for insightful discussions. Q.\ Xie is partially supported by NSF grant CNS-1955997.

\bibliographystyle{apalike}
\bibliography{rl_refs}

\newpage

\appendix

\section{Technique Lemmas} \label{sec:tech_lemmas}

The proofs of our main Theorems involve several technical lemmas, which are summarized below.

\begin{lemma}\label{lemma:freedman} [Freedman's Inequality]
Let $(M_{n})_{n\geq 0}$ be a  martingale such that $M_{0}=0$ and $|M_{n}-M_{n-1}|\leq c$. Let $\mathrm{Var}_{n}=\sum_{k=1}^{n}\mathbb{E}[(M_{k}-M_{k-1})^{2}|\mathcal{F}_{k-1}]$ for $n\geq 0$,
	where $\mathcal{F}_{k}=\sigma(M_0,M_{1},M_{2},...,M_{k})$. Then, for any positive $x$ and for any positive $y$,
	\begin{equation}\label{Bernstein2}
	\mathbb{P}\left[ \exists n:  M_{n}\geq x ~\text{and}~\mathrm{Var}_{n}\leq y \right]  \leq \exp\left(-\frac{x^{2}}{2(y+cx)} \right).
	\end{equation}

\end{lemma}

\begin{lemma}\label{lemma:con}
Let $X_1,X_2,\ldots$ be a sequence of random variables taking value in $[0,l]$. Define $\mathcal{F}_k =\sigma(X_1,X_2,\ldots,X_{k-1})$ and $Y_k = \mathbb{E}[X_k|\mathcal{F}_k]$ for $k\geq 1$. For any $\delta>0$, we have that
\begin{align}
& \mathbb{P}\left[ \exists n, \sum_{k=1}^n X_k \geq  3\sum_{k=1}^n Y_k+ l\ln(1/\delta)\right]\leq \delta\nonumber
\\  & \mathbb{P}\left[  \exists n,  \sum_{k=1}^n Y_k \geq 3\sum_{k=1}^n X_k + l\ln(1/\delta)  \right]    \leq \delta .\nonumber 
\end{align}
\end{lemma}

\begin{proof} Let $t\in [0,1/l]$ be fixed.
Define $Z_k:=\exp\Big(t\sum_{k'=1}^k(X_{k'}-3Y_{k'})  \Big)$. By definition, we have that
\begin{align}
    \mathbb{E}[Z_k|\mathcal{F}_k]  & =\exp\left(t\sum_{k'=1}^{k-1}(X_{k'}-3Y_{k'})\right) \mathbb{E}\left[\exp\big(t(X_{k}-3Y_{k})\big)|\mathcal{F}_k\right] \nonumber 
    \\ & \leq \exp\left(t\sum_{k'=1}^{k-1}(X_{k'}-3Y_{k'})\right)\exp(-3tY_{k})\cdot \mathbb{E}\big[1+tX_k+2t^2X^2_{k}|\mathcal{F}_k\big]\nonumber 
    \\ & \leq \exp\left(t\sum_{k'=1}^{k-1}(X_{k'}-3Y_{k'})\right)\exp(-3tY_{k})\cdot \mathbb{E}\big[1+3tX_k|\mathcal{F}_k\big]\nonumber 
    \\ & = \exp\left(t\sum_{k'=1}^{k-1}(X_{k'}-3Y_{k'})\right)\exp(-3tY_{k})\cdot (1+3tY_k)\nonumber 
    \\ & \leq \exp\left(t\sum_{k'=1}^{k-1}(X_{k'}-3Y_{k'})\right)\nonumber
    \\ & =Z_{k-1},\nonumber 
\end{align}
where the second line is by the fact that $e^x\leq 1+x+2x^2$ for $x\in [0,1]$. 
Define $Z_0 = 1$
Then $\{Z_{k}\}_{k\geq 0}$ is a super-martingale with respect to $\{\mathcal{F}_{k}\}_{k\geq 1}$. Let $\tau$ be the smallest $n$ such that $\sum_{k=1}^n X_k - 3\sum_{k=1}^n Y_k >l\ln(1/\delta)$.
It is easy to verify that $Z_{\min\{\tau,n\}}\leq \exp(tl\ln(1/\delta)+tl)<\infty$. Choose $t=1/l$.
By the optimal stopping time theorem, we have that
\begin{align}
 & \mathbb{P}\left[ \exists n\leq N, \sum_{k=1}^n X_k \geq  3\sum_{k=1}^n Y_k + l\ln(1/\delta)\right]\nonumber 
 \\ & =\mathbb{P}\left[\tau \leq N\right]\nonumber 
 \\ & \leq \mathbb{P}\left[ Z_{\min\{\tau,N\}}\geq \exp(tl\ln(1/\delta))  \right]\nonumber 
 \\ & \leq \frac{\mathbb{E}[Z_{\min\{\tau,N\}}]}{\exp(tl\ln(1/\delta))}\nonumber 
 \\ & \leq \delta. \nonumber
\end{align}

Letting $N\to \infty$, we have that
\begin{align}
& \mathbb{P}\left[ \exists n, \sum_{k=1}^n X_k \leq  3\sum_{k=1}^n Y_k+ l\ln(1/\delta)\right]\leq \delta.\nonumber
\end{align}
Considering $W_k = \exp(t\sum_{k'=1}^k (Y_k/3-X_k))$, using similar arguments  and choosing $t=1/(3l)$, we have that
\begin{align}
 & \mathbb{P}\left[  \exists n,  \sum_{k=1}^n Y_k \geq 3\sum_{k=1}^n X_k + l\ln(1/\delta)  \right]    \leq \delta .\nonumber 
\end{align}
The proof is completed.
\end{proof}

\begin{lemma}\label{lemma:con4}
Let $v\in R^{S}$ be fixed. Let $X_1,X_{2},\ldots, X_{t}$ be i.i.d. multinomial distribution with parameter $p\in \Delta^{S}$. Define $\widehat{\mathrm{Var}} = \frac{1}{t} \left( \sum_{i=1}^t v^2_{X_i} - \frac{1}{t}\left(\sum_{i=1}^t v_{X_i}\right)^2 \right)$ be the empirical variance of $\{v_{X_i}\}_{i=1}^t$ and $\mathrm{Var}=\mathbb{V}(p,v)$ be the true variance of $v_{X_1}$. With probability $1-3\delta$, it holds that
\begin{align}
\frac{1}{3}\mathrm{Var} - \frac{7\spn^2(v)\ln(1/\delta)}{3t} \leq \widehat{\mathrm{Var}}\leq 3\mathrm{Var}+ \frac{\spn^2(v)\ln(1/\delta)}{t}.\label{eq:ee1}
\end{align}
\end{lemma}
\begin{proof}
Without loss of generality, we assume $\mathbb{E}[v_{X_1}]=0$. By Lemma~\ref{lemma:con}, with probability $1-\delta$, it holds that
\begin{align}
    t\widehat{\mathrm{Var}} \leq \sum_{i=1}^t v^2_{X_i}\leq 3t\mathbb{E}[v^2_{X_i}] + \spn^2(v)\ln(1/\delta).
\end{align}
Dividing both sides with $t$, we prove the right-hand side of \eqref{eq:ee1}. For the other side, by Hoeffding's inequality, we have that $|\sum_{i=1}^t v_{X_i}|\leq \spn(v)\sqrt{2t\ln(1/\delta)}$ holds with probability $1-\delta$. Using Lemma~\ref{lemma:con} again, with probability $1-\delta$ it holds that
\begin{align}
    \sum_{i=1}^tv^2_{X_i} \geq \frac{t}{3}\mathbb{E}[v^2_{X_i}] - \frac{1}{3}\spn^2(v)\ln(1/\delta).\nonumber
\end{align}
As a result, with probability $1-2\delta$ it holds  that
\begin{align}
    t\widehat{\mathrm{Var}} & \geq \frac{t}{3}\mathbb{E}[v^2_{X_i}] - \frac{1}{3}\spn^2(v)\ln(1/\delta) - \frac{1}{t}\cdot  2t\ln(1/\delta)\spn^2(v) \nonumber 
    \\ & = \frac{t}{3}\mathbb{E}[v^2_{X_i}]-\frac{7}{3}\spn^2(v)\ln(1/\delta).\nonumber
\end{align}
The proof is completed by dividing both sides by $t$.

\end{proof}

\begin{lemma}\label{self-norm}[Lemma 10 in \cite{zhang2020almost}] \label{lemma:freeman_zhang}
Let $(M_{n})_{n\geq 0}$ be a  martingale  such that $M_0 = 0$  and $|M_{n}-M_{n-1}|\leq c$ for some $c>0$ and any $n\geq 1$. Let $\mathrm{Var}_{n}=\sum_{k=1}^{n}\mathbb{E}[(M_{k}-M_{k-1})^{2}|\mathcal{F}_{k-1}]$ for $n\geq 0$,
where $\mathcal{F}_{k}=\sigma(M_{1},M_{2},...,M_{k})$. Then for any positive integer $n$, and any $\epsilon,p>0$, we have that
\begin{equation}\label{self-bernstein}
\mathbb{P}\left[|M_{n}|\geq   2\sqrt{\mathrm{Var}_{n}\log(\frac{1}{p})}+2\sqrt{\epsilon\log(\frac{1}{p} )} +2c\log(\frac{1}{p}) \right] \leq \left(\frac{2nc^2}{\epsilon}+2\right)p.
\end{equation}
\end{lemma}

\begin{lemma} [Lemma 4.1 of~\citep{jin2018q}] \label{lemma:jin_alpha}
Recall that $\alpha_{i}=\frac{H+1}{H+i}$. 
Let $\alpha_{\tau}^0 = \Pi_{i=1}^{\tau}(1-\alpha_{i})$ and $\alpha_{\tau}^i = \alpha_i \Pi_{j=i+1}^{\tau}(1-\alpha_{j})$. The following inequalities hold:
\begin{align}
& \frac{1}{\sqrt{t}}\leq \sum_{i=1}^{\tau}\frac{\alpha_{\tau}^i}{\sqrt{i}}\leq \frac{2}{\sqrt{t}},\quad    \sum_{i=1}^{\tau}(\alpha_{\tau}^i)^2\leq \frac{2H}{\tau}, \quad  \sum_{\tau = i}^{\infty}\alpha_{\tau}^i \leq 1+\frac{1}{H},\quad \max_{i}\alpha_{\tau}^i \leq \frac{H}{\tau}.\label{lemma:parameter}
\end{align}
\end{lemma}

\begin{lemma}\label{lemma:tool1}
Let $\alpha_t^i$ be defined as in Lemma~\ref{lemma:jin_alpha}. 
 For $j\geq i $, define $w_{j}^i = \sum_{i'=i}^j\alpha_{i'}^i$. Then we have that $w_{j}^{i_1} \leq w_{j}^{i_2} $ for $j\geq i_1\geq i_2$.
\end{lemma}
\begin{proof}
It suffices to show that  $w_{j}^{i_1} \leq  w_{j}^{i_2} $ for $j\geq i_1 = i_2 + 1$ and $i_2\geq 1$. Noting that $w_j^{i_2} = \sum_{k=0}^{j-i_2}\alpha_{i_2+k}^{i_2}\geq \sum_{k=0}^{j-i_1}\alpha_{i_2+k}^{i_2}$, it suffices to prove that $w_{i_2+k}^{i_2}\geq w_{i_1+k}^{i_1}$ for any $k\geq 0$. Let $i_2 = i$ and $i_1 = i_2+1$. Below we prove by induction that
\begin{align}
w_{i_2+k}^{i_2} - w_{i_1+k}^{i_1} = \frac{(H+1)(k+1)(i+1)(i+2)\ldots (i+k)}{(H+i)(H+i+1)\ldots (H+i+k+1)} 
 = \frac{(H+1)(k+1)\Pi_{j=1}^k(i+j)}{\Pi_{j=0}^{k+1}(H+i+j)}.\label{eq:211xx}
\end{align}
For $k=0$, it is easy to verify that $w_{i_2}^{i_2} - w_{i_1}^{i_1} = \frac{H+1}{(H+i)(H+i+1)}$. Suppose \eqref{eq:211xx} holds for $k$. Then we have that
\begin{align}
 &  w_{i_2+k+1}^{i_2} - w_{i_1+k+1}^{i_1}  \nonumber
\\ & = w_{i_2+k}^{i_2} - w_{i_1+k}^{i_1}  + (\alpha_{i_2+k+1}^{i_2} - \alpha_{i_1+k+1}^{i_1})   \nonumber
\\ & =  \frac{(H+1)(k+1)\Pi_{j=1}^k(i+j)}{\Pi_{j=0}^{k+1}(H+i+j)} + \frac{(H+1) \Pi_{j=0}^{k}(i+j) }{\Pi_{j=0}^{k+1}(H+i+j)  }  - \frac{(H+1) \Pi_{j=0}^{k}(i+1+j) }{\Pi_{j=0}^{k+1}(H+i+1+j)  } \nonumber
\\ & =(H+1)\cdot  \frac{(k+1)(H+i+k+2)\Pi_{j=1}^k(i+j) +(H+i+k+2)\Pi_{j=0}^k(i+j) - (H+i)\Pi_{j=0}^k(i+1+j) }{\Pi_{j=0}^{k+2}(H+i+j)} \nonumber
\\ & = (H+1)\Pi_{j=1}^k(i+j)\cdot  \frac{(k+1)(H+i+k+2) +(H+i+k+2)i - (H+i)(i+k+1)}{\Pi_{j=0}^{k+2}(H+i+j)} \nonumber
\\ & =  (H+1)\Pi_{j=1}^k(i+j)\cdot  \frac{(k+2)(k+1+i)}{\Pi_{j=0}^{k+2}(H+i+j)} \nonumber
\\ & =  \frac{(H+1)(k+2) \Pi_{j=1}^{k+1}(i+j)}{\Pi_{j=0}^{k+2}(H+i+j)}.\nonumber
\end{align}
The proof is completed.
\end{proof}

\section{Proof of Theorem~\ref{thm:2}}\label{sec:onlinepf}

In this section, we prove Theorem~\ref{thm:2} for the online setting. We shall make use of the technical lemmas given in Appendix~\ref{sec:tech_lemmas}. The proof of Theorem~\ref{thm:2} consists of the following three key steps:
\begin{enumerate}
    \item We first analyze the performance of the key subroutine $\mathtt{UCB-REF}$ (Algorithm~\ref{alg:ucbref}). In particular, we establish the regret of Algorithm~\ref{alg:ucbref} w.r.t.~a given graph input $\mG=\{\mE,\{\Delta_e,\omega_e\}_{e\in \mE}\}.$ See Section~\ref{sec:pf_ucb_ref_regret}.
    \item We next show the guarantee on the value difference estimate achieved by the graph update. In particular, we establish the guarantees on the output of $\mathtt{EBF^+}$ (Line~\ref{line:ebf}-\ref{line:omega1} Algorithm~\ref{alg:ucbref}) and $\mathtt{UpdateG}$ (Algorithm~\ref{alg:udtg}) respectively. See Section~\ref{sec:pf_graph}.
    \item Building on the results from the first two steps, we derive an upper bound on  the regret of each epoch, which leads to the final desired bound on the total regret and thereby completing the proof of the theorem. See Section~\ref{sec:pf_thm_online}
\end{enumerate}

\paragraph{Additional notation used in the proof.} Before moving to the details of each step, we introduce additional notations used throughout the analysis. Consider a fixed epoch $k\geq 1$. 
\begin{itemize}
    \item $V^{\mathrm{ref},k}$: the reference value function at the beginning of the $k$-th epoch.
    \item $n^k(s,a,s')$: the count of $(s,a,s')$ visits in the $k$-th epoch. 
    \item $N^k(s,a,s')$: the total count of $(s,a,s')$ visits before the $k$-th epoch, i.e., $N^k(s,a,s') = \sum_{1\leq k'<k}n^{k'}(s,a,s')$.
    \item $N^k(s,a):$ the total count of $(s,a)$ visits before the $k$-th epoch.
    \item  $\iota:=\ln \frac{1}{\delta}.$
    \item Define $\tilde{d}^k_{(s,s')}:= \min_{\mathcal{C}\subset \mathcal{E}, \mathcal{C} \text{ connects } (s,s')}\sum_{e\in \mathcal{C}}\omega_e$ as the upper bound on estimate error of $V^*(s)-V^*(s')$ at the end of the $k$-th epoch.
    \item $(I^k(s,a),J^k(s,a)):$ the value of $I(s,a)$ and $J(s,a)$ before the $k$-th epoch;
\end{itemize}

We further define $n^k(s,a) =\max\{ \sum_{s'}n^k(s,a,s'),1 \}$, $\check{n}^k(s,s') =\sum_{a}n^k(s,a,s')$,  $\check{N}^k(s,s') =  \sum_{a} N^k(s,a,s')$. 
 Let $t_k$ be the number of steps in the $k$-th epoch and  $R_k$ denote the regret in the $k$-th round.

 \subsection{Regret Analysis for $\mathtt{UCB-REF}$ (Algorithm~\ref{alg:ucbref})} \label{sec:pf_ucb_ref_regret}

In the main algorithm, we invoke a subroutine called UCB-REF (Algorithm~\ref{alg:ucbref}), in which we play UCB-Q learning with a reference value function as input. Formally, we have the lemma below to bound the regret for Algorithm~\ref{alg:ucbref}. 

\begin{lemma}\label{lemma:1_1} 
Let $\mathcal{E}$ and $\{\Delta_e,\omega_e\}_{e\in \mathcal{E}}$ be the input for  Algorithm~\ref{alg:ucbref}. Define $\tilde{d}_{(s,s')}:= \min_{\mathcal{C}\subset \mathcal{E}, \mathcal{C} \text{ connects } (s,s')}\sum_{e\in \mathcal{C}}\omega_e$. Then with 
probability at least $1-5SAt\delta$, the regret of running Algorithm~\ref{alg:ucbref} for $t$ steps is bounded by 
\begin{align}
R(t)\leq 400\sqrt{SAt\spn^2(h^*)\iota} +200\sqrt{\frac{SA\iota\sum_{s,a,s'}N(s,a,s')\tilde{d}^2_{(s,s')}}{1-\gamma}} +400\frac{SA\mathrm{sp}(h^*)\iota}{1-\gamma} +t(1-\gamma)\spn(h^*),\nonumber
\end{align}
 where $N(s,a,s')$ denotes the count of $(s,a,s')$ visits during the $t$ steps for any proper $(s,a,s')$.
\end{lemma}

\smallskip
In the rest of the proof, we use $\mathcal{F}_t$ to denote the event field of $\{s_i,a_i\}_{i=1}^{t}$.

\begin{proof}[Proof of Lemma~\ref{lemma:1_1}]
We start with some notations, which will be used throughout this section and the proofs in Appendix~\ref{sec:onlinepf}.

We denote by $V_{\tau},Q_{\tau}$ and $N_{\tau}$ the values of $V$, $Q$ and $N$ respectively at the beginning of the $\tau$-th step. With a slight abuse of notations, we define $N_{\tau}(s,a,s') = \sum_{\tau'=1}^{\tau-1}\mathbb{I}[(s_{\tau'},a_{\tau'},s_{\tau'+1})=(s,a,s')]$ as the count of $(s,a,s')$ visits before the $\tau$-th step. Let $N_{\tau}(s,s'):=\sum_{a}N_{\tau}(s,a,s')$. We remark that $\{N_{\tau}(s,s')\}$ only appear in the analysis and we do not to maintain accumulators for these quantities.

Besides, we need the accumulators below to facilitate the algorithm and proof.

\begin{enumerate}
    \item   $\mu(s,a): $ accumulators for $V^{\mathrm{ref}}(s')$ with $s'$ as the next state of $(s,a)$
    \item $\xi^2(s,a) :  $ accumulators for $(V^{\mathrm{ref}}(s'))^2$ with $s'$ as the next state of $(s,a)$;
    \item $\kappa^2(s,a):$ accumulators for $\tilde{d}_{(s,s')}$ with $s'$ as the next state of $(s,a)$;
\end{enumerate}

We let $\mu_{\tau}(s,a)$, $\xi^2_{\tau}(s,a)$ and $\kappa^2_{\tau}(s,a)$ respectively denote the value of $\mu(s,a)$, $\xi^2(s,a)$ and $\kappa^2(s,a)$ before the $\tau$-th step.

\smallskip

Our proof consists of three steps:
\begin{itemize}
    \item \textbf{Step 1: Optimism of Q-estimate $Q_{\tau}$.}  We begin by showing that the Q-function estimate $Q_{\tau}$ maintained in the algorithm is an upper bound on the optimal Q-function $Q^*$. 
    \item \textbf{Step 2: Regret Decomposition.} We then derive an upper bound on the regret in terms of the bonus term used in the Q-function update.
    \item \textbf{Step 3: Establishing the Final Bound.} We establish the desired bound on the regret.
\end{itemize}

Below we provide the details of each step.

\noindent \textbf{Step 1: Optimism of Q-estimate $Q_{\tau}$.}

We will show that $Q_{\tau} \geq Q^*$ holds with high probability for $\tau$ by induction. First note that $Q_1(s,a)=\frac{1}{1-\gamma} \geq Q^*(s,a)$ for all $(s,a)\in \mS\times\mA.$ That is, the base case $Q_{1}\geq Q^*$ holds.

Now assume that $Q_{\tau'} \geq Q^*$ hold for all $\tau'\leq \tau.$ 
By Proposition~\ref{pro:proj}, we have that $V_{\tau'}\geq V^*$
holds for all $\tau'\leq \tau$.
Let $V^{\mathrm{ref}}$ be defined as in Line~\ref{line:vref} Algorithm~\ref{alg:ucbref}.
By the update rule, for each $(s,a)\neq (s_{\tau},a_{\tau}),$
we have $Q_{\tau+1}(s,a)=Q_{\tau}(s,a)\geq Q^*(s,a).$ For $(s,a)=(s_{\tau},a_{\tau})$, we have
we then have 
\begin{align}
Q_{\tau+1}(s,a) & = (1-\alpha_{N_{\tau}(s,a)})Q_{\tau}(s,a)+ \nonumber 
\\ &  \alpha_{N_{\tau}(s,a)}\cdot \left( r(s,a)+\gamma \left(V_{\tau}(s_{\tau+1})-V^{\mathrm{ref}}(s_{\tau+1})\right) +  \frac{\gamma \mu_{\tau}(s,a)}{N_{\tau}(s,a)} + b_{\tau}(s,a) \right),\label{eq:ucbrefup}
\end{align}
where 
\begin{align}
    b_{\tau}(s,a) =   6\sqrt{\frac{\mathrm{sp}^2(h^*)\iota}{N_{\tau}(s,a)}} + 38\frac{H\spn(h^*)\iota }{N_{\tau}(s,a)} + 36\sqrt{\frac{H\iota\kappa^2_{\tau}(s,a)}{N^2_{\tau}(s,a)}}.\label{eq:define_xi}
\end{align}
Recall the definition of $\mu_{\tau}(s,a)$ and $\xi_{\tau}^2(s,a)$. Also note that $|V^{\mathrm{ref}}(s')|\leq 2\spn(h^*).$
By Hoeffding's Lemma, and noting that $\mathrm{sp}(V^{\mathrm{ref}})\leq 2\mathrm{sp}(h^*)$, with probability at least $1-2\delta,$ it holds that 
\begin{align}
\left|\frac{\gamma\mu_{\tau}(s,a)}{N_{\tau}(s,a)}-  \gamma P_{s,a}V^{\mathrm{ref}}\right| \leq 6\sqrt{\frac{\mathrm{sp}^2(h^*)\iota}{N_{\tau}(s,a)}} + 19\frac{\spn(h^*)\iota }{N_{\tau}(s,a)}.\label{eq:bonus1}
\end{align}

We thus have
\begin{align*}
    &Q_{\tau+1}(s,a)  \geq (1-\alpha_{N_{\tau}})Q_{\tau}(s,a) \nonumber \\
    & \quad+ \alpha_{N_{\tau}}\cdot \left( r(s,a)+\gamma \left(V_{\tau}-V^{\mathrm{ref}}\right)(s_{\tau+1})+\gamma P_{s,a}V^{\mathrm{ref}}  + 36\sqrt{\frac{H\iota\kappa^2_{\tau}(s,a)}{N^2_{\tau}}} +{19\frac{H\spn(h^*)\iota}{N_{\tau}}}\right).
\end{align*}
Here for the ease of exposition, we simply use $N_{\tau}$ as a shorthand for $N_{\tau}(s,a).$ Let $l_i^{\tau}$ denote the time step for the $i$-th sample of $(s_{\tau},a_{\tau})$. By recursively substituting $Q_{\tau},$ we have
\begin{align*}
    &Q_{\tau+1}(s,a) \geq H\alpha^0_{\tau} \\
    &\qquad+  \sum_{i=1}^{N_{\tau}} \alpha_{N_{\tau}}^i \left( r(s,a)+\gamma \left(V_{l_i^{\tau}}-V^{\mathrm{ref}}\right)(s_{l_i^{\tau}+1}) +\gamma P_{s,a}V^{\mathrm{ref}}  + 36\sqrt{\frac{H\kappa^2_{\tau}(s,a)}{N^2_{\tau}}} +{19\frac{H\spn(h^*)\iota}{N_{\tau}}}\right).
\end{align*}
On the other hand, since $\sum_{i=1}^{N_{\tau}} \alpha_{N_{\tau}}^i=1$ (cf.~ Lemma \ref{lemma:jin_alpha}), by Bellman equation we have
\begin{align*}
    Q^*(s,a)=\alpha^0_{\tau}Q^*(s,a)+\sum_{i=1}^{N_{\tau}} \alpha_{N_{\tau}}^i \left( r(s,a)+\gamma P_{s,a}V^*\right).
\end{align*}
By induction, for each $s'\in \mS,$ we have $V_{\tau'}(s')=\max_{a}Q_{\tau'}(s',a) \geq \max_{a}Q^*(s',a)=V^*(s')$ for all $\tau'\leq \tau.$
Therefore, we have
\begin{align}
   &Q_{\tau+1}(s,a)-Q^*(s,a) \nonumber \\
   &\geq \alpha^0_{\tau}(H-Q^*(s,a)) \nonumber \\
   &+ \sum_{i=1}^{N_{\tau}} \alpha_{N_{\tau}}^i \left( \gamma \left(V_{l_i^{\tau}}-V^{\mathrm{ref}}\right)(s_{l_i^{\tau}+1}) -\gamma P_{s,a}\left(V_{l_i^{\tau}}-V^{\mathrm{ref}} \right) + 36\sqrt{\frac{H\kappa^2_{\tau}(s,a)}{N^2_{\tau}}} +{19\frac{H\spn(h^*)\iota}{N_{\tau}}}
   \right)\nonumber\\
   &\geq \sum_{i=1}^{N_{\tau}} \alpha_{N_{\tau}}^i \left( \gamma \left(V_{l_i^{\tau}}-V^{\mathrm{ref}}\right)(s_{l_i^{\tau}+1}) -\gamma P_{s,a}\left(V_{l_i^{\tau}}-V^{\mathrm{ref}} \right) + 36\sqrt{\frac{H\kappa^2_{\tau}(s,a)}{N^2_{\tau}}} +{19\frac{H\spn(h^*)\iota}{N_{\tau}}} \right), \label{eq:Q_difference}
\end{align}
where the last inequality follows from the fact that $H=\frac{1}{1-\gamma}\geq Q^*(s,a).$

\begin{lemma}\label{lemma:boundkappa}
With probability $1-4SAt\delta$, for any proper $\tau,s,a$, it holds that
\begin{align}
N_{\tau}(s,a)\sum_{s'}P_{s,a,s'}\tilde{d}^2_{(s,s')}\leq 3\kappa^2_{\tau}(s,a)+\mathrm{sp}^2(h^*)\iota .\label{eq:21022}
\end{align}
\end{lemma}
\begin{proof}
Fix $(s,a,\tau)$. Then \eqref{eq:21022} follows by Lemma~\ref{lemma:con}. With a union bound on the failure probability we finish the proof. 
\end{proof}
 Let $\epsilon_i = (V^{\mathrm{ref}}-V_{l_i^{\tau}})(s_{l_i^{\tau}+1})-P_{s,a}(V^{\mathrm{ref}}-V_{l_i^{\tau}})$. Then $\epsilon_i$ is a zero-mean random variable with variance at most $x:=4\sum_{s'}P_{s,a,s'} \tilde{d}^2_{(s,s')}$.
Then by Freedman's Inequality, with probability at least $1-\delta,$ we have that
\begin{align}
\left|\sum_{i=1}^{N_{\tau}}\alpha_{N_{\tau}}^{i}\epsilon_i \right| & \stackrel{(a)}{\leq} 3\sqrt{   \sum_{i=1}^{N_{\tau}}\left(\alpha_{N_{\tau}}^{i}\right)^2 x\iota   }+ \max_{i}\alpha_{N_{\tau}}^i \spn(h^*)\iota \nonumber 
\\ & 
\stackrel{(b)}{\leq} 3\sqrt{   \sum_{i=1}^{N_{\tau}}\left(\alpha_{N_{\tau}}^{i}\right)^2 x\iota   }+ \frac{H\spn(h^*)\iota}{\tau}  \nonumber   
\\ & \stackrel{(c)}{\leq} 3\sqrt{\frac{2Hx\iota}{N_{\tau}}} + \frac{H\spn(h^*)\iota}{N_{\tau}}\nonumber 
\\ & \stackrel{(d)}{\leq} 6\sum_{i=1}^{N_{\tau}}\alpha_{N_{\tau}}^i\sqrt{\frac{Hx\iota}{i}}+ \sum_{i=1}^{N_{\tau}}\alpha_{N_{\tau}}^i\frac{H\spn(h^*)\iota}{i}\nonumber 
\\ & \stackrel{(e)}{\leq} 12\sum_{i=1}^{N_{\tau}}\alpha_{N_{\tau}}^i\sqrt{\frac{6H(\kappa^2_i(s,a)+ \spn^2(h^*)\iota)\iota}{i^2}}+ \sum_{i=1}^{N_{\tau}}\alpha_{N_{\tau}}^i\frac{H\spn(h^*)\iota}{i}\nonumber 
\\ & {\leq} 36\sum_{i=1}^{N_{\tau}}\alpha_{N_{\tau}}^i \sqrt{\frac{H\iota \kappa_i^2(s,a)}{i^2}} + 18\sum_{i=1}^{N_{\tau}}\alpha_{N_{\tau}}^i \spn(h^*)\iota \sqrt{\frac{H}{i^2}} + \sum_{i=1}^{N_{\tau}}\alpha_{N_{\tau}}^i\frac{H\spn(h^*)\iota}{i}\nonumber 
\\ & \leq 
36\sum_{i=1}^{N_{\tau}}\alpha_{N_{\tau}}^i \sqrt{\frac{H\iota \kappa_i^2(s,a)}{i^2}} + 19 \sum_{i=1}^{N_{\tau}}\alpha_{N_{\tau}}^i\frac{H\spn(h^*)\iota}{i},\label{eq:bonus2}
\end{align}
where $(a)$ holds by Lemma~\ref{lemma:freeman_zhang}, $(b)-(d)$ follow by Lemma~\ref{lemma:jin_alpha}, $(e)$ is true due to Lemma~\ref{lemma:boundkappa}.

Plugging the above inequality into equation \eqref{eq:Q_difference}, and noting the induction assumption $V_{\tau'}\geq V^*$ for any $\tau'<\tau$, we have that with probability at least $1-\delta,$ 
\begin{align*}
&Q_{\tau+1}(s,a)-Q^*(s,a)\geq 0.
\end{align*}
The induction is thus completed. 

Therefore, by union bound, we show that with probability at least $1-tSA\delta,$
\begin{align*}
&Q_{\tau}(s,a)-Q^*(s,a)\geq 0,
\end{align*}
holds simultaneously for all $(s,a,\tau)\in \mS\times\mA \times[t].$ As a consequence, we have that $V_{\tau}(s)\geq V^*(s)$ for all $(s,\tau)\in \mathcal{S}\times [t]$.

\smallskip
\noindent\textbf{Step 2: Regret Decomposition.}

We next analyze the regret. 
Let  $\check{b}_{\tau}=   6\sqrt{\frac{\mathrm{sp}^2(h^*)\iota}{N_{\tau}(s_{\tau},a_{\tau})}} + 38\frac{H\spn(h^*)\iota }{N_{\tau}(s_{\tau},a_{\tau})} + 36\sqrt{\frac{H\iota\kappa^2_{\tau}(s_{\tau},a_{\tau})}{N^2_{\tau}(s_{\tau},a_{\tau})}}.$

By  \eqref{eq:bonus1} and \eqref{eq:bonus2}, we have that 
\begin{align}
 & V_{\tau}(s_{\tau}) \leq  Q_{\tau}(s_{\tau},a_{\tau }) \nonumber
 \\& =\alpha_{N_{\tau}}^0 H +   \sum_{i=1}^{N_{\tau}}\alpha_{N_{\tau}}^i \left( r(s_{\tau},a_{\tau})+ \gamma (V_{l_i^{\tau}} - V^{\mathrm{ref}})(s_{l_i^{\tau}+1})+\gamma \frac{\mu_{l_i^{\tau}}(s_{\tau},a_{\tau})}{N_{l_i^{\tau} }(s_{\tau},a_{\tau})}+ b_{l_i^{\tau}}(s_{\tau},a_{\tau})\right) \nonumber
\\ & \leq \alpha_{N_{\tau}}^0 H +   \sum_{i=1}^{N_{\tau}}\alpha_{N_{\tau}}^i \left( r(s_{\tau},a_{\tau})+ \gamma P_{s,a}V_{l_i^{\tau}}\right)+ 2\sum_{i=1}^{N_{\tau}}\alpha_{N_{\tau}}^i \check{b}_{l_i^{\tau}} \nonumber
\\ & = \alpha_{N_{\tau}}^0 H +  r(s_{\tau},a_{\tau})+ \gamma \sum_{i=1}^{N_{\tau}}\alpha_{N_{\tau}}^i P_{s,a}V_{l_i^{\tau}} + 2\sum_{i=1}^{N_{\tau}}\alpha_{N_{\tau}}^i \check{b}_{l_i^{\tau}} 
\end{align}
where we let $N_{\tau}$ be shorthand of $N_{\tau}(s_{\tau},a_{\tau})$ and $l_i^{\tau}$ be shorthand of $l_i^{\tau}(s_{\tau},a_{\tau})$.

Define $\check{\alpha}_{\tau}^i$ be $\alpha_{x}^y$ iff $(s_{\tau},a_{\tau})=(s_i,a_i)$ and  there exists $x,y$ such that $N_{\tau}(s_{\tau},a_{\tau})=x$ and $N_{i}(s_i,a_i)=y$, and otherwise $0$.

The total regret could be decomposed as
\begin{align}
R(t)  & =  t\rho^* - \sum_{\tau=1}^t r(s_{\tau},a_{\tau}) \nonumber 
\\ & \leq  t\rho^* - \sum_{\tau = 1}^{t}\left(  V_{\tau}(s_{\tau}) - \alpha_{N_{\tau}}^0 H -  \gamma \sum_{i=1}^{\tau}\check{\alpha}_{\tau}^i P_{s_{l^{\tau}_i},a_{l^{\tau}_i}}V_{l_i^{\tau}}   -2\sum_{i=1}^{N_{\tau}}\alpha_{N_{\tau}}^i \check{b}_{l_i^{\tau}}  \right) \nonumber 
\\ & \leq t\rho^* - \sum_{\tau = 1}^{t}V_{\tau}(s_{\tau}) + SA(1+\frac{1}{H})H + \gamma \sum_{i=1}^{t}\sum_{\tau \geq i}^{t}\check{\alpha}_{\tau}^i P_{s_{i},a_{i}}V_{i} +  2\sum_{i=1}^{t}\sum_{\tau \geq i}^{t}\check{\alpha}_{\tau}^i \check{b}_{l_i^{\tau}} \nonumber
\\ & \leq  t\rho^* - \sum_{\tau = 1}^{t}V_{\tau}(s_{\tau}) + 2SAH + \gamma \sum_{i=1}^{t}\sum_{\tau \geq i}^{t}\check{\alpha}_{\tau}^i (P_{s_{i},a_{i}}V_{i}-V_i(s_{i+1})) + \sum_{i=1}^{t}\sum_{\tau \geq i}^{t}\check{\alpha}_{\tau}^i V_i(s_{i+1})) + 3\sum_{\tau=1}^t \check{b}_{\tau} \nonumber
\\ & \leq  t\rho^*  -\sum_{\tau=1}^{t} \left( V_{\tau}(s_{\tau})   - \gamma\sum_{i=1}^{\tau}\check{\alpha}_{\tau}^i V_{i}(s_{i+1})   \right) + \gamma \sum_{i=1}^{t}\sum_{\tau \geq i}^{t}\check{\alpha}_{\tau}^i (P_{s_{i},a_{i}}V_{i}-V_i(s_{i+1}))+ 3\sum_{\tau = 1}^t \check{b}_{\tau} + 2SAH     \nonumber 
\\ & \leq  t\rho^*  -\sum_{\tau=1}^{t} \left( V_{\tau}(s_{\tau+1})   - \gamma\sum_{i=1}^{\tau}\check{\alpha}_{\tau}^i V_{i}(s_{i+1})   \right)   + \gamma \sum_{i=1}^{t}\sum_{\tau \geq i}^{t}\check{\alpha}_{\tau}^i (P_{s_{i},a_{i}}V_{i}-V_i(s_{i+1}))+  3\sum_{\tau = 1}^t \check{b}_{\tau} + 3SAH   \ \nonumber 
\\ & \leq t\rho^* - t(1-\gamma)\min_{s}V^*(s) +  \gamma \sum_{i=1}^{t}\sum_{\tau \geq i}^{t}\check{\alpha}_{\tau}^i (P_{s_{i},a_{i}}V_{i}-V_i(s_{i+1}))+  3\sum_{\tau = 1}^t \check{b}_{\tau} + 3SAH \nonumber 
\\ & \leq t\rho^* - t(1-\gamma)\cdot \left(\frac{\rho^*}{1-\gamma}-\spn(h^*)\right) +  \gamma \sum_{i=1}^{t}\sum_{\tau \geq i}^{t}\check{\alpha}_{\tau}^i (P_{s_{i},a_{i}}V_{i}-V_i(s_{i+1}))+  3\sum_{\tau = 1}^t \check{b}_{\tau} + 3SAH \nonumber 
\\ & \leq\gamma \sum_{i=1}^{t}\sum_{\tau \geq i}^{t}\check{\alpha}_{\tau}^i (P_{s_{i},a_{i}}V_{i}-V_i(s_{i+1}))+  3\sum_{\tau = 1}^t \check{b}_{\tau} + 3SAH + t(1-\gamma)\spn(h^*).
\end{align}
In the inequality above, we use several facts as below: (1) $\min_{s} V^*(s)\geq \rho^*/(1-\gamma)-\spn(h^*)$; 
(2)  $\sum_{\tau=1}^t V_{\tau}(s_{\tau})+SAH\geq \sum_{\tau =1}^{t}V_{\tau}(s_{\tau+1})$;
(3)$\gamma \sum_{i\geq \tau }\check{\alpha}_{i}^{\tau}\leq 1$; (4) rearranging the summation.

\medskip
\noindent\textbf{Step 3. Establishing the Final Bound.}
\smallskip
We first bound the term $\sum_{i=1}^{t}\sum_{\tau \geq i}^{t}\check{\alpha}_{\tau}^i (P_{s_{i},a_{i}}V_{i}-V_i(s_{i+1}))$. Let $\zeta_i = P_{s_{i},a_{i}}V_{i}-V_i(s_{i+1})$ and $w_i = \sum_{\tau \geq i}^{t}\check{\alpha}_{\tau}^i$. Then we have that $0 \leq w_i \leq (1+\frac{1}{H})$, $\mathbb{E}[\zeta_i = 0|\mathcal{F}_i]$ and $\mathbb{E}[\zeta^2_i]\leq \mathrm{sp}^2(h^*)$. Then we can rewrite 
$$\sum_{i=1}^{t}\sum_{\tau \geq i}^{t}\check{\alpha}_{\tau}^i (P_{s_{i},a_{i}}V_{i}-V_i(s_{i+1})) = \sum_{i=1}^{t}w_i \zeta_i.$$

Fix $(s,a)$. We consider to bound $\sum_{i=1}^t w_i \mathbb{I}[(s_i,a_i)=(s,a)]\zeta_i$. Using Hoeffding's inequality, we have that
\begin{align}
\sum_{i=1}^{\tau} \mathbb{I}[(s_i,a_i)=(s,a)]\zeta_i \leq 4\mathrm{sp}(h^*)\sqrt{N_{\tau+1}(s,a)\iota }, \forall 1\leq \tau \leq t \label{eq:2111}
\end{align}
holds with probability $1-4t\delta$. Assume this inequality holds.

On the other hand, for $\tau_1<\tau_2$ such that $(s_{\tau_1},a_{\tau_1}) = (s_{\tau_2},a_{\tau_2})$, 
by Lemma~\ref{lemma:tool1}, we have that $w_{\tau_1}\geq w_{\tau_2}$. Let $\tau_1<\tau_2<\ldots< \tau_o$ be the indices such that $(s_{\tau_j},a_{\tau_j})=(s,a)$.
By Abel summation formula, we have that
\begin{align}
  \sum_{i=1}^t w_i  \mathbb{I}[(s_i,a_i)=(s,a)]\zeta_i  & = \sum_{j=1}^{o} w_{\tau_{j}} \zeta_{\tau_i}\nonumber
 \\  & = \sum_{j=2}^{o} (w_{\tau_{j-1}}-w_{\tau_{j}})\sum_{j'=1}^{j-1}\zeta_{\tau_{j'}} + w_{\tau_o}\sum_{j'=1}^o \zeta_{\tau_{j'}}\nonumber
 \\ & \leq w_{\tau_1} \cdot 4\mathrm{sp}(h^*)\sqrt{N_{t+1}(s,a)\iota}\nonumber
 \\ & \leq 8\mathrm{sp}(h^*)\sqrt{N_{t+1}(s,a)\iota}.
\end{align}
 Taking sum over $(s,a)$, we have that
\begin{align}\sum_{i=1}^{t}\sum_{\tau \geq i}^{t}\check{\alpha}_{\tau}^i (P_{s_{i},a_{i}}V_{i}-V_i(s_{i+1})) = \sum_{i=1}^{t}w_i \zeta_i\leq 8 \mathrm{sp}(h^*)\sqrt{SAt\iota}.\label{eq:finalbound2}\end{align}

Next we bound $\sum_{\tau = 1}^t \check{b}_{\tau}$. 
\begin{align}
\sum_{\tau = 1}^{t}\check{b}_{\tau} & \leq     36\sum_{\tau=1}^{t}\sqrt{\frac{3H(N_{\tau}(s_{\tau},a_{\tau})\sum_{s'}P_{s_{\tau},a_{\tau},s'}\tilde{d}^2_{(s_{\tau},s')} + \spn^2(h^*)\iota  )\iota}{N^2_{\tau}(s_{\tau},a_{\tau})}}  \nonumber 
\\ &\quad +    6\sum_{\tau=1}^{t}\sqrt{ \frac{\spn^2(h^*)\iota  )\iota}{N_{\tau}(s_{\tau},a_{\tau})}}\nonumber
+   76SAH\spn(h^*)\iota \nonumber    
\\  & \leq 72\sqrt{HSA\iota \sum_{s,a}N_{t+1}(s,a) \sum_{s'}P_{s,a,s'}\tilde{d}^2_{(s,s')} } +  6\sqrt{SA\iota \sum_{s,a}N_{t+1}(s,a)\mathrm{sp}^2(h^*)} \nonumber 
\\ & \qquad \quad \quad \quad + 118 SAH\spn (h^*)\iota .
\label{eq:bdf}
\end{align}

By Lemma \ref{lemma:con}, with probability at least $1-2t\delta,$ we have
\begin{align}
 & \sum_{s,a}N_{t+1}(s,a)\sum_{s'}P_{s,a,s'}\tilde{d}^2_{(s,s')}\leq 3\sum_{\tau=1}^{t} \tilde{d}^2_{(s_{\tau},s_{\tau+1})} +\spn^2(h^*)\iota .\nonumber
\end{align}
Plugging the inequality above into \eqref{eq:bdf} and noting that $H = 1/(1-\gamma)$, we finish the proof.

\end{proof}

\subsection{Guarantees on the Graph Update Output} \label{sec:pf_graph}

In this subsection, we derive the guarantee on the outputs of the graph update steps. We first analyze the output of the bias difference estimators  (Line~\ref{line:ebf}-\ref{line:omega1} in Algorithm~\ref{alg:ucbref}) in Section~\ref{sec:pf_ebf} and then analyze  $\mathtt{UpdateG}$ (Algorithm~\ref{alg:udtg}) in Section~\ref{sec:pf_UpdateG}.

\subsubsection{Analysis of Bias Difference Estimators} \label{sec:pf_ebf}

\begin{algorithm}[ht]
\caption{$\mathtt{EBF}^{+}$}
\begin{algorithmic}\label{alg:ebf}
\STATE{\textbf{Input :} $(s,s')\in \mathcal{S}\times\mS$}
\STATE{\textbf{Initialization: }$I\leftarrow 0$, $W\leftarrow 0$, $B\leftarrow 0$, $N_1 \leftarrow 0$, $N_2\leftarrow 0$;}
\FOR{$ i= 1,2,\ldots,t$}
\STATE{Observe $s_i$, take some action $a_i$, transit to $s_{i+1}$ and get reward $r_i$; }
\IF{$I = 0 $ and $s_i=s$}
\STATE{$I \leftarrow 1$; $N_2\leftarrow N_2+1$;}
\ENDIF
\IF{$I = 1$ and $s_i = s'$}
\STATE{$I \leftarrow 0$;}
\ENDIF
\STATE{$ W\leftarrow W+ r(s_i,a_i) $; \quad $B \leftarrow B + I \cdot r(s_i,a_i);$ \quad $N_1 \leftarrow N_1 + I$;}
\ENDFOR
\STATE{$\Delta\leftarrow \frac{ B - WN_1/t}{N_2}$; \quad $\omega\leftarrow \frac{ 10\mathrm{sp}(h^*)\sqrt{t\iota}+ 4t(1-\gamma)\mathrm{sp}(h^*) }{N_2}$}
\STATE{\textbf{Output :}$ \{ \Delta,\omega,N_2\}$ }
\end{algorithmic}
\end{algorithm}

We first split the subroutine in Line~\ref{line:ebf}-\ref{line:omega1} out as Algorithm~\ref{alg:ebf}. Here we recall that $\iota = \ln(2/\delta)$, $C=400$ and $R = 3600C^2S^6A^2\ln(T)\mathrm{sp}(h^*)H$. 
The $\mathtt{EBF^+}$ subroutine (Algorithm~\ref{alg:ebf}) estimates the optimal value difference. More precisely, for a fixed $(s,s')$ pair, following a trajectory $\{s_i,a_i,r_i,s_{i+1}  \}_{i=1}^{t}$, Algorithm~\ref{alg:ebf} constructs an estimator $\Delta$ for $V^*(s)-V^*(s')$, and an upper confidence bound $\omega$ for the estimator. We also remark that the algorithm can be implemented in an online fashion, while only consuming constant space. The following lemma states the guarantee on the output of Algorithm~\ref{alg:ebf}.

\begin{lemma}\label{lemma:ebf}
Fix $(s,s')$. Let $\{\Delta, \omega\}$ be the output of Algorithm~\ref{alg:ebf}. Let  $\{s_i,a_i,r_i,s_{i+1} \}_{i=1}^{t}$ be the final trajectory. Then with probability at least $1-2\delta$, we have
\begin{align}
    N(s,s')   \left| V^*(s) - V^*(s')- \Delta \right| &  \leq N(s,s')\left(\omega+ \frac{\mathrm{2Reg}}{N(s,s')} \right) \nonumber
    \\ &= 10\mathrm{sp}(h^*)\sqrt{t\iota}+ 4t(1-\gamma)\mathrm{sp}(h^*)+2\mathrm{Reg},\label{eq:ebf2}
\end{align}
where $\mathrm{Reg} =\sum_{i=1}^t \mathrm{reg}_{s_i,a_i}$, with $\mathrm{reg}_{s,a}: = V^*(s)- (r(s,a)+\gamma P_{s,a}V^*)$, and
$N(s,s') = \sum_{i=1}^t \mathbb{I}[s_i = s,s_{i+1}=s']$.
\end{lemma}

\begin{proof}[Proof of Lemma~\ref{lemma:ebf}]
With a slight abuse of notations, we let $B,W,N_1,N_2,\Delta$ and $\omega$ denote respectively the value of $B,W,N_1,N_2,\Delta$ and $\omega$ at the end of Algorithm~\ref{alg:ebf}.

By definition of $\mathrm{reg}_{s,a}$ ,  we have that
\begin{align}
V^*(s) = \mathrm{reg}_{s,a}+ (r(s,a)+\gamma P_{s,a}V^*). \label{eq:9221}
\end{align}

Define $X_i = V^*(s_i)- V^*(s_{i+1})$. It then holds that
\begin{align}
    \mathbb{E}[X_i|\mathcal{F}_i] &= V^*(s_i) - P_{s_i,a_i}V^* \nonumber \\ 
    &= \mathrm{reg}_{s_i,a_i}+r(s_i,a_i) - \rho^* - (1-\gamma)P_{s_i,a_i}\left(V^* -\frac{\rho^*}{1-\gamma}\cdot \textbf{1} \right).\label{eq:9222}
\end{align}

With a slight abuse of notations, we let $I_i$ be the value of $I$ in the $i$-th step in Algorithm~\ref{alg:ebf}. Then $I_i$ is measurable with respect to $\mathcal{F}_i$. Define $\check{X}_i = X_i I_i$ and $\check{Y}_i = \mathbb{E}[\check{X}_i|\mathcal{F}_i] = I_i \mathbb{E}[X_i|\mathcal{F}_i]$. Then $\{  \check{X}_i - \check{Y}_i    \}_{i=1}^{t}$ is a martingale with respect to $\{\mathcal{F}_i\}_{i=1}^t$. Note that $|X_i|\leq \mathrm{sp}(V^*)\leq 2\mathrm{sp}(h^*)$. Using Freedman's inequality (Lemma~\ref{lemma:freedman}), with probability at least $1-\delta,$ we have
\begin{align}
    \left|  \sum_{i=1}^{m} (\check{X}_i - \check{Y}_i)          \right| \leq  4\mathrm{sp}(h^*)\sqrt{m \iota  } \label{eq:biascon}
\end{align}
for any $1\leq m \leq t$. 
On the other hand, by definition of $\check{X}_i$, we have that 
\begin{align}
      \sum_{i=1}^t (\check{X}_i - \check{Y}_i)  & =  \sum_{l=1}^{t} \big(V^*(s_{i})- V^*(s_{i+1})\big) I_i  - \sum_{i=1}^t \big( \mathrm{reg}_{s_i,a_i} + r(s_i,a_i)-  \rho^* \big)I_i \nonumber
     \\ & \quad \quad + (1-\gamma)\sum_{i=1}^t P_{s_i,a_i} \left(V^* - \frac{\rho^*}{1-\gamma}\textbf{1} \right)I_i \nonumber \\
     & = M_1 - M_2 + M_3, \nonumber
\end{align}
where $M_1 = \sum_{l=1}^{t} \big(V^*(s_{i})- V^*(s_{i+1})\big) I_i$, $M_2 =\sum_{i=1}^t ( \mathrm{reg}_{s_i,a_i} + r(s_i,a_i)-  \rho^* )I_i  $ and $M_3 = (1-\gamma)\sum_{i=1}^t P_{s_i,a_i} \left(V^* - \frac{\rho^*}{1-\gamma}\textbf{1} \right)I_i$. Therefore, we have  $
    |M_1 - M_2 + M_3| \leq 4\mathrm{sp}(h^*)\sqrt{t \iota  }.$

By definition of $I_i$, we have that 
$M_1  = \sum_{i=1}^{n} \big(V^*(s_{\tau_i}) - V^*(s_{\kappa_i+1})\big) ,$
where  $\tau_i$, $\kappa_i$ and $n$ are defined as
\begin{align}
    & \tau_1 = \min\left(\{ j|j \leq t, s_{j}=s  \}\cup \{t\}\right);\nonumber
    \\ & \tau_{i}=  \min\{ j |\kappa_{i-1}< j \leq t, s_j = s \}, \forall i\geq 2; \nonumber
    \\ & \kappa_i = \min \left(\{ j|  \tau_{i}\leq j\leq t  , s_{j+1} = s'\}\cup\{t\}\right), \forall i\geq 1; \nonumber
    \\ & n = \min\{ i | \kappa_i = t\}. \label{eq:deftime}
\end{align}
 In words, $(\tau_i(m),\kappa_i(m))$ denotes the start point and end point of the $i$-th segments in the  $t$-step trajectory. 
 Clearly, $n=N_2$ by definition. 
 By the definitions in \eqref{eq:deftime}, we also have that
\begin{align*}
    |M_1 - n (V^*(s)-V^*(s'))| = |V^*(s')- V^*(s_{\kappa_{n}+1})|\leq 2\mathrm{sp}(h^*).
\end{align*}

On the other hand, by definition of $B,W,N_1,N_2$, we have that
\begin{align}
     B =  \sum_{i=1}^t r(s_i,a_i)I_i, \quad \quad W = \sum_{i=1}^t r(s_i,a_i).\label{eq:9231}
\end{align}
It then follows that
\begin{align}
    M_2 = \sum_{i=1}^t \left( \mathrm{reg}_{s_i,a_i} + r(s_i,a_i)- \rho^*\right) I_i = B + \sum_{i=1}^t \mathrm{reg}_{s_i,a_i} I_i -N_1\rho^* .\label{eq:9232}
\end{align}
 
Defining $\hat{\rho} = W/t$, then with probability at least $1-\delta,$ we have
 
\begin{align}
 t (\hat{\rho}-\rho^*)  &  = \left| \sum_{i=1}^t r(s_i,a_i) - t\rho^*\right| \nonumber
 \\ & = \left|\sum_{i=1}^t \left(  V^*(s_i)- P_{s_i,a_i}V^* - \mathrm{reg}_{s_i,a_i}  \right) + (1-\gamma)\sum_{i=1}^t P_{s,a}V^*- t\rho^*\right|\nonumber
 \\ & \leq \left|\sum_{i=1}^t (V^*(s_i)- P_{s_i,a_i}V^*) \right| +\sum_{i=1}^t \mathrm{reg}_{s_i,a_i}          +    2(1-\gamma)t \mathrm{sp}(h^*)\nonumber
 \\ & 
 \leq \sum_{i=1}^t \mathrm{reg}_{s_i,a_i} + 4\mathrm{sp}(h^*)\sqrt{t\iota} +   2(1-\gamma)t \mathrm{sp}(h^*),\label{eq:9233}
 \end{align}
where the second last inequality follows from Lemma~\ref{lemma:discount_approximate} and the
last inequality holds due to the fact that $\sum_{i=1}^t (V^*(s_{i+1})- P_{s_i,a_i}V^*)$ is a martingale sequence.

At last, by Lemma~\ref{lemma:discount_approximate}, we have that $|M_3|\leq 2t(1-\gamma)\mathrm{sp}(h^*)$.

Recall that $\Delta =\frac{B-WN_1/t}{N_2}$. By \eqref{eq:9231}, \eqref{eq:9232} and \eqref{eq:9233}, with probability at least $1-2\delta$,
\begin{align}
 & \left|    N_2\Delta - N_2 (V^*(s)-V^*(s'))\right| \nonumber
 \\ & = \left|B - N_1\hat{\rho} - N_2 (V^*(s)-V^*(s'))\right| \nonumber
 \\ & = \left|M_2 +N_1 \rho^* - N_1\hat{\rho} - \sum_{i=1}^t\mathrm{reg}_{s_i,a_i}I_i - N_2 (V^*(s)-V^*(s')) \right|\nonumber
 \\ & \leq \left|M_2 - M_1 \right| + 2\mathrm{sp}(h^*) + N_1|\rho^*-\hat{\rho}| + \sum_{i=1}^t \mathrm{reg}_{s_i,a_i} I_i \nonumber
 \\ & \leq \left|M_1-M_2+M_3 \right|+|M_3| + 2\mathrm{sp}(h^*) + N_1|\rho^*-\hat{\rho}| + \sum_{i=1}^t \mathrm{reg}_{s_i,a_i} I_i\nonumber
 \\ & \leq  4\mathrm{sp}(h^*)\sqrt{t\iota}+ 2\mathrm{sp}(h^*) + 2\sum_{i=1}^t \mathrm{reg}_{s_i,a_i} + 4\mathrm{sp}(h^*)\sqrt{t\iota}+ 4t(1-\gamma) \mathrm{sp}(h^*)\nonumber
 \\ & \leq  10\mathrm{sp}(h^*)\sqrt{t\iota}+ 2\sum_{i=1}^t \mathrm{reg}_{s_i,a_i} + 4t (1-\gamma) \mathrm{sp}(h^*).\label{eq:9234}
\end{align}

Dividing both side by $N_2$, we obtain that
\begin{align}
    |\Delta - (V^*(s)-V^*(s')| \leq \frac{10\mathrm{sp}(h^*)\sqrt{t\iota}+4 t(1-\gamma)\mathrm{sp}(h^*)}{N_2} + \frac{2\sum_{i=1}^t \mathrm{reg}_{s_i,a_i}}{N_2}.\nonumber
\end{align}

To prove \eqref{eq:ebf2}, it suffices to note that $\mathrm{Reg}= \sum_{i=1}^t \mathrm{reg}_{s_i,a_i}$, $\omega= \frac{10\mathrm{sp}(h^*)\sqrt{t\iota}+4t(1-\gamma)\mathrm{sp}(h^*)}{N_2}$, and $N(s,s'):=\sum_{i=1}^t \mathbb{I}[s_i=s,s_{i+1}=s']\leq N_2$ . The proof is completed.

\end{proof}

\subsubsection{Analysis of Algorithm~\ref{alg:udtg}} \label{sec:pf_UpdateG}

Next we focus on the analysis of the graph update guarantee. We establish the following guarantee on upper bound $\tilde{d}^k_{s,s'}$ for the value difference estimate error. 
Recall that in Algorithm~\ref{alg:main}, $C =400 $ and  $R = 3600C^2S^6A^2 \ln(T)\mathrm{sp}(h^*)H$.

\begin{lemma}\label{lemma:G} Fix $k$ and assume $R_{k'}\leq R$ for $1\leq k'\leq k-1$. 
Let $\mathcal{E}^{k'}$ and $\{\Delta^{k'}_e,\omega^{k'}_e\}_{e\in \mathcal{E}^k}$ be the output of Algorithm~\ref{alg:udtg} at the end of the $(k'-1)$-th outer epoch.   Define $\mathcal{V}^{k'} := \big\{  h\in \mathbb{R}^{S} \mid \mathrm{sp}(h^*)\leq \mathrm{sp}(h^*), |h(s)-h(s')-\Delta^{k'}_{(s,s')}|\leq \omega^{k'}_{(s,s')},  \forall e= (s,s')\in\mathcal{E}^{k'} \big\}$. 
Then with probability $1-2S^2AT\delta$, $V^* \in 
 \cap_{k'\leq k }\mathcal{V}^{k'}$.
Moreover, letting $\tilde{d}^k_{(s,s')}: = \min_{\mathcal{C} \subset \mathcal{E}^k: \mathcal{C} \text{ connects } s,s} \sum_{e\in \mathcal{C}}\omega^k_{e}$, we have that
\begin{align}
\tilde{d}^k_{s,s'}\leq  &  S\cdot\min_{a}  \Bigg(\frac{12S^3A\ln(T)\cdot \left( 10 \mathrm{sp}(h^*)\sqrt{T\iota}+4T(1-\gamma)\mathrm{sp}(h^*)+ R\right)}{N^k(s,a)P_{s,a,s'}}\nonumber
\\ & \quad \quad \quad \quad \quad \quad \quad \quad \quad \quad \quad \quad + \mathbb{I}\left[\frac{N^k(s,a)P_{s,a,s'}}{36S^2A\ln(T)\iota}\leq 1\right]\cdot \frac{\mathrm{sp}(h^*)}{S}\Bigg).
\end{align}

\end{lemma}

\begin{proof}
We first show that $V^*\in \mathcal{V}^k$ with high probability.   
By Lemma~\ref{lemma:ebf}, noting that $R_{k'}\leq R$ for $1\leq k'\leq k-1$, for any  proper $k'\leq k$ and $(s,s')$, 
 with probability at least $1-2\delta$, it holds that
 \begin{align}
|V^*(s)-V^*(s')-\Delta^{k'}_{(s,s')}|\leq \omega^{k'}_{(s,s')}.
 \end{align}
It then follows  that $V^*\in \cap_{k'\leq k}\mathcal{V}^{k'}$.

Fix $(s,s')$.
Define $\tilde{\omega}^k_{(s,s')}$ be the smallest value of $\omega_{(s,s')}$ in Line~\ref{line:omega} Algorithm~\ref{alg:main} before the $k+1$-th outer epoch. 

Let us use $\mathcal{I}^k(s,a,s')$ to denote the indices of outer epochs where the stopping condition is triggered by $(s,a,s')$  before the $k$-th outer epoch. We define $\mathcal{I}^k(s,s') = \cup_{a}\mathcal{I}^{k}(s,a,s')$ and $\tilde{N}^k(s,s'): = \max_{k'\in  \mathcal{I}^k(s,s')}n^{k'}(s,s')$. By Lemma~\ref{lemma:countbound}, with probability $1-SA\delta$,  we have that $\tilde{N}^k(s,s')\geq  \max_{a}\frac{N^k(s,a)P_{s,a,s'}}{6S^2A\ln(T)}-3\iota$.

By definition of $\tilde{w}^k_{(s,s')}$, we have that
\begin{align}
 \tilde{w}^k_{(s,s')} &  \leq    \frac{10 \mathrm{sp}(h^*)\sqrt{T\iota}+4T(1-\gamma)\mathrm{sp}(h^*)+ R}{\tilde{N}^k(s,s')}     \nonumber
 \\ & \leq \min_{a}  \frac{10 \mathrm{sp}(h^*)\sqrt{T\iota}+4T(1-\gamma)\mathrm{sp}(h^*)+R}{    \frac{N^k(s,a)P_{s,a,s'}}{6 S^3A\ln(T)} - 3\iota  } \nonumber
 \\ &\leq  \min_{a} \frac{12S^3A\ln(T)\cdot \left( 10 \mathrm{sp}(h^*)\sqrt{T\iota}+4T(1-\gamma)\mathrm{sp}(h^*)+ R\right)}{N^k(s,a)P_{s,a,s'}}  \nonumber
 \\ & \quad \quad \quad \quad \quad \quad\quad \quad \quad \quad \quad \quad + \mathbb{I}\left[\frac{N^k(s,a)P_{s,a,s'}}{36S^2A\ln(T)\iota}\leq 1\right]\cdot \mathrm{sp}(h^*), 
\end{align}
where the last inequality is by the fact that $\tilde{w}^k_{(s,s')}\leq 2\mathrm{sp}(h^*)$.

Then it suffices to prove that
\begin{align}
\tilde{d}^k_{(s,s')}\leq S\cdot \tilde{\omega}^k_{(s,s')} \label{eq:biasbound}
\end{align}

By the update rule in Algorithm~\ref{alg:udtg},  in each round (here a round means playing Algorithm~\ref{alg:udtg} for one time in Line~\ref{line:updatee} Algorithm~\ref{alg:main}) we add an edge to $\mathcal{E}$ and delete an edge in a cycle (if there exists) of $\mathcal{E}^k$. Noting that $\mathcal{E}^0$ is initialized as a tree, we learn that $\mathcal{E}^k$ is always a tree $k\geq 1$.

Let $\mathcal{C}^k\subset \mathcal{E}^k$ be the path connecting $s$ with $s'$.
We will prove that, for any $e\in \mathcal{C}^k$, $\omega_e^k \leq \tilde{\omega}^k_{(s,s')}$, where \eqref{eq:biasbound} follows easily.

We prove induction. Assume for some $k'$ and any $e\in \mathcal{C}^{k'}$, $\omega_e^{k'}\leq \tilde{\omega}^k_{(s,a)}$. If $\mathcal{C}^{k'}\subset \mathcal{E}^{k'+1}$, then the conclusion holds trivially. Otherwise, there exists some $e\in \mathcal{C}^{k'}$ and a cycle $\mathcal{C}'$ such that (1) $\mathcal{C}^{k'}/e\subset \mathcal{E}^{k'+1}$;  (2) $\mathcal{C}'/e\subset \mathcal{E}^{k'+1}$ and for any $e'\in \mathcal{C}'/e$, $\omega^{k'+1}_{e'}\leq \omega^{k'}_{e}\leq \tilde{\omega}^k_{(s,s')}$. By definition, we learn that $\mathcal{C}^{k'+1}\subset (\mathcal{C}^{k'}/e) \cap (\mathcal{C'}/e)$, which implies that for any $e\in \mathcal{C}^{k'+1}$, $\omega_{e}^{k'+1}\leq \tilde{\omega}^k_{(s,s')}$.

Let $\tilde{k}<k$  be the index such that $((s,s'),\Delta, \tilde{w}^k_{(s,s')})$ is set as  input in Line~\ref{line:updatee} Algorithm~\ref{alg:main} with some $\Delta \in \mathbb{R}$ in the $\tilde{k}$ rounds.
 Then there are two cases: (1) $(s,s')\in \mathcal{E}^{\tilde{k}+1}$: then we have $\omega^{\tilde{k}}_{(s,s')} \leq  \tilde{w}^{k}_{(s,s')}$; (2) $(s,s')\notin \mathcal{E}^{\tilde{k}+1}$: then $(s,s')$ is deleted in the $\tilde{k}+1$-th round, which implies that
$w_{e}^{\tilde{k}}\leq  \tilde{w}^k_{(s,s')}$ for any $e\in \mathcal{C}^{\tilde{k}}$. 
\end{proof}

\begin{lemma}\label{lemma:countbound} For any fixed  $k \leq K$ and $(s,s')$,
with probability at least $1-\delta$, it holds that  
\begin{align}
\tilde{N}^k(s,s')\geq \max_{a} \frac{N^k(s,a)P_{s,a,s'}}{6K} - 3\iota. \nonumber 
\end{align}
\end{lemma}

\begin{proof}
Let $\tilde{k}$ be such that  $n^{\tilde{k}}(s,a) = \max_{k'\in \mathcal{I}^k(s,s')}n^{k'}(s,a)$. By Lemma~\ref{lemma:con}, with probability $1-2S^2AT\ln(T)\delta$, for any $k'$, it holds that
\begin{align}
n^{k'}(s,s')\geq \frac{1}{3}n^{k'}(s,a)P_{s,a,s'}- {3}\iota .\label{eq:1501}
\end{align}

On the other hand, by the update rule, for any $k'\leq k-1$, $n^{k'}(s,a)\leq 2 n^{\tilde{k}}(s,a)$. Therefore,  we have that
\begin{align}
N^k(s,a) = \sum_{k'=1}^{k-1}n^{k'}(s,a)\leq \sum_{k'=1}^{k-1}2n^{\tilde{k}}(s,a)\leq 2k n^{\tilde{k}}(s,a) \leq 2Kn^{\tilde{k}}(s,a). \label{eq:Nk_upper} 
\end{align}
Recall that $\tilde{N}^k(s,s'): = \max_{k'\in  \mathcal{I}^k(s,s')}n^{k'}(s,s')$. Combining eqs \eqref{eq:1501}-\eqref{eq:Nk_upper} yields the desired inequality. The proof is complete.
\end{proof}

\subsection{Proof of Theorem~\ref{thm:2}} \label{sec:pf_thm_online}

Building on the results established in Sections~\ref{sec:pf_ucb_ref_regret} and \ref{sec:pf_graph}, we are now ready to prove Theorem~\ref{thm:2}.

\begin{proof} [Proof of Theorem~\ref{thm:2}]
Note that the regret in the $k$-th epoch is given by $R_k =t_k \rho^* - \sum_{t\in \mathcal{L}_k} r(s_t,a_t),$ where $\mathcal{L}_k$ is the set of steps in the $k$-th outer epoch.
 
Recall that $R = 3600 C^2 S^6A^2 \ln(T)\mathrm{sp}(h^*)H.$ By induction we will prove the following guarantees on the output of each epoch.

\begin{lemma}\label{lemma:online_main}
For any $k\geq 1$, with probability $1-4kS^2AT\delta$, it holds that 
\begin{itemize}
\item[(\romannumeral1)]  For each $(s,a,s')\in \mathcal{S}\times \mathcal{A}\times \mathcal{S}$, $\tilde{d}^{k}_{(s,s')}N^k(s,a)P_{s,a,s'}\leq 100 S^4A\log_2(T)R$; 
\item[(\romannumeral2)] The regret in the $k$-th epoch satisfies that $R_k\leq R$. 
\end{itemize}
\end{lemma}

By the stopping condition in Algorithm~\ref{alg:ucbref}, $I^k(s,a)$ can never exceeds $\log_2(T)$, which means that each $(s,a)$ could trigger the stopping condition for at most $S\log_2(T)$ times. As result, there are at most  $K=S^2A\log_2(T)$ epochs. 
Using Lemma~\ref{lemma:online_main}, with probability $1-2KSA\delta$ the total regret is bounded by $R(T)\leq KR =  S^2A\log_2(T)R\leq \tilde{O}(S^5A^2\mathrm{sp}(h^*)\iota\sqrt{ T})$. The proof of Theorem~\ref{thm:2} is completed.  
\end{proof}

\medskip
Below we present the proof of Lemma~\ref{lemma:online_main}.

\begin{proof}[Proof of Lemma~\ref{lemma:online_main}]
We will prove by induction.
For $k = 1$, we have  $\tilde{d}^1_{s,s'}\leq  2\mathrm{sp}(h^*)$, where (\romannumeral1) follows. Also noting that $t_1 \leq 2SA$, (\romannumeral2) holds by the fact that $R_1 \leq t_1 \leq 2SA\leq R$.

Now we fix $k \geq 1$, assume that the claims hold for $k'\leq k$ and we will prove that (\romannumeral1) and (\romannumeral2) for $k+1$.

 By Lemma~\ref{lemma:G}, with probability $1-S^2A\delta$, for any $(s,a,s')$, we have that
\begin{align}
\tilde{d}^{k+1}_{s,s'}N^{k+1}(s,a)P_{s,a,s'}\leq   &  12S^4A\ln(T)\cdot  \left(10 \mathrm{sp}(h^*)\sqrt{T\iota}+4T(1-\gamma)\mathrm{sp}(h^*)+ \max_{k'\leq k }R_{k'}\right) \nonumber
\\ & \quad \quad \quad \quad \quad \quad  + N^{k+1}(s,a)P_{s,a,s'} 
\mathbb{I}\left[\frac{N^{k+1}(s,a)P_{s,a,s'}}{36S^2A\ln(T)\iota}\leq 1\right]\cdot \mathrm{sp}(h^*)\nonumber
\\ & \leq 12S^4A\ln(T)\cdot  \left(10 \mathrm{sp}(h^*)\sqrt{T\iota}+4T(1-\gamma)\mathrm{sp}(h^*)+ R\right) \nonumber
\\ & \quad \quad \quad \quad \quad \quad + 36S^2A\ln(T)\iota \mathrm{sp}(h^*)\nonumber
\\ & \leq 100 S^4A\ln(T)R,\label{eq:12221}
\end{align}
where the second equality follows from the induction $(\romannumeral2)$ for $k'\leq k$, and the last equality holds by the definition of $R.$ Therefore, $(\romannumeral1)$ holds for $k+1$.

By Lemma~\ref{lemma:1_1}, with probability $1-2SAT\delta$ the regret in the $k+1$-th outer epoch is bounded by 
\begin{align}
R_{k+1} \leq 400 \left(\sqrt{SAt_{k+1}\spn^2(h^*)\iota} + \sqrt{SAH\iota \sum_{s,a,s'}n^{k+1}(s,a,s')(\tilde{d}^{k+1}_{(s,s')})^2 } +SAH\mathrm{sp}(h^*)\iota+t_{k+1}\mathrm{sp}(h^*)/H\right).\label{eq:rk}
\end{align}

 To proceed, we use the following lemma that bounds the count of $(s,a,s')$ visits in the $k$-th epoch by the total count of $(s,s')$ visits before the $k$-th epoch.
 
\begin{lemma}\label{lemma:2291}
Fix $(s,a,s')$ and $k\geq 1$, with probability $1-\delta$,
\begin{align}
n^{k+1}(s,a,s')\leq \frac{6}{S}N^{k+1}(s,a)P_{s,a,s'}+\iota .\label{eq:12222}
\end{align}

\end{lemma}
\begin{proof}
By Lemma~\ref{lemma:con}, it suffices to prove that $n^{k+1}(s,a)\leq \frac{2N^{k+1}(s,a)}{S}$. By the update rule, suppose $I^{k+1}(s,a)=i$, then we have that $n^{k+1}(s,a)\leq 2^{i}$ and $N^{k+1}(s,a)\geq S2^{i-1}$, where the conclusion follows easily.

\end{proof}

By  \eqref{eq:12221}, \eqref{eq:rk} and Lemma~\ref{lemma:2291}, we have that, with probability $1-3S^2AT\delta$,
\begin{align}
R_{k+1} &  \leq C\left( \sqrt{SAH\iota \sum_{s,a,s'}\left(\frac{6}{S}N^{k+1}(s,a)P_{s,a,s'}+\iota \right)\min \left\{\spn^2(h^*), \frac{10000S^8A^2\ln^2(T)R^2}{(N^{k+1}(s,a)P_{s,a,s'})^2}  \right\} }\right) \nonumber
\\ & \quad \quad \quad \quad \quad \quad \quad \quad \quad   +C\left(\sqrt{SAt_{k+1}\spn^2(h^*)\iota}+ SAH\mathrm{sp}(h^*)\iota+t_{k+1}\spn(h^*)/H\right)\nonumber 
\\ & \leq C\left(\sqrt{ 600S^6A^2\ln(T)H\spn(h^*)\iota R + S^3A^2H\iota^2 \spn^2(h^*)} \right)\nonumber
\\ & \quad \quad \quad \quad \quad \quad \quad \quad \quad \quad +C\left(\sqrt{SAt_{k+1}\spn^2(h^*)\iota}+ SAH\mathrm{sp}(h^*)\iota+t_{k+1}\spn(h^*)/H\right) \label{eq:311}
\end{align}

Recall that $R = 3600C^2 S^6A^2 \ln(T)\mathrm{sp}(h^*)H =3600C^2 S^6A^2 \ln(T)\mathrm{sp}(h^*) \cdot \frac{1}{1-\gamma}$ , we have that 
\begin{align}
\sqrt{SAt_{k+1}\mathrm{sp}(h^*)\iota }, \sqrt{S^2A^2H\iota^2\mathrm{sp}^2(h^*)} , 300S^6A^2\ln(T)H\mathrm{sp}(h^*)\iota, T\mathrm{sp}(h^*)/H \leq \frac{R}{12C^2}. \label{eq:312}
\end{align}
Plugging \eqref{eq:312} into \eqref{eq:311} we learn that  $
R_{k+1} \leq C\left( \frac{6CR}{12C^2}+  \frac{R}{2C} \right) \leq R$. 
Then (\romannumeral2) is proven for $k+1$. 
The proof is completed.

\end{proof}

\subsection{Proof of Proposition~\ref{pro:proj}}\label{sec:proj}

Recall that 
\begin{align*}
\vspace{-0.1in}
\mathcal{C}:&= \big\{v\in \mathbb{R}^S | \mathrm{sp}(v)\leq 2\mathrm{sp}(h^*) ,\big| V^{\mathrm{ref}}(s)- V^{\mathrm{ref}}(s')   - ( v(s)-v(s')) \big|\leq w_{(s,s')},\forall (s,s')\in \mathcal{S}\times \mathcal{S} \big\}
\vspace{-0.1in}
\end{align*}

We first show that, if $v_1,v_2\in \mathcal{C}$, then $v_3:=\max\{v_1,v_2\}\in \mathcal{C}$, where $v_{3}(s) = \max\{v_1(s),v_2(s)\}$ for any $s\in \mathcal{S}$. 

For any $s,s'$, it is easy to verify that 
$$v_3(s)- v_3(s')\leq \max\{ v_1(s)-v_1(s'),v_2(s)-v_2(s')   \}\leq (V^{\mathrm{ref}(s)-V^{\mathrm{ref}}}(s'))+ \omega_{(s,s')}$$
and
$$ v_3(s)- v_3(s') \geq \min \{\ v_1(s)-v_1(s'),v_2(s)-v_2(s') \}\geq  (V^{\mathrm{ref}(s)-V^{\mathrm{ref}}}(s'))+ \omega_{(s,s')}.$$

In a similar way, we can verify that $\mathrm{sp}(v_3)\leq 2\mathrm{sp}(h^*)$. Therefore, we have $v_3\in \mathcal{C}$.

For $v\in \mathbb{R}^{S}$, we define $\mathrm{Proj}_{\mathcal{C}}(v)$ be $\max L(v):=\{ v'\leq v, v'\in \mathcal{C}\}$. Clearly, the set $\{ v'\leq v, v'\in \mathcal{C}\}$ is non-empty.

We first show that the maximum operator over this set is properly defined. Let $u_{s} = \arg\max_{v'\in L(v)} v'(s)$. Let $u_{\mathrm{max}}$ be such that $u_{\mathrm{max}}(s) = \max_{s'\in \mathcal{S}}u_{s'}(s)\geq u_{s}(s) = \max_{v'\in L(v)}v'(s)$ for any $s\in \mathcal{S}$. Then $u_{\mathrm{max}}\in L(v)$ and  $u_{\mathrm{max}}\geq v'$ for any $v'\in L(v)$. So it suffices to choose $\max L(v) = u_{\mathrm{max}}$. In this way, it is easy to see that $\mathrm{Proj}_{\mathcal{C}}(v_1)\geq \mathrm{Proj}_{\mathcal{C}}(v_2)$ when $v_1\geq v_2$, and $\mathrm{Proj}_{\mathcal{C}}(v)=v$ for $v\in \mathcal{C}$. 

To compute $u_{\mathrm{max}}$ it suffices to compute each $u_{\mathrm{max}}(s)$ separately, which could be efficiently solved using linear programming. Considering that there are $O(S^2)$ linear constraints in $\mathcal{C}$, the time cost of $\mathrm{Proj}_{\mathcal{C}}$  is $\tilde{O}(S^{2\times 3}\cdot S)=\tilde{O}(S^7)$.

We remark that a better computational complexity is possible, since each constraint in $\mathcal{C}$ has the form $v(s)-v(s')\leq b_{(s,s')}$.

\subsection{Computational Cost Analysis}\label{sec:computation_cost}
In this section, we proof the proposition below.
\begin{proposition}\label{claim:time_cost}
The time cost of Algorithm~\ref{alg:main} is $\tilde{O}(SAT+S^2T + S^4A)$.
\end{proposition}
\begin{proof}

We first show that  $\mathrm{Proj}_{\mathcal{C}}$  could be implemented within time complexity $O(S^2)$. Define an operator $\Gamma : \mathbb{R}^S\to \mathbb{R}^S$ as $(\Gamma v)(s):=\min\{\min v(s')+\Delta_{(s,s')}+\omega_{(s,s')}, v(s) \}$, where $\Delta_{(s,s')}=0,\omega_{(s,s')}=2\mathrm{sp}(h^*)$ for $(s,s')\notin \mathcal{E}$. It is easy to verify that: (\romannumeral1) $\Gamma v\leq v$; (\romannumeral2) $\Gamma v_1\leq \Gamma v_2$ if $v_1\leq v_2$; (\romannumeral3) $\Gamma v = v$ iff $v\in \mathcal{C}$. We then can prove by induction that: by applying $\Gamma$ for $k$ times on $v$, there are $k$ different states $u_1,u_2,\ldots,u_k$ such that $(\Gamma^k v)(u_i)=v'(u_i), \forall i\in[k]$, where $v'=\mathrm{Proj}_{\mathcal{C}}v$. We thus have $\Gamma^S v = \mathrm{Proj}_{\mathcal{C}}v$. 

We then count the time cost due to $\mathtt{UCB-REF}$ (Algorithm~\ref{alg:ucbref}).  
At each time step, the algorithm needs to update the accumulators and $Q$-function for current state-action pair, which takes time at most $O(1)$. By Proposition~\ref{pro:proj}, the time cost due to $\mathrm{Proj}_{\mathcal{C}}$
 is $\tilde{O}(S^2)$ per step. Therefore, the time cost of $\mathtt{UCB-REF}$  is bounded by $\tilde{O}(S^2T+SAT)$. For $\mathtt{EBF}+$ (Algorithm~\ref{alg:ebf}), we need $O(SAT)$ time to maintain the \emph{value}-difference estimators. At last, the time cost due to updating the graph $\mathcal{G}$ is bounded by $O(S^2)$ each time, where the total time cost is $\tilde{O}(S^4A)$. Putting all together we finish the proof.

 \end{proof}
\section{Proofs of Results in Section~\ref{sec:simulator}}

In this section, we provide detailed proof for the main results in Section~\ref{sec:simulator}. We prove Theorem~\ref{thm:simu_warm_up} in Section~\ref{sec:pf_simu_warmup} and prove Theorem~\ref{theorem:resc} in Section~\ref{sec:pf_simu_refine}. Sections~\ref{sec:pf_lemma_con1}-\ref{sec:pf_lemma_unify_add} provide the missing lemmas and proofs for the results in Section~\ref{sec:pf_simu_warmup}-\ref{sec:pf_simu_refine}.

\subsection{Sample Complexity Analysis for Algorithm~\ref{alg:sc}}  \label{sec:pf_simu_warmup}

\paragraph{Notations.} In the following analysis, $v^{k,l}$ and $q^{k,l}$ represent the value function and $Q$-function at the beginning of the $l$-th round in the $k$-th epoch, respectively.
We define the following two quantities: 
\begin{align}
 \beta^{k,l}(s,a) &= 6\sqrt{\frac{\mathbb{V}(P_{s,a},V^{k,\mathrm{ref}})\iota}{T}} + 6 \sqrt{\frac{\mathbb{V}(P_{s,a},V^{k,\mathrm{ref}}-v^{k,l})\iota}{T_1}} \nonumber\\
&\quad+ 7\frac{(\|V^{k,\mathrm{ref}}-v^{k,l}\|_{\infty}+\|V^{k,\mathrm{ref}}-V^{*}\|_{\infty})\iota}{T_1} +38\epsilon_k;\nonumber 
 \\    \beta^{k}(s,a) &= 6\sqrt{\frac{\mathbb{V}(P_{s,a},V^{k,\mathrm{ref}})\iota}{T}} + 6\sqrt{\frac{\mathbb{V}(P_{s,a},V^{k,\mathrm{ref}}-V^{k})\iota}{T_1}} \nonumber\\
 &\quad+ 7 \frac{(\|V^{k,\mathrm{ref}}-V^{k}\|_{\infty}+\|V^{k,\mathrm{ref}}-V^{*}\|_{\infty})\iota}{T_1}+38\epsilon_k.  \label{eq:defbeta}
\end{align}

\begin{proof}[Proof of Theorem~\ref{thm:simu_warm_up}] 
Recall $ K= \left\lfloor \log_2\left( \min\left\{\frac{1}{(1-\gamma)\spn(h^*)}, \frac{T}{\spn(h^*)}, \sqrt{\frac{T}{64\spn^2(h^*)}}\right\} \right) \right \rfloor$.
Let $L = T^3$. 
We have the two lemmas below.

\begin{lemma}\label{lemma:part1} Under the setting of Theorem~\ref{thm:simu_warm_up}, with probability at least $1-6SAKL\delta$, for each $k\in [K]$, the value function $V^k$ induced policy $\pi^{k}$ is an $O(\epsilon_k)$-optimal policy.
\end{lemma} 

\begin{lemma}\label{lemma:part2} Under the setting of Theorem~\ref{thm:simu_warm_up}, with probability at least $1-12SAKL\delta$, 
 the total number of samples is bounded by $O\left(S^2AT\right)$.
\end{lemma}

Note that $\epsilon_{K}=2^{-K+1}=O\left(\max\left\{(1-\gamma)\spn(h^*),\frac{\spn(h^*)}{T},\sqrt{\frac{\spn^2(h^*)\iota}{T}}\right\}\right)=O\left(\sqrt{\frac{\spn^2(h^*)\iota}{T}}\right)$. Recall that $T = \tilde{O}\left(\frac{\spn^2(h^*)\iota}{\epsilon^2}\right).$ By Lemmas~\ref{lemma:part1} and \ref{lemma:part2}, and replacing $\delta = \frac{\delta}{18SAKT^3}$,
Algorithm~\ref{alg:sc} can find an $\epsilon$-optimal policy with $\tilde{O}\left( \frac{S^2A\spn^2(h^*)\iota }{\epsilon^2}\right)$ samples. The proof of Theorem~\ref{thm:simu_warm_up} is completed.

\end{proof}

The rest of the section is devoted to the proofs of Lemma~\ref{lemma:part1} (see Section~\ref{sec:pf_lemma_part1}) and \ref{lemma:part2} (see Section~\ref{sec:pf_lemma_part2}).

\subsubsection{Proof of Lemma~\ref{lemma:part1}} \label{sec:pf_lemma_part1}

We start with the following lemma, which bounds the Bellman error of $V^k$.

\begin{lemma}\label{lemma:cond1}
With probability at least $1-6SAKL\delta$, it holds that
\begin{align}
    \max_{a} \left( r(s,a)+\gamma P_{s,a} V^{k}  \right) \leq V^{k}(s)\leq \max_{a}\left( r(s,a)+\gamma P_{s,a} V^{k}+ \beta^{k}(s,a) \right)\label{eq:condv}
\end{align}
 any $s\in \mathcal{S}$ and $1\leq k \leq K$.
\end{lemma}

Define $C_1 =C_2=6, C_3=7$ and $C_4=38$. 
Note that setting $1-\gamma = \frac{14000\sqrt{\iota}}{\sqrt{T}}$ and $T_1=37\sqrt{T\iota}$, we can verify that $T_1(1-\gamma)\geq 320(C_1+C_2)^2\iota$. 
Using Lemma~\ref{lemma:unify} (with statement and proof in Section~\ref{sec:lemma_unify}) with $c_1=C_1,c_2=C_2,c_3=C_3$ and $c_4 = C_4$, we obtain that 
\begin{align}
\|V^{k}-V^{k+1}\|_{\infty}\leq \|V^{k}-V^*\|_{\infty}=O\left(\frac{\epsilon_k}{1-\gamma}\right),\label{eq:ref1}
\end{align} and $\pi^k$ is $O(\epsilon_k)$-optimal for $1\leq k \leq K$. We complete the proof of Lemma~\ref{lemma:part1}.

\medskip

Now we provide the proof of Lemma~\ref{lemma:cond1}.
\begin{proof}[Proof of Lemma~\ref{lemma:cond1}] We prove \eqref{eq:condv} by induction. 

Recall that $V^0(s)=\frac{1}{1-\gamma}$ for all $s\in \mS.$ It is easy to verify that $V^0$ satisfies \eqref{eq:condv}. 

Now assume that $\eqref{eq:condv}$ holds for $k = n-1$. For $k=n$, noting that  $ \min \{ v^{k,l-1}(s), \max_{a}q^{k,l}(s,a) \}\leq v^{k,l}(s)\leq v^{k,l-1}(s)$, by Lemma~\ref{lemma:con1} (statement and proof provided in Section~\ref{sec:pf_lemma_con1}), with probability $1-4SAKL\delta$, we have that
\begin{align*}
v^{k,l}(s) & \geq\min\{\max_{a}(r(s,a)+\gamma P_{s,a}v^{k,l}),v^{k,l-1}\}\\
 & \geq\min\{\max_{a}\left(r(s,a)+\gamma P_{s,a}v^{k,l}\right),\max_{a}(r(s,a)+P_{s,a}v^{k,l-1}),v^{k,l-2}\}\\
 & \geq\min\left\{ \max_{a}\left(r(s,a)+\gamma P_{s,a}v^{k,l}\right),v^{k,l-2}\right\} \\
 & \geq...\\
 & \geq\min\left\{ \max_{a}\left(r(s,a)+\gamma P_{s,a}v^{k,l}\right),V^{k-1}(s)\right\} \\
 & \stackrel{(i)}{\geq}\min\left\{ \max_{a}\left(r(s,a)+\gamma P_{s,a}v^{k,l}\right),\max_{a}\left(r(s,a)+\gamma P_{s,a}V^{k-1}\right)\right\} \\
 & \stackrel{(ii)}{\geq}\max_{a}\left(r(s,a)+\gamma P_{s,a}v^{k,l}\right),
\end{align*}
where $(i)$ follows the assumption that \eqref{eq:condv} holds for $k-1$, and $(ii)$ follows from the fact that $V^{k-1}=v^{k,1}\geq v^{k,l}.$

 For the right side of \eqref{eq:condv}, by the update rule (see line~\ref{line:updateq} Algorithm~\ref{alg:sc}) and  Lemma~\ref{lemma:con1} we have
\begin{align}
    v^{k,l}(s)\leq \max_{a}(q^{k,l-1}(s,a)+\epsilon_k)\leq \max_{a}\left( r(s,a)+\gamma P_{s,a} v^{k,l-1}+\beta^{k,l-1}(s,a) \right), \label{eq:upxxx}
\end{align} 
for any $k\geq 1$ and $l\geq 1$. 

Also noting that $V^{k}=v^{k,l_k}=v^{k,l_k-1}$, plugging $l = l_k$ into \eqref{eq:upxxx}, we have that
\begin{align}
    V^k(s)&\leq \max_{a}\left(q^{k,l_k-1}(s,a)+\epsilon_k \right) \nonumber 
    \\ & \leq \max_{a}\left(r(s,a)+\gamma P_{s,a}v^{k,l_k-1}+\beta^{k,l_k-1}(s.a) \right) \nonumber
    \\ & =\max_{a}\left(r(s,a)+\gamma P_{s,a}v^{k,l_k}+\beta^{k,l_k}(s,a) \right) \nonumber
    \\ & =\max_{a}\left(r(s,a)+\gamma P_{s,a}V^{k}+\beta^{k}(s,a) \right).\nonumber  
\end{align} 
The proof is completed.

\end{proof}

\subsubsection{Proof of Lemma~\ref{lemma:part2}} \label{sec:pf_lemma_part2}

We now analyze the number of rounds in each epoch. Note that for each  $1\leq l<l_k$, there exists some state $s$ such that $v^{k,l}(s)\leq v^{k,l-1}(s)-\epsilon_k$. As a result, the number of samples in the $k$-th epoch is at most $\left(\frac{S\|V^{k-1}-V^{k}\|_{\infty}}{\epsilon_k}+1\right)\cdot SAT_1+SAT$. By \eqref{eq:ref1}, assuming the good events in Lemma~\ref{lemma:part1}, we have that
\begin{align}
    \sum_{1\leq k \leq K}\left(\left(\frac{S\|V^{k-1}-V^{k}\|_{\infty}}{\epsilon_k}+1\right)\cdot SAT_1 +SAT\right) & \leq KS^2AT_1\cdot O\left(\frac{1}{(1-\gamma)}\right) +KS^2AT\nonumber
    \\ & =O(KS^2AT).\nonumber\end{align}
The proof is completed.

\subsection{ Sample Complexity Analysis for Algorithm~\ref{alg:scre}} \label{sec:pf_simu_refine}

The refined algorithm for the simulator setting is presented in Algorithm~\ref{alg:scre}. 

\begin{algorithm}
\caption{Refined Monotone $Q$-learning}
\begin{algorithmic}[1]\label{alg:scre}

\STATE{\textbf{Input}:  $S$,  $A$, $T$, $\epsilon$}
\STATE{\textbf{Initialization:} $\iota \leftarrow \ln(2/\delta)$, $T\leftarrow \frac{5\cdot 10^{7}\cdot \mathrm{sp}^2(h^*)\iota }{\epsilon^2} $, $\gamma \leftarrow  1-\frac{14000\sqrt{\iota}}{\sqrt{T}}$,  $T_1\leftarrow 37\sqrt{T\iota}$, $T_2\leftarrow 10\iota$, $V^{0}(s)\leftarrow \frac{1}{1-\gamma}, \forall s\in \mathcal{S}$; $\epsilon_k\leftarrow 2^{-k+1}$ for $k\geq 1$;}
\FOR{$k=1,2,\ldots, K:= \left\lfloor \log_2\left( \min\left\{\frac{1}{(1-\gamma)\spn(h^*)}, \frac{T}{\spn(h^*)}, \sqrt{\frac{T}{64\spn^2(h^*)}}\right\} \right) \right \rfloor$}
\STATE{$V^{k,\mathrm{ref}}\leftarrow V^{k-1}$; }
\STATE{$v^{k,1} \leftarrow V^{k-1}$;}
\STATE{$q^{k,1}\leftarrow Q^{k-1}$;}

\STATE{ $\verb|\\|$ \emph{Compute the reference value}}
\STATE{Sample $(s,a)$ for $T$ times and collect the next states $\{s_i\}_{i=1}^{T}$ for all $(s,a)\in \mathcal{S}\times \mathcal{A}$;\label{line:sample0}}
\STATE{$u^k(s,a)\leftarrow \frac{1}{T}\sum_{i=1}^{T}V^{k,\mathrm{ref}}(s_i)$; $\sigma^k(s,a)\leftarrow \frac{1}{T}\sum_{i=1}^{T}(V^{k,\mathrm{ref}}(s_i))^2$;  }
\STATE{$\mathrm{Trigger}\leftarrow \mathrm{FALSE}$;}
\STATE{$\kappa(s,a)\leftarrow 4\epsilon_k$, $\forall (s,a)$;}
\FOR{$l =1 ,2,\ldots$}
\STATE{$\mathrm{Trigger}\leftarrow \mathrm{FALSE}$;}

\STATE{$\verb|\\|$\emph{Update the $Q$-function}}
\STATE{$q^{k,l}(s,a)\leftarrow q^{k,l-1}(s,a)$ for $(s,a)$ such that $\kappa(s,a)<4\epsilon_k$;}
\FOR{$(s,a)\in \mathcal{S}\times \mathcal{A}$ such that $\kappa(s,a)\geq 4\epsilon_k $}
\STATE{$\kappa(s,a)\leftarrow 0$;}
\STATE{Sample $(s,a)$ for $T_{1}$ times to collect the next states $\{s_i\}_{i=1}^{T_{1}};$\label{line:sample1}}
\STATE{$\zeta^{k,l}(s,a)\leftarrow \frac{1}{T_{1}}\sum_{i=1}^{T_{1}}(v^{k,l}(s_i) - V^{k,\mathrm{ref}}(s_i))$;}
\STATE{$\xi^{k,l}(s,a)\leftarrow \frac{1}{T_{1}}\sum_{i=1}^{T_{1}}\left(v^{k,l}(s_i)-V^{k,\mathrm{ref}}(s_i)\right)^2$;}
\STATE{ Compute $q^{k,l}(s,a)$ by \eqref{eq:updateq};\label{line:updateq}}
\ENDFOR 
\FOR{$s\in \mathcal{S}$}
\IF{$v^{k,l}(s)\geq \max_{a}q^{k,l}(s,a)+\epsilon_k$}
\STATE{ $v^{k,l+1}(s)\leftarrow \max_{a}q^{k,l}(s,a)$;}
\STATE{$\mathrm{Trigger}=\mathrm{TRUE}$;}
\ELSE
\STATE{$v^{k,l+1}(s)\leftarrow v^{k,l}(s)$;}
\ENDIF
\ENDFOR
\FOR{$(s,a)\in \mathcal{S}\times \mathcal{A}$}
\STATE{Sample $(s,a)$ for $T_2$ times to collect the next states $\{\bar{s}_i\}_{i=1}^{T_2}$;\label{line:sample2}}
\STATE{$\kappa(s,a)\leftarrow \kappa(s,a) +\frac{1}{T_2} \sum_{i=1}^{T_2}(v^{k,l}(\bar{s}_i)-v^{k,l+1}(\bar{s}_i))$;}
\ENDFOR

\IF{$!\mathrm{Trigger}$}
\STATE{$V^{k}\leftarrow v^{k,l+1}$;}
\STATE{$Q^k \leftarrow q^{k,l+1}$;\label{line:defq}}
\STATE{\textbf{break};}
\ENDIF

\ENDFOR

\ENDFOR

\end{algorithmic}
\end{algorithm}

Now we provide the proof of Theorem~\ref{theorem:resc}. 

\begin{proof} [Proof of Theorem~\ref{theorem:resc}]
The proof is similar to that of Theorem~\ref{thm:simu_warm_up}. We show that the following two lemmas hold for Algorithm~\ref{alg:scre}. Recall that $T = \tilde{O}\left(\frac{\mathrm{sp}^2(h^*)\iota}{\epsilon^2}\right)$ and $L = T^3$.

\begin{lemma}\label{lemma:part3} Under the setting of Theorem~\ref{theorem:resc}, with probability $1-6SAKL\delta$, $\pi^k$ is an $O(\epsilon_k)$-optimal policy for $1\leq k \leq K$.
\end{lemma} 
 
\begin{lemma}\label{lemma:part4}
Under the setting of Theorem~\ref{theorem:resc}, with probability $1-12SAKL\delta$, the total number of samples is bounded by $O(KSAT+KS^2A\iota)$.
\end{lemma} 

The proofs of Lemma~\ref{lemma:part3} and Lemma~\ref{lemma:part4} can be found in Section~\ref{sec:pf_lemma_part3} and Section~\ref{sec:pf_lemma_part4} respectively.

Recall that  $ K= \left\lfloor \log_2\left( \min\left\{\frac{1}{(1-\gamma)\spn(h^*)}, \frac{T}{\spn(h^*)}, \sqrt{\frac{T}{64\spn^2(h^*)}}\right\} \right) \right \rfloor$. Similarly, we have $\epsilon_{K}=2^{-K+1}=O\left(\max\left\{(1-\gamma)\spn(h^*),\frac{\spn(h^*)}{T},\sqrt{\frac{\spn^2(h^*)\iota}{T}}\right\}\right)=O\left(\sqrt{\frac{\spn^2(h^*)\iota}{T}}\right)$. Note that $T = \tilde{O}\left(\frac{\spn^2(h^*)\iota}{\epsilon^2}\right).$ By Lemmas~\ref{lemma:part3}-\ref{lemma:part4},  and replacing $\delta = \frac{\delta}{18SAKT^3}$, Algorithm~\ref{alg:scre} can find an $O(\epsilon)$-optimal policy, with sample complexity $\tilde{O}\left( \frac{SA\spn^2(h^*)\iota }{\epsilon^2} + \frac{S^2A\spn(h^*)\iota}{\epsilon} \right).$ We complete the proof of Theorem~\ref{theorem:resc}.

\end{proof}
 
\subsubsection{Proof of Lemma~\ref{lemma:part3}} \label{sec:pf_lemma_part3}

\begin{proof} [Proof of Lemma~\ref{lemma:part3}]
We start with the following lemma that bounds the Bellman error of $V^k$ (the proof is provided at the end of this subsection). Recall that $\theta^k = \|V^k-V^*\|_{\infty}$. 

\begin{lemma}\label{lemma:bdbe2}
With probability $1-12SAKL\delta$, it holds that
\begin{align}
    \max_{a}\left(r(s,a)+\gamma P_{s,a}V^k \right)\leq V^k(s)\leq \max_{a}\left(r(s,a)+\gamma P_{s,a}V^k+10\beta^k(s,a) +\frac{20\theta^k}{T_1}\right) \nonumber
\end{align}
for any $s\in \mathcal{S}$ and $1\leq k \leq K$.
\end{lemma}

Define $C_1 =C_2=60, C_3=70$ and $C_4=380$. Recalling that $1-\gamma = \frac{14000\sqrt{\iota}}{\sqrt{T}}$ and $T_1 = 370\sqrt{T\iota}$, it is easy to verify that $T_1(1-\gamma)\geq 320 (C_1+C_2)^2\iota$. Given Lemma~\ref{lemma:bdbe2}, using Lemma~\ref{lemma:unify} (statement and proof provided in Section~\ref{sec:lemma_unify}) with $c_1=C_1,c_2=C_2,c_3=C_3$ and $c_4=C_4$, we obtain that
\begin{align}
\|V^{k}-V^{k+1}\|_{\infty}\leq \|V^k-V^*\|_{\infty} =O\left(\frac{\epsilon_k}{1-\gamma} \right),    \label{eq:ref2}
\end{align}
 and $\pi^k$ is $O(\epsilon_k)$-optimal for $1\leq k \leq K$. The proof is completed.
 \end{proof}
 
 \subsubsection{Restatement and Proof of Lemma~\ref{lemma:bdbe2}}\label{sec:lemma_bdbe2}

\textbf{Lemma~\ref{lemma:bdbe2}}\emph{
With probability $1-2SAKL\delta$, it holds that
\begin{align}
    \max_{a}\left(r(s,a)+\gamma P_{s,a}V^k \right)\leq V^k(s)\leq \max_{a}\left(r(s,a)+\gamma P_{s,a}V^k+10\beta^k(s,a)+\frac{20\theta^k}{T_1} \right) \label{eq:wxr1}
\end{align}
for any $s\in \mathcal{S}$ and $1\leq k \leq K$.
}

\begin{proof} By the update rule, it is easy to see $l^k \leq L=T^3$.

Let $\tilde{l}$ be the smallest value such that $q^{k,\tilde{l}}(s,a)=v^{k,\tilde{l}+1}(s)=v^{k,l^k}(s)=V^k(s)$ for some $a\in \mathcal{A}$. 
Let $\{s^{k,l}_{j}\}_{j=1}^{T_2}$ be the sampled next states for $(s,a)$ in line~\ref{line:sample2} in the $l$-th round, $k$-th epoch.

By the update rule,  we have that
\begin{align}
    \frac{1}{T_2}\sum_{i=\tilde{l}}^{l^k-1}\sum_{j = 1}^{T_2} \left(v^{k,j}(s^{k,i}_{j})-v^{k,j+1}(s^{k,i}_{j})\right) \leq 4\epsilon_k.\label{eq:22454}
\end{align}

Using Lemma~\ref{lemma:con}, with probability $1-l^k\delta$, it follows that
\begin{align}
P_{s,a}(v^{k,l^k}-v^{k,\tilde{l}}) =     \sum_{i=\tilde{l}}^{l^{k}-1}P_{s,a}(v^{k,i+1}-v^{k,i})\leq 4\frac{1}{T_2}\sum_{i=\tilde{l}}^{l^k-1}\sum_{j=1}^{T_2} \left(v^{k,i+1}(s^{k,i}_{j})-v^{k,i}(s^{k,i}_{j})\right) \leq  16\epsilon_k. \label{eq:xxxxx3}
\end{align}

Continuing the computation:
\begin{align}
  V^k(s)  & = q^{k,\tilde{l}}(s)\nonumber
  \\ & \leq r(s,a) + \gamma P_{s,a}v^{k,\tilde{l}} + \beta^{k,\tilde{l}}(s,a) -10\epsilon_k  \nonumber
  \\ & \leq r(s,a)+\gamma P_{s,a}V^k +\beta^{k,\tilde{l}}(s,a) +6\epsilon_k\nonumber
  \\  &  \leq r(s,a) + \gamma P_{s,a}V^k +10\beta^k(s,a)+ \frac{20\theta^{k-1}\iota }{T_1}\label{eq:22456}.
\end{align}
Here \eqref{eq:22456} is by the fact: 
\begin{align}
     & \beta^{k,\tilde{l}}(s,a) \nonumber \\ & = 6\sqrt{\frac{\mathbb{V}(P_{s,a},V^{k-1})\iota }{T}} + 6\sqrt{\frac{\mathbb{V}(P_{s,a},V^{k-1}-v^{k,\tilde{l}})\iota }{T_1}} + \frac{7(|V^{k-1}-v^{k,\tilde{l}}|_{\infty}+\|V^{k-1}-V^*\|_{\infty})\iota}{T_1} + 28\epsilon_k \nonumber
    \\ & \leq 6\sqrt{\frac{\mathbb{V}(P_{s,a},V^{k-1})\iota }{T}} + \frac{20\theta^{k-1}\iota }{T_1} + 28\epsilon_k + 12\sqrt{\frac{\mathbb{V}(P_{s,a},V^{k-1}-V^k)\iota}{T_1}} +12\sqrt{\frac{\mathbb{V}(P_{s,a},V^{k}-v^{k,\tilde{l}})\iota}{T_1}} \label{eq:cxxx1}
    \\ & \leq 20\sqrt{\frac{\mathbb{V}(P_{s,a},V^{k-1})\iota }{T}} + \frac{20\theta^{k-1}\iota }{T_1} + 12\epsilon_k + 40\sqrt{\frac{\mathbb{V}(P_{s,a},V^{k-1}-V^k)\iota}{T_1}}  + 40\sqrt{\frac{ \theta^{k-1}\cdot  P_{s,a}(v^{k,\tilde{l}}-V^k) \iota }{T_1}} \nonumber
    \\ & \leq  20\sqrt{\frac{\mathbb{V}(P_{s,a},V^{k-1})\iota }{T}} + \frac{20\theta^{k-1}\iota }{T_1} + 12\epsilon_k + 40\sqrt{\frac{\mathbb{V}(P_{s,a},V^{k-1}-V^k)\iota}{T_1}} + 160\sqrt{\frac{\theta^{k-1}\epsilon_k\iota }{T_1}}\label{eq:cxxx2}
    \\ & \leq 20\sqrt{\frac{\mathbb{V}(P_{s,a},V^{k-1})\iota }{T}} + \frac{20\theta^{k-1}\iota }{T_1} + 12\epsilon_k + 40\sqrt{\frac{\mathbb{V}(P_{s,a},V^{k-1}-V^k)\iota}{T_1}} + 80\left(\frac{\theta^{k-1}\iota }{T_1}+\epsilon_k\right)\nonumber
    \\ & \leq 10\beta^{k}(s,a) + \frac{20\theta^{k-1}\iota}{T_1}.\nonumber
\end{align}

In \eqref{eq:cxxx1} we use the fact that $\mathrm{Var}(X+Y) \leq 2\mathrm{Var}(X)+2\mathrm{Var}(Y)$ and \eqref{eq:cxxx2} holds by \eqref{eq:xxxxx3}.

As for the left side of \eqref{eq:wxr1}, by Lemma~\ref{lemma:con1} it suffices to note that
\begin{align}
    V^k(s) = q^{k,\tilde{l}}(s,a)\geq \max_{a'}\left( r(s,a')+\gamma P_{s,a'}v^{k,\tilde{l}}\right)\geq \max_{a}\left( r(s,a')+\gamma P_{s,a'}V^k\right).\nonumber
\end{align}
The proof is completed.

\end{proof}

\subsubsection{Proof of Lemma~\ref{lemma:part4}} \label{sec:pf_lemma_part4}

\begin{proof}[Proof of Lemma~\ref{lemma:part4}] 
By \eqref{eq:ref2}, assuming the good events in Lemma~\ref{lemma:part3}, we have that $\|V^{k}-V^*\|_{\infty}\leq =O\left(\frac{\epsilon_k}{1-\gamma}\right)$.  We continue the proof conditioned on this successful event.

Fix $k$ and consider the number of samples in the $k$-th epoch. Firstly, the number of samples in Line~\ref{line:sample0} Algorithm~\ref{alg:scre} is bounded by $O(SAT)$. On the other hand, by the update rule, for each  $1\leq l<l_k$, there exists some state $s$ such that $v^{k,l}(s)\leq v^{k,l-1}(s)-\epsilon_k$.
As a result, we have that $l_k \leq \frac{S\|V^{k-1}-V^k\|_{\infty}}{\epsilon_k}+1$, and the number of samples in Line~\ref{line:sample2} Algorithm~\ref{alg:scre}  is bounded by $S\left(\frac{\|V^{k-1}-V^{k}\|}{\epsilon_k}+1\right)\cdot SAT_2=O\left(\frac{S^2AT_2}{1-\gamma}\right)$. 

To count the number of samples in Line~\ref{line:sample1} Algorithm~\ref{alg:scre}, we have the lemma below.
\begin{lemma}\label{lemma:countw}
For any $(s,a)\in \mathcal{S}\times \mathcal{A}$ and $1\leq k \leq K$, the number of samples of  $(s,a)$  in Line~\ref{line:sample1} Algorithm~\ref{alg:scre} is bounded by $O\left(\frac{SAT_1}{1-\gamma}\right)$ with probability $1-l^k\delta$.
\end{lemma}
\begin{proof} [Proof of Lemma~\ref{lemma:countw}]
 Fix $(s,a)$. Define $x^{k,l}(s,a) = \frac{1}{T_2}\sum_{i=1}^{T_2}\left(v^{k,l}(\bar{s}_i)-v^{k,l+1}(\bar{s}_i) \right)$ where $\{\bar{s}_i\}_i^{T_2}$ are the next states of samples in Line~\ref{line:sample2} for $(s,a)$ in the $l$-th round $k$-th epoch.
 
 Let $\{i_1,i_2,...,i_m\}$ be the indices of rounds where $(s,a)$ is sampled in Line~\ref{line:sample1} and $i_0 = 1$. Then we have that $\sum_{j=i_u}^{i_{u+1}-1}x^{k,j}(s,a)\geq 4\epsilon_k$ for any $1\leq u \leq m$. Therefore $\sum_{j=1}^{l^k}x^{k,l}(s,a)\geq 4m\epsilon_k$. On the other hand, by Lemma~\ref{lemma:con}, with probability $1-\delta$, it holds that $x^{k,l}(s,a)\leq 4P_{s,a}(v^{k,l}-v^{k,l+1})$. 
 Therefore, with probability $1-l^k\delta$, we have that
 \begin{align}
     4m\epsilon_k \leq \sum_{j=1}^{l^k}x^{k,l}(s,a) \leq 4P_{s,a}(V^{k-1}-V^k)= O\left( \frac{\epsilon_{k-1}}{1-\gamma}\right),\nonumber
 \end{align}
 which implies that $m = O\left(\frac{1}{1-\gamma}\right)$. Taking sum over $(s,a)$, we learn that the number of samples in Line~\ref{line:sample1} is bounded by  $O\left(\frac{SAT_1}{1-\gamma}\right)$.
\end{proof}

Putting all together, the sample complexity is bounded by 
\begin{align}
  &  \sum_{1\leq k \leq K}\left( \frac{SAT_1}{1-\gamma}+\frac{S^2AT_2}{1-\gamma}+ SAT\right) =O(KSAT+KS^2A\iota).\nonumber
\end{align}
The proof of Lemma \ref{lemma:part4} is completed.

\end{proof}

\subsection{Statement and Proof of Lemma~\ref{lemma:con1}} \label{sec:pf_lemma_con1}

The following Lemma holds for both Algorithm \ref{alg:sc} and Algorithm \ref{alg:scre}.

\begin{lemma}\label{lemma:con1}
With probability $1-4SAKL\delta$, it holds that 
\begin{align}
   r(s,a) + \gamma  P_{s,a} v^{k,l} \leq q^{k,l}(s,a)\leq r(s,a)+\gamma P_{s,a}v^{k,l} + \beta^{k,l}(s,a)-10\epsilon_{k},\label{eq:cond11}
\end{align}
for any $s\in\mS,a\in\mA,$ $1\leq k \leq K$ and $1\leq l \leq l_k$. 
\end{lemma}  

\begin{proof}
Fix $l$.  Recall the definition of $q^{k,l}(s,a)$ in \eqref{eq:updateq}: 
\begin{align}
q^{k,l}(s,a) &  = r(s,a) + \gamma \left(u^k(s,a) + \zeta^{k,l}(s,a)  \right) \nonumber 
\\ & \quad \quad \quad \quad \quad + 2\sqrt{\frac{3\left(\sigma^k(s,a)-( u^k(s,a))^2\right) \iota }{T}} + 2\sqrt{\frac{3\left( \xi^{k,l}(s,a)-(\zeta^{k,l}(s,a))^2 \right) \iota}{T_1}} \nonumber
\\& \quad \quad \quad \quad \quad + \frac{5 \spn(V^{k,\mathrm{ref}})\iota}{T} + \frac{5\spn(V^{k,\mathrm{ref}}- v^{k,l} )\iota}{T_1}.\nonumber
\end{align}
Recall that $V^{k,\mathrm{ref}}=V^{k-1}$.

By Lemma~\ref{self-norm}, with probability $1-3T\delta$, we have that
\begin{align}
    & T\left|u^k(s,a)-P_{s,a}V^{k-1}\right|\leq 2\sqrt{T
    \mathbb{V}(P_{s,a},V^{k-1})} + 2\iota.\nonumber
\end{align}
Using Lemma~\ref{lemma:con4}, we further have that
\begin{align}
T\left|u^k(s,a)-P_{s,a}V^{k-1}\right| &\leq 2\sqrt{T
    \mathbb{V}(P_{s,a},V^{k-1})\iota } + 2\spn(V^{k-1})\iota\nonumber
    \\ & \leq  2\sqrt{3T
    \mathbb{V}(\hat{P}^k_{s,a},V^{k-1})\iota }+5\spn(V^{k-1})\iota \nonumber
    \\ & = 2\sqrt{3T
    \left(\sigma^k(s,a)-( u^k(s,a))^2\right)\iota }+5\spn(V^{k-1})\iota.\label{eq:kkkk1}
\end{align}
In a similar way, with probability $1-3T\delta$, we have that
\begin{align}
    T_1\left| \zeta^{k,l}(s,a)  - P_{s,a}(v^{k,l}-V^{k-1})\right|\leq 2\sqrt{3T
    \left(\sigma^k(s,a)-( u^k(s,a))^2\right)\iota }+5\spn(V^{k-1}-v^{k,l})\iota.\label{eq:kkkk2}
\end{align}

With \eqref{eq:kkkk1} and \eqref{eq:kkkk2}, we have that $q^{k,l}(s,a)\geq r(s,a)+\gamma P_{s,a}v^{k,l}$. For the other side, by Lemma~\ref{lemma:con4}, we have that with probability $1-6\delta$:
\begin{align}
     & q^{k,l}(s,a)- r(s,a)+\gamma P_{s,a}v^{k,l}\nonumber
     \\ & = 2\sqrt{\frac{ 3\left(\sigma^k(s,a)-( u^k(s,a))^2\right) \iota }{T}} + 2\sqrt{\frac{3\left( \xi^{k,l}(s,a)-(\zeta^{k,l}(s,a))^2 \right) \iota}{T_1}} \nonumber
\\& \quad \quad \quad \quad \quad + \frac{5 \spn(V^{k-1})\iota}{T} + \frac{5\spn(V^{k-1}- v^{k,l} )\iota}{T_1} \nonumber
\\ & \leq 6\sqrt{\frac{ \mathbb{V}(P_{s,a},V^{k-1}) \iota }{T}} + 6\sqrt{\frac{ \mathbb{V}(P_{s,a},V^{k-1}- v^{k,l}) \iota}{T_1}}\nonumber
\\& \quad \quad \quad \quad \quad + \frac{7 \spn(V^{k-1})\iota}{T} + \frac{7\spn(V^{k-1}- v^{k,l} )\iota}{T_1}\nonumber
\\ & \leq 6\sqrt{\frac{ \mathbb{V}(P_{s,a},V^{k-1}) \iota }{T}} + 6\sqrt{\frac{ \mathbb{V}(P_{s,a},V^{k-1}- v^{k,l}) \iota}{T_1}}\nonumber
\\& \quad \quad \quad \quad \quad + \frac{28 \spn(h^*)\iota}{T} + \frac{7\|V^{k-1}- V^{k}\|_{\infty}\iota+7\|V^{k-1}- V^{*}\|_{\infty}\iota }{T_1}.\nonumber
\end{align}
Here the last line is by the fact that $\spn(V^{k-1})\leq \spn(V^*)+\|V^{k-1}-V^*\|_{\infty}\leq 2\spn(h^*)+\|V^{k-1}-V^*\|_{\infty}$ and $\spn(V^{k-1}-v^{k,l})\leq \|V^{k-1}-v^{k,l}\|_{\infty}\leq \|V^{k-1}-V^{k}\|_{\infty}$.

Noting that $\spn(h^*)\iota/T\leq \epsilon_k$, we finish the proof for $l$.

By the update rule, it is easy to derive that 
 $l^k \leq \left(\frac{S\|V^{k-1}-V^{k}\|_{\infty}}{\epsilon_k}+1\right)\leq L$. With a union bound over $l$ we finish the proof.

\end{proof}

\subsection{Statement and Proof of A Unified Lemma~\ref{lemma:unify}} \label{sec:lemma_unify}
Here we state an important lemma that holds for both Algorithm~\ref{alg:sc} and Algorithm~\ref{alg:scre}.

\begin{lemma}\label{lemma:unify} Let $c_1,c_2,c_3$ and $c_4$ be some universal constants.  
Assume $\{V^k\}_{k=0}^{K}$ satisfies that $V^0 = \frac{1}{1-\gamma}\cdot \textbf{1}$, $V^k\leq V^{k-1}$ and
\begin{align}
 \max_{a}\left( r(s,a)+\gamma P_{s,a}V^k \right) \leq    V^k(s) \leq \max_{a}\left(r(s,a)+\gamma P_{s,a}V^k + \lambda^k(s,a) \right), \label{eq:csss}
\end{align}
where 
\begin{align}
    \lambda^k(s,a) &= c_1\sqrt{\frac{\mathbb{V}(P_{s,a},V^{k-1})\iota }{T}} + c_2\sqrt{\frac{\mathbb{V}(P_{s,a}, V^{k-1}-V^k)\iota }{T_1}}\nonumber \\
    &\quad+ c_3\frac{(\|V^{k-1}-V^k\|_{\infty}+\|V^{k-1}-V^*\|_{\infty})\iota}{T_1}+c_4\epsilon_k,  \label{eq:deflambda}
\end{align}
 for $1\leq k \leq K $. Suppose  $T_1(1-\gamma)\geq \max\{320(c_1+c_2)^2,800c_3\}\iota$.
 Then we have 
\begin{align}
    V^*(s) \leq V^k(s)\leq  V^*(s)+ c_5\frac{\epsilon_k}{1-\gamma} \label{eq:www3}
\end{align}
for any $s\in \mathcal{S}$ and $1\leq k \leq K$, where  $c_5= \max\{16c_4,16c_1,1\}$. Moreover, letting $\pi^k$ be the policy induced by $V^k$, we have that $\pi^k$ is $O(\epsilon_k)$-optimal for the average MDP for $1\leq k \leq K$. 

\end{lemma}

\begin{proof}[Proof of Lemma~\ref{lemma:unify}] To facilitate the analysis, we define the discounted and horizon-$t$ state-action visitation measure of a policy $\pi$ for $t\geq 1$, respectively: 
\begin{align}
 &    \tilde{d}_{\gamma}^{\pi}(s,a|\tilde{s}) := \mathbb{E}_{\pi}\left[\sum_{t=1}^{\infty}\gamma^{t-1}\mathbb{I}[(s_t,a_t)=(s,a)]|s_1=\tilde{s} \right],\label{eq:def_d_discount}
 \\ &   d_{t}^{\pi}(s,a|\tilde{s}) := \mathbb{E}_{\pi}\left[\sum_{\tau = 1}^{t}\mathbb{I}[(s_{\tau},a_{\tau})=(s,a)]|s_1=\tilde{s} \right], \label{eq:def_d_finite}
\end{align}
where the expectation is taken over the randomness of the MDP trajectory under the policy $\pi$, i.e., $a_t\sim\pi(\cdot|s_t),s_{t+1}\sim P_{s_t,a_t}.$
Let $\tilde{d}_{\gamma}^{k}$ and ${d}_{t}^k$ be the shorthand of $\tilde{d}_{\gamma}^{\pi^k}$ and $d_{t}^{\pi^k}$, respectively. It is easy to verify that $\tilde{d}_{\gamma}^{\pi^k}$ and $d_{t}^{\pi^k}$ satisfy the following properties:
\begin{align}
    & \sum_{s,a} \tilde{d}_{\gamma}^{\pi}(s,a|\tilde{s}) = \frac{1}{1-\gamma},\nonumber \\
    & \sum_{s,a} d_{t}^{\pi}(s,a|\tilde{s})= t\nonumber,\\
    & \tilde{d}_{\gamma}^{\pi}(s,a|\tilde{s}) = \mathbb{I}[s=\tilde{s}]\pi(a|s) + \gamma \sum_{s',a'} \tilde{d}_{\gamma}(s',a'|\tilde{s}) P_{s',a',s} \pi(a|s), \label{d_discount_difference} \\
    & d_{t}^{\pi}(s,a|\tilde{s}) = \mathbb{I}[s= \tilde{s}]\pi(a|s)+ \sum_{s',a'}d_{t-1}^{\pi}(s',a'|\tilde{s})P_{s',a',s}\pi(a|s)\nonumber
\end{align}

We start with proving \eqref{eq:www3}. Recall that $\pi^k$ is the policy induced by $V^k$. We have the following lemma stating the property of the value function $V^{\pi^k}$ under policy $\pi$ (the proof is provided in Section~\ref{sec:lemma_as1}).

\begin{lemma}\label{lemma:as1} Under the same conditions of Lemma~\ref{lemma:unify}, we have 
\begin{align}
V^{\pi^k}(\tilde{s})\leq V^*(\tilde{s})\leq V^k(\tilde{s}) \leq  V^{\pi^k}(\tilde{s}) +  \min \left\{\sum_{s,a}\tilde{d}_{\gamma}^{k}(s,a|\tilde{s})\lambda^k(s,a), \frac{1}{1-\gamma }\right\}\label{eq:as1}
\end{align} for any $\tilde{s} \in \mathcal{S}$ and $1\leq k\leq n+1$.
\end{lemma}

By Lemma~\ref{lemma:as1}, we have that
\begin{align}
    V^*(\tilde{s})\leq V^k(\tilde{s})\leq V^*(\tilde{s}) + \min \left\{ \sum_{s,a}\tilde{d}_{\gamma}^{k}(s,a|\tilde{s})\lambda^k(s,a), \frac{1}{1-\gamma}\right\},\label{eq:local11}
\end{align}
which proves the left hand side (LHS) of \eqref{eq:www3}.



Next we focus on proving the right hand side (RHS) of \eqref{eq:www3}.
We define 
\begin{align*}
    \theta^k &:=\|V^k-V^*\|_{\infty},\\
    Z_{m}^k(\tilde{s}) &:= \min \left\{ \sum_{s,a}\tilde{d}_{\gamma}^{k}(s,a|\tilde{s})\lambda^m(s,a), \frac{1}{1-\gamma} \right \}, \qquad \mbox{ for } 0\leq m\leq k \leq K.
\end{align*}
By definition of $\lambda^m(s,a)$ in \eqref{eq:deflambda}, we have that
\begin{align}
    Z_m^k(\tilde{s}) & \leq c_1\underbrace{ \sum_{s,a}\tilde{d}_{\gamma}^{k}(s,a|\tilde{s}) \sqrt{\frac{ \mathbb{V}(P_{s,a},V^{m-1})\iota}{T}}}_{\mathbf{Term.1}} + c_2 \underbrace{\sum_{s,a}\tilde{d}_{\gamma}^{k}(s,a|\tilde{s})\sqrt{\frac{\mathbb{V}(P_{s,a},V^{m-1}-V^{m})\iota}{T_1}} }_{\mathbf{Term.2}} \nonumber
    \\ & \quad \quad + c_3 \underbrace{\sum_{s,a}\tilde{d}_{\gamma}^{k}(s,a|\tilde{s}) \frac{(\|V^{m-1}-V^m\|_{\infty}+\|V^{m-1}-V^*\|_{\infty})\iota}{T_1} }_{\mathbf{Term.3}} +c_4\underbrace{\sum_{s,a}\tilde{d}_{\gamma}^{k}(s,a|\tilde{s})\epsilon_m }_{\mathbf{Term.4}}.\label{eq:decomposition} 
\end{align}

Below we bound the four terms in \eqref{eq:decomposition} separately.

\paragraph{ $\mathbf{Term.1}$}
Using Cauchy's inequality, and noting that $\sum_{s,a}\tilde{d}_{\gamma}^{k}(s,a|\tilde{s})=\frac{1}{1-\gamma}$, we have 
\begin{align}
    \mathbf{Term.1} & \leq \sqrt{\sum_{s,a}\tilde{d}_{\gamma}^{k}(s,a|\tilde{s})}\cdot \sqrt{\frac{\sum_{s,a}\tilde{d}_{\gamma}^{k}(s,a|\tilde{s})\mathbb{V}(P_{s,a},V^{m-1})\iota }{T}}\nonumber
    \\ &  = \sqrt{\frac{\iota }{T(1-\gamma)}} \cdot \sqrt{\sum_{s,a}\tilde{d}_{\gamma}^{k}(s,a|\tilde{s})\mathbb{V}(P_{s,a},V^{m-1}) }.\label{eq:term1}
\end{align}

Note that we have proved $V^{m-1}\geq V^*$. Using Lemma~\ref{lemma:com1} (statement and proof deferred to Section~\ref{sec:lemma_com1}), 
we have that
\begin{align}
   & \sum_{s,a}\tilde{d}_{\gamma}^{k}(s,a|\tilde{s})\mathbb{V}(P_{s,a},V^{m-1}) \leq 8\spn^2(h^*) + 5(\theta^{m-1})^2 +\frac{6\spn(h^*)}{1-\gamma} + 2\theta^{m-1}(V^*(\tilde{s})-V^{\pi^k}(\tilde{s})).\nonumber
\end{align}
By Lemma~\ref{lemma:as1}, we have that $V^*(\tilde{s})-V^{\pi^k}(\tilde{s})\leq V^k(\tilde{s}) - V^{\pi^k}(\tilde{s})\leq Z_k^k(\tilde{s})$. Because $\mathrm{sp}(h^*)\leq \sqrt{T}$, we thus obtain that
\begin{align}
    \mathbf{Term.1}& \leq \sqrt{\frac{\iota }{T(1-\gamma)}}\cdot \sqrt{ 8\spn^2(h^*) + 5(\theta^{m-1})^2 +\frac{6\spn(h^*)}{1-\gamma} + 2\theta^{m-1}Z_k^k(\tilde{s})  } \nonumber
    \\ & \leq  \sqrt{\frac{\iota }{T(1-\gamma)}}\cdot \sqrt{ 5(\theta^{m-1})^2 +\frac{14\spn(h^*)}{1-\gamma} + 2\theta^{m-1}Z_k^k(\tilde{s})  }    .\label{eq:bdterm1}
\end{align}

\paragraph{$\mathbf{Term.2}$} Using Cauchy's inequality, and noting that $\sum_{s,a}\tilde{d}_{\gamma}^{k}(s,a|\tilde{s}) = \frac{1}{1-\gamma}$, we have that
\begin{align}
    \mathbf{Term.2}&\leq \sqrt{\sum_{s,a}\tilde{d}_{\gamma}^{k}(s,a|\tilde{s})}\cdot \sqrt{\frac{\sum_{s,a}\tilde{d}^k_{\gamma}(s,a|\tilde{s}) \mathbb{V}(P_{s,a},V^{m-1}-V^m) \iota}{T_1}} \nonumber
    \\ & = \sqrt{\frac{\iota}{T_1(1-\gamma)}}\cdot \sqrt{\sum_{s,a}\tilde{d}^k_{\gamma}(s,a|\tilde{s}) \mathbb{V}(P_{s,a},V^{m-1}-V^m)  }.\nonumber
\end{align}

By Lemma~\ref{lemma:com2} (statement and proof deferred to Section~\ref{sec:lemma_com2}), we have that
\begin{align*}
    \sum_{s,a}\tilde{d}^k_{\gamma}(s,a|\tilde{s}) \mathbb{V}(P_{s,a},V^{m-1}-V^m) &  \leq 4(\theta^{m-1})^2 +2\theta^{m-1}\left( V^*(\tilde{s}) -V^{\pi^k}(\tilde{s}) \right).
\end{align*}
Recalling that that $V^*(\tilde{s})-V^{\pi^k}(s)\leq V^k(\tilde{s})-V^{\pi^k}(\tilde{s})\leq Z_k^k(\tilde{s})$, we thus have 
\begin{align}
    \mathbf{Term.2} & \leq \sqrt{\frac{\iota }{T_1(1-\gamma)}}\cdot \sqrt{ 4(\theta^{m-1})^2 +2\theta^{m-1}Z_k^k(\tilde{s})}.\label{eq:bdterm2} 
\end{align}

\paragraph{$\mathbf{Term.3}$}

Note that $V^{m-1}\geq V^m\geq V^*.$ By the definition of $\theta^{m-1}=:\|V^{m-1}-V^*\|_{\infty}$, and the fact that $\sum_{s,a}\tilde{d}_{\gamma}^k(s,a|\tilde{s}) = \frac{1}{1-\gamma}$, we have 
\begin{align}
    \mathbf{Term.3}\leq \frac{\iota (\|V^{m-1}-V^m\|_{\infty}+\|V^{m-1}-V^*\|_{\infty})}{T_1(1-\gamma)} \leq \frac{2\theta^{m-1}\iota}{T_1(1-\gamma)}.\label{eq:bdterm3}
\end{align}

\paragraph{{$\mathbf{Term.4}$}}
By definition, we have that
\begin{align}
    \mathbf{Term.4}\leq \frac{\epsilon_m}{1-\gamma}.\label{eq:bdterm4}
\end{align}

Combining \eqref{eq:decomposition} \eqref{eq:bdterm1}, \eqref{eq:bdterm2},\eqref{eq:bdterm3} and \eqref{eq:bdterm4}, we obtain that 

\begin{align}
    Z_m^k(\tilde{s}) & \leq c_1 \sqrt{\frac{\iota }{T(1-\gamma)}}\cdot \sqrt{\left( 5(\theta^{m-1})^2 +\frac{14\spn(h^*)}{1-\gamma} + 2\theta^{m-1}Z_k^k(\tilde{s}) \right) } \nonumber
    \\ & \quad + c_2 \sqrt{\frac{\iota }{T_1(1-\gamma)}}\cdot \sqrt{ 4(\theta^{m-1})^2 +2\theta^{m-1}Z_k^k(\tilde{s})}\nonumber
    \\ & \quad + c_3\frac{2\theta^{m-1}\iota}{T_1(1-\gamma)} +c_4\frac{\epsilon_m}{1-\gamma} .\label{eq:boundz1}
\end{align}

Let $\overline{Z}_m^k = \max_{\tilde{s}}Z_m^k(\tilde{s})$. By \eqref{eq:local11}, $\theta^{m'} = \max_{\tilde{s}}(V^{m'}(\tilde{s})-V^*(\tilde{s}))\leq \max_{\tilde{s}}Z_{m'}^{m'}(\tilde{s})= \overline{Z}_{m'}^{m'}$. Continuing the computation, we have that
\begin{align}
     Z_m^k(\tilde{s}) & \leq c_1 \sqrt{\frac{\iota }{T(1-\gamma)}}\cdot \sqrt{\left( 5(\overline{Z}^{m-1}_{m-1})^2 +\frac{14\spn(h^*)}{1-\gamma} + 2\overline{Z}^{m-1}_{m-1}\overline{Z}_k^k \right) } \nonumber
    \\ & \quad + c_2 \sqrt{\frac{\iota }{T_1(1-\gamma)}}\cdot \sqrt{ 4(\overline{Z}_{m-1}^{m-1})^2 +2\overline{Z}_{m-1}^{m-1}\overline{Z}_k^k}\nonumber
    \\ & \quad + c_3\frac{2\overline{Z}_{m-1}^{m-1}\iota}{T_1(1-\gamma)} +c_4\frac{\epsilon_m}{1-\gamma} \nonumber
\end{align}
for any $\tilde{s}\in \mathcal{S}$. As a result, we obtain that 
\begin{align}
     \overline{Z}_m^k & \leq c_1 \sqrt{\frac{\iota }{T(1-\gamma)}}\cdot \sqrt{\left( 5(\overline{Z}^{m-1}_{m-1})^2 +\frac{14\spn(h^*)}{1-\gamma} + 2\overline{Z}^{m-1}_{m-1}\overline{Z}_k^k \right) } \nonumber
    \\ & \quad + c_2 \sqrt{\frac{\iota }{T_1(1-\gamma)}}\cdot \sqrt{ 4(\overline{Z}_{m-1}^{m-1})^2 +2\overline{Z}_{m-1}^{m-1}\overline{Z}_k^k}\nonumber
    \\ & \quad + c_3\frac{2\overline{Z}_{m-1}^{m-1}\iota}{T_1(1-\gamma)} +c_4\frac{\epsilon_m}{1-\gamma} .\label{eq:boundz2}
\end{align}

Below we prove that $\overline{Z}_m^k \leq c_5 \frac{\epsilon_m}{1-\gamma}$ for $0\leq m\leq k \leq K$. We prove by induction on $k$. 

\paragraph{Base case:} For $k = m = 0$, we have that $\overline{Z}_m^k\leq \frac{1}{1-\gamma}\leq c_5 \frac{\epsilon_0}{1-\gamma}$. 

\paragraph{Induction step:} Suppose that $\overline{Z}_m^k \leq c_5 \frac{\epsilon_m}{1-\gamma}$ for $0\leq m\leq k \leq n<K$. It suffices to show that $\overline{Z}_m^k \leq c_5 \frac{\epsilon_m}{1-\gamma}$ holds for $0\leq m \leq k=n+1.$
We consider the two cases $m=k$ and $m<k$ separately below.

{\bf Case 1: $m=k=n+1$.} Plugging $k =m= n+1$ into \eqref{eq:boundz2}, we have that 
\begin{align}
\overline{Z}_{n+1}^{n+1}\leq &  c_1 \sqrt{\frac{\iota}{T(1-\gamma)}}\sqrt{  \frac{5c_5^2\epsilon_n^2 }{(1-\gamma)^2}+\frac{14\spn(h^*)}{1-\gamma}+ \frac{2c_5\epsilon_n}{1-\gamma}\overline{Z}_{n+1}^{n+1} }\nonumber
\\ & \quad + c_2 \sqrt{\frac{\iota }{T_1(1-\gamma)}}\cdot \sqrt{\frac{4c_5^2\epsilon_n^2 }{(1-\gamma)^2} + \frac{2c_5\epsilon_n}{1-\gamma}\overline{Z}_{n+1}^{n+1}}\nonumber
\\ & \quad + \frac{2c_3c_5 \epsilon_n\iota}{T_1(1-\gamma)^2} +\frac{c_4\epsilon_{n+1}}{1-\gamma}. \nonumber\\
&\leq  c_1 \sqrt{\frac{\iota}{T(1-\gamma)}}  \left(\sqrt{\frac{5c_5^2\epsilon_n^2 }{(1-\gamma)^2}}+\sqrt{\frac{14\spn(h^*)}{1-\gamma}}+ \sqrt{\frac{2c_5\epsilon_n}{1-\gamma}\overline{Z}_{n+1}^{n+1}} \right)\nonumber
\\ & \quad + c_2 \sqrt{\frac{\iota }{T_1(1-\gamma)}}\cdot \left(\sqrt{\frac{4c_5^2\epsilon_n^2 }{(1-\gamma)^2}} + \sqrt{\frac{2c_5\epsilon_n}{1-\gamma}\overline{Z}_{n+1}^{n+1}}\right)\nonumber
\\ & \quad + \frac{2c_3c_5 \epsilon_n\iota}{T_1(1-\gamma)^2} +\frac{c_4\epsilon_{n+1}}{1-\gamma},
\label{eq:local22}
 \end{align}
where the second inequality follows from the fact that $\sqrt{a+b}\leq \sqrt{a}+\sqrt{b}.$

Solving \eqref{eq:local22} for $\overline{Z}_{n+1}^{n+1}$, we have
\begin{align*}
    \overline{Z}_{n+1}^{n+1}&  \leq 
    \left( c_1\sqrt{\frac{2c_5\epsilon_n\iota}{T(1-\gamma)^2}} +c_2\sqrt{\frac{2c_5\epsilon_n\iota}{T_1(1-\gamma)^2}} \right)^2 \nonumber
    \\ & \quad  +         2\left(c_1\sqrt{\frac{5c_5^2\epsilon_n^2\iota}{T(1-\gamma)^3}}+c_1\sqrt{\frac{14\spn(h^*)\iota }{T(1-\gamma)^2}} +c_2\sqrt{\frac{4c_5^2\epsilon_n^2\iota}{T_1(1-\gamma)^3}}+ \frac{2c_3c_5\epsilon_n\iota}{T_1(1-\gamma)^2}+\frac{c_4\epsilon_{n+1}}{1-\gamma}\right)\\ 
    & =   \frac{2c_5\epsilon_n\iota}{(1-\gamma)^2}  \left( \frac{c_1}{\sqrt{T}}+\frac{c_2}{\sqrt{T_1}} \right)^2 \nonumber
    \\ & \quad  + \frac{2c_5\epsilon_n}{1-\gamma}
    \left(c_1\sqrt{\frac{5\iota}{T(1-\gamma)}}+c_1\sqrt{\frac{14\spn(h^*)\iota }{Tc_5^2\epsilon_n^2}} +c_2\sqrt{\frac{4\iota}{T_1(1-\gamma)}}+ \frac{2c_3\iota}{T_1(1-\gamma)}+\frac{c_4\epsilon_{n+1}}{c_5\epsilon_n}\right). 
\end{align*}
Because $\epsilon_{n+1} = \frac{\epsilon_n}{2}$ and $T_1\leq T$,
we have
\begin{align}
    \frac{\overline{Z}^{n+1}_{n+1}(1-\gamma)}{c_5\epsilon_{n+1}} & \leq   \frac{4\iota}{1-\gamma}  \left( \frac{c_1}{\sqrt{T}}+\frac{c_2}{\sqrt{T_1}} \right)^2 \nonumber
    \\ & \quad  + 4
    \left(c_1\sqrt{\frac{5\iota}{T(1-\gamma)}}+c_1\sqrt{\frac{14\spn(h^*)\iota }{Tc_5^2\epsilon_n^2}} +c_2\sqrt{\frac{4\iota}{T_1(1-\gamma)}}+ \frac{2c_3\iota}{T_1(1-\gamma)}+\frac{c_4}{2c_5}\right) \nonumber\\
    & \leq   \frac{8\iota}{1-\gamma}  \left( \frac{c^2_1}{T}+\frac{c^2_2}{T_1} \right) \nonumber
    \\ & \quad  + 4
    \left(c_1\sqrt{\frac{5\iota}{T(1-\gamma)}}+c_1\sqrt{\frac{14\spn(h^*)\iota }{Tc_5^2\epsilon_n^2}} +c_2\sqrt{\frac{4\iota}{T_1(1-\gamma)}}+ \frac{2c_3\iota}{T_1(1-\gamma)}+\frac{c_4}{2c_5}\right) \nonumber\\
    & \leq   \frac{8(c_1^2+c_2^2)\iota}{T_1(1-\gamma)}  + 4(c_1+c_2)\sqrt{\frac{5\iota}{T_1(1-\gamma)}}+ 4c_1\sqrt{\frac{14\spn(h^*)\iota }{Tc_5^2\epsilon_n^2}}+\frac{8c_3\iota}{T_1(1-\gamma)}+\frac{2c_4}{c_5}\nonumber\\
     & \leq   \frac{8(c_1+c_2)^2\iota}{T_1(1-\gamma)}  + 4(c_1+c_2)\sqrt{\frac{5\iota}{T_1(1-\gamma)}}+ \frac{2c_1}{c_5}+\frac{8c_3\iota}{T_1(1-\gamma)}+\frac{2c_4}{c_5}\nonumber\\
  & \leq  \frac{1}{40}+\frac{1}{2}+\frac{1}{8}+\frac{9}{40}+\frac{1}{8}=1,\label{eq:local44} 
\end{align}
where the second last inequality follows from the fact that $T\epsilon_{n}^2\geq 56\spn(h^*)\iota$ and $c_5\geq 1$, the last inequality holds due to the choice of $c_5\geq \max\{ 16c_1,16c_4\}$ and $T_1(1-\gamma) \geq \max\{320(c_1+c_2)^2, 320c_3/9\}\iota.$ Therefore, we have $\overline{Z}_{n+1}^{n+1}\leq \frac{c_5\epsilon_{n+1}}{1-\gamma}$.

{\bf Case 2: $m<k=n+1$.} Using the inductive assumption $\overline{Z}_{m-1}^{m-1}\leq c_5\frac{\epsilon_{m-1}}{1-\gamma}$ in \eqref{eq:boundz2}, we have that
\begin{align*}
\overline{Z}_{m}^{n+1} & \leq c_{1}\sqrt{\frac{\iota}{T(1-\gamma)}}\cdot\sqrt{\frac{5c_{5}^{2}\epsilon_{m-1}^{2}}{(1-\gamma)^{2}}+\frac{14\spn(h^{*})}{1-\gamma}+2\frac{c_{5}^{2}\epsilon_{m-1}\epsilon_{n+1}}{(1-\gamma)^{2}}}\\
 & \quad+c_{2}\sqrt{\frac{\iota}{T_{1}(1-\gamma)}}\cdot\sqrt{4\frac{c_{5}^{2}\epsilon_{m-1}^{2}}{(1-\gamma)^{2}}+2\frac{c_{5}^{2}\epsilon_{m-1}\epsilon_{n+1}}{(1-\gamma)^{2}}}+\frac{2c_{3}c_{5}\epsilon_{m-1}\iota}{T_{1}(1-\gamma)^{2}}+c_{4}\frac{\epsilon_{m}}{1-\gamma}\\
 & \leq c_{1}\sqrt{\frac{\iota}{T(1-\gamma)}}\cdot\left(\sqrt{\frac{7c_{5}^{2}\epsilon_{m-1}^{2}}{(1-\gamma)^{2}}}+\sqrt{\frac{14\spn(h^{*})}{1-\gamma}}\right)\\
 & \quad+c_{2}\sqrt{\frac{\iota}{T_{1}(1-\gamma)}}\cdot\sqrt{\frac{6c_{5}^{2}\epsilon_{m-1}^{2}}{(1-\gamma)^{2}}}+\frac{2c_{3}c_{5}\epsilon_{m-1}\iota}{T_{1}(1-\gamma)^{2}}+c_{4}\frac{\epsilon_{m}}{1-\gamma}\\
 & =\frac{c_{5}\epsilon_{m-1}}{1-\gamma}\cdot\left(c_{1}\sqrt{\frac{7\iota}{T(1-\gamma)}}+c_{1}\sqrt{\frac{14\spn(h^{*})\iota}{Tc_{5}^{2}\epsilon_{m-1}^{2}}}\right)\\
 & \quad+\frac{c_{5}\epsilon_{m-1}}{1-\gamma}\cdot\left(c_{2}\sqrt{\frac{6\iota}{T_{1}(1-\gamma)}}+\frac{2c_{3}\iota}{T_{1}(1-\gamma)}+\frac{c_{4}\epsilon_{m}}{c_{5}\epsilon_{m-1}}\right)
\end{align*}

Recalling that $\epsilon_{m}=\frac{\epsilon_{m-1}}{2}$, $T_1\leq T$,  $T\epsilon_{m}^2\geq 56\spn(h^*)$, and the choice of $c_5\geq \max\{ 16c_1,16c_4\}$ and $T_1(1-\gamma) \geq \max\{320(c_1+c_2)^2, 320c_3/9\}\iota$, using similar arguments in \eqref{eq:local44}, we have that
\begin{align}
\frac{\overline{Z}_{m}^{n+1}(1-\gamma)}{c_{5}\epsilon_{m}} & \leq2c_{1}\sqrt{\frac{7\iota}{T(1-\gamma)}}+2c_{1}\sqrt{\frac{14\spn(h^{*})\iota}{Tc_{5}^{2}\epsilon_{m-1}^{2}}} \nonumber\\
 & \quad+2c_{2}\sqrt{\frac{6\iota}{T_{1}(1-\gamma)}}+\frac{4c_{3}\iota}{T_{1}(1-\gamma)}+\frac{c_{4}}{c_{5}}\nonumber\\
 & \leq2(c_{1}+c_{2})\sqrt{\frac{7\iota}{T_{1}(1-\gamma)}}++2c_{1}\sqrt{\frac{14\spn(h^{*})\iota}{Tc_{5}^{2}\epsilon_{m-1}^{2}}}+\frac{4c_{3}\iota}{T_{1}(1-\gamma)}+\frac{c_{4}}{c_{5}}\nonumber\\
 & \leq\sqrt{\frac{7}{80}}+\frac{c_{1}}{c_{5}}+\frac{9}{80}+\frac{c_{4}}{c_{5}}\nonumber\\
 & \leq\sqrt{\frac{7}{80}}+\frac{1}{16}+\frac{9}{80}+\frac{1}{16} <1.\nonumber
\end{align}
We thus have $\overline{Z}_{m}^{n+1}\leq \frac{c_5 \epsilon_{m}}{1-\gamma}.$

Putting together, we have proved that $\overline{Z}_m^k \leq c_5 \frac{\epsilon_m}{1-\gamma}$ for $0\leq m\leq k \leq K$ by induction. Therefore, we have completed the proof for the RHS of \eqref{eq:www3}.

Given \eqref{eq:www3}, next we prove that $V^k$ induced policy $\pi_k$ is an $O(\epsilon_k)$-optimal policy for the average MDP. To this end, we consider a finite-horizon-$T$ problem with the same reward function and transition kernel as the original average MDP. For a finite-horizon policy $\pi=(\pi^{1},\pi^{2},\ldots,\pi^{T})$ (where $\pi^{t}$ is a mapping $\mS\rightarrow \Delta(\mA)$), we define  $V^{\pi}_T$ as the value function  under $\pi$, i.e., 
\begin{align}
    V^{\pi}_T(s) = \E_{a_t\sim~\pi(\dot|s_t),s_{t+1}\sim P_{s_t,a_t}} \left[ \sum_{t=1}^T r(s_t,\pi(s_t)) | s_t=s \right]. \label{eq:def_VT}
\end{align}
Let $\bar{\pi}^*$ be the optimal policy finite-horizon policy with value function $V^*_T.$ The following lemma connects the optimal reward of the infinite-horizon setting and the finite-horizon setting.

\begin{lemma}[Lemma 13 of \cite{wei2021linear}] \label{lemma:finite_approximate}
For each $s\in \mS$ and $t\geq 1$, $|t\rho^* -V^*_t(s)|\leq \spn(h^*)$. 
\end{lemma}

We bound the gap between $V_{T}^{\pi^k}(\tilde{s})$ and $V^*_{T}(\tilde{s})$, as stated in the following lemma (the proof is deferred to Section~\ref{sec:lemma_as2}).

\begin{lemma}\label{lemma:as2}
Under the same conditions of Lemma~\ref{lemma:unify}, it holds that
\begin{align}
     V_{T}^*(\tilde{s}) & \leq V_{T}^{\pi^k}(\tilde{s})+ \min\left\{ \sum_{s,a}d_{T}^{k}(s,a|\tilde{s})\lambda^k(s,a) +\left(T(1-\gamma)+2\right)\spn(h^*)+\frac{c_5\epsilon_k}{1-\gamma}, T\right\} \label{eq:561}
     \\ & \leq V_{T}^{\pi^k} + O\left(T\epsilon_k+(T(1-\gamma)+2)\spn(h^*)       +\frac{c_5\epsilon_k}{1-\gamma}\right) \nonumber
     \\ &  = V_{T}^{\pi^k}+O(T\epsilon_k),\label{eq:562}
\end{align}
\end{lemma}
where the last inequality is by the fact that $\epsilon_k\geq (1-\gamma)\mathrm{sp}(h^*)$.

With Lemma~\ref{lemma:finite_approximate} and \ref{lemma:as2}, we are ready to show that $\pi^k$ is $O(\epsilon_k)$-optimal. Let $\rho^{k}$ be the average reward of $\pi^k$. Let $\mathcal{D}^k_{i}$ denote the state-distribution of the $(iT+1)$-th step following $\pi^k$ starting from $\tilde{s}$.
By Lemma~\ref{lemma:as2}, we have that
\begin{align}
    \rho^k  & =\lim_{N\to \infty} \frac{V^{\pi^k}_{NT}(\tilde{s})}{NT}\nonumber = \lim_{N\to \infty}\frac{\sum_{i=0}^{N-1} \mathbb{E}_{s\sim \mathcal{D}_i^k}[V_{T}^{\pi^k}(s)]  }{NT}\nonumber
    \\ & \geq  \lim_{N\to \infty}\frac{\sum_{i=0}^{N-1} \mathbb{E}_{s\sim \mathcal{D}_i^k}[V_{T}^{*}(s)-O(T\epsilon_k)]  }{NT}\nonumber
    \\ & \stackrel{(a)}{\geq}  \lim_{N\to \infty}\frac{N\left(T\rho^*-O(T\epsilon_k) -\spn(h^*) \right)}{NT}\nonumber
    \\ & \geq \rho^* - O\left(\epsilon_k + \frac{\spn(h^*)}{T}\right) \nonumber
    \\ & \stackrel{(b)}{\geq} \rho^* -O(\epsilon_k),\nonumber
 \end{align}
 where $(a)$ follows from Lemma~\ref{lemma:finite_approximate}, and $(b)$ follows from the fact that $\epsilon_k \geq \epsilon_K = O\left( \frac{\spn(h^*)}{T}  \right).$
 
 Therefore, $\pi^k$ is $O(\epsilon_k)$-optimal policy for the average MDP. We complete the proof of Lemma~\ref{lemma:unify}.
 
 \end{proof}

\subsection{Missing Lemmas and Proofs for Results in Section~\ref{sec:lemma_unify}} \label{sec:pf_lemma_unify_add}

\subsubsection{Proof of Lemma~\ref{lemma:as1}}\label{sec:lemma_as1}

\begin{proof}[Proof of Lemma~\ref{lemma:as1}]
First note that $V^{\pi^k}(\tilde{s})\leq V^*(\tilde{s})$ holds by definition.  Let $\Gamma$ denote the Bellman operator for the $\gamma$-discounted MDP. The LHS of \eqref{eq:csss} states that $V^k\geq \Gamma V^k.$ By the monotonicity of $\Gamma$, we have $V^k\geq \Gamma V^k \geq (\Gamma)^2 V^k\geq \cdots \geq (\Gamma)^nV^k$ for all $n\geq 1,$ which implies that $V^k\geq \lim_{n\rightarrow \infty} (\Gamma)^nV^k= V^*.$

By the RHS of \eqref{eq:csss}, for any $s\in \mathcal{S}$, we have that 
\begin{align}
    V^k(s)\leq \sum_{a}\pi^k(a|s)\left( r(s,a)+\gamma P_{s,a}V^k+\lambda^k(s,a) \right)\nonumber  
\end{align}
Recalling the definition of $\tilde{d}_{\gamma}^{k}(s,a|\tilde{s})$, we have that
\begin{align}
    V^k(\tilde{s})\leq \sum_{s,a} \tilde{d}_{\gamma}^{k}(s,a|\tilde{s}) (r(s,a)+\lambda^k(s,a)) = V^{\pi^k}(\tilde{s})+ \sum_{s,a}\tilde{d}^{k}_{\gamma}(s,a|\tilde{s})\lambda^k(s,a).\nonumber
\end{align}
The proof is completed by noting that $V^k \leq \frac{1}{1-\gamma}$.
\end{proof}

\subsubsection{Statement and Proof of Lemma~\ref{lemma:com1}}\label{sec:lemma_com1}

\begin{lemma}\label{lemma:com1} Recall that $\theta^{m-1}= \|V^{m-1}-V^*\|_{\infty}$. It then holds that
$$ \sum_{s,a}\tilde{d}_{\gamma}^{k}(s,a|\tilde{s})\mathbb{V}(P_{s,a},V^{m-1}) \leq 8\spn^2(h^*) + 5(\theta^{m-1})^2 +\frac{6\spn(h^*)}{1-\gamma} + 2\theta^{m-1}(V^*(\tilde{s})-V^{\pi^k}(\tilde{s}))$$

for $ 2 \leq m \leq k\leq K$ and any $\tilde{s}\in \mathcal{S}$.
\end{lemma}

\begin{proof} 
Recall that $\tilde{d}_{\gamma}^{\pi}(s,a|\tilde{s}) = \mathbb{I}[s=\tilde{s}]\pi(a|s) + \gamma \sum_{s',a'} \tilde{d}_{\gamma}(s',a'|\tilde{s}) P_{s',a',s} \pi(a|s).$ For any function $f:S\rightarrow \mathbb{R},$ 
\begin{align}
     \sum_{s,a} \tilde{d}_{\gamma}^{\pi}(s,a|\tilde{s})f(s) &= \sum_{s,a} \mathbb{I}[s=\tilde{s}]\pi(a|s) f(s) + \gamma \sum_{s,a}\sum_{s',a'} \tilde{d}_{\gamma}(s',a'|\tilde{s}) P_{s',a',s} \pi(a|s)f(s)\nonumber \\
    & = f(\tilde{s}) + \gamma \sum_{s',a'} \tilde{d}_{\gamma}(s',a'|\tilde{s}) P_{s',a'}f.   \label{eq:d_discount_sum_f}
\end{align}

Note that $\mathrm{Var}(A+B)\leq 2\mathrm{Var}(A)+2\mathrm{Var}(B)$ for any two random variables $A$ and $B$. 
\begin{align}
 & \sum_{s,a}\tilde{d}^k_{\gamma}(s,a|\tilde{s})\mathbb{V}(P_{s,a},V^{m-1})   \nonumber 
 \\ & \leq  2\sum_{s,a}\tilde{d}^k_{\gamma}(s,a|\tilde{s})\mathbb{V}(P_{s,a},V^{m-1}-V^*) + 2\sum_{s,a}\tilde{d}^k_{\gamma}(s,a|\tilde{s})\mathbb{V}(P_{s,a},V^*). \nonumber
 \end{align}
 Then we bound the two terms above separately as below. 
 \begin{align}
& \sum_{s,a}\tilde{d}^k_{\gamma}(s,a|\tilde{s})\mathbb{V}(P_{s,a},V^{m-1}-V^*) \nonumber
 \\ & = \sum_{s,a}\tilde{d}^k_{\gamma}(s,a|\tilde{s}) (P_{s,a}(V^{m-1}-V^*)^2 - (P_{s,a}V^{m-1}-P_{s,a}V^*)^2   ) \nonumber 
 \\ & \leq  \sum_{s,a}\tilde{d}^k_{\gamma}(s,a|\tilde{s}) (P_{s,a}(V^{m-1}-V^*)^2 - (V^{m-1}(s)-V^*(s))^2   ) \nonumber
 \\ & \quad + 2\|V^{m-1}-V^*\|_{\infty}\sum_{s,a}\tilde{d}^k_{\gamma}(s,a|\tilde{s})\max\left\{ V^{m-1}(s)-V^*(s) -  P_{s,a}(V^{m-1} - V^*) ,0 \right\}\nonumber 
 \\ & \stackrel{(i)}{=}  -(V^{m-1}(\tilde{s})-V^*(\tilde{s}))^2+(1-\gamma)\sum_{s,a}\tilde{d}^k_{\gamma}(s,a|\tilde{s}) P_{s,a}(V^{m-1}-V^*)^2 \nonumber
 \\& \quad+ 
 2\|V^{m-1}-V^*\|_{\infty}  \sum_{s,a}\tilde{d}_{\gamma}^k(s,a|\tilde{s})\max\left\{ V^{m-1}(s)-V^*(s) - \gamma P_{s,a}(V^{m-1} - V^*) ,0 \right\}\nonumber 
 \\ & \stackrel{(ii)}{\leq } (\theta^{m-1})^2+2\theta^{m-1} \sum_{s,a}\tilde{d}_{\gamma}^k(s,a|\tilde{s}) (V^{m-1}(s)-r(s,a) - \gamma P_{s,a}V^{m-1}) \nonumber
 \\ & \stackrel{(iii)}{=}  (\theta^{m-1})^2 +2\theta^{m-1}\left(V^{m-1}(\tilde{s}) - \sum_{s,a}\tilde{d}^k_{\gamma}(s,a|\tilde{s})r(s,a) \right) \nonumber 
\\ & \stackrel{(iv)}{\leq}  2(\theta^{m-1})^2+ 2\theta^{m-1}\left({V^*(\tilde{s})}-V^{\pi^k}(\tilde{s})\right), \label{eq:compute0}
\end{align}
where (i) and (iii) follow from \eqref{eq:d_discount_sum_f}, (ii) holds because $V^*(s')\geq r(s',a)+\gamma P_{s',a}V^*$ and $V^{m-1}(s)\geq r(s',a)+\gamma P_{s',a}V^{m-1}$ for any $s',a$, and (iv) holds by the fact that $V^{\pi^k}(\tilde{s})=\sum_{s,a}\tilde{d}^k_{\gamma}(s,a|\tilde{s})r(s,a)$.

Define $\underline{v}=\min_{s}V^*(s)$. Then we have that 
\begin{align}
 & \sum_{s,a}\tilde{d}_{\gamma}^{k}(s,a|\tilde{s})\mathbb{V}(P_{s,a},V^{*})\nonumber \\
 & =\sum_{s,a}\tilde{d}_{\gamma}^{k}(s,a|\tilde{s})\left(P_{s,a}(V^{*}-\underline{v})^{2}-(P_{s,a}V^{*}-\underline{v})^{2}\right)\nonumber \\
 & \leq\sum_{s,a}\tilde{d}_{\gamma}^{k}(s,a|\tilde{s})(P_{s,a}(V^{*}-\underline{v})^{2}-(V^{*}(s)-\underline{v})^{2})+\spn(V^{*})\sum_{s,a}\tilde{d}_{\gamma}^{k}(s,a|\tilde{s})\cdot\max\{V^{*}(s)-P_{s,a}V^{*},0\}\nonumber \\
 & \stackrel{(i)}{\leq}-(V^{*}(\tilde{s})-\underline{v})^{2}+(1-\gamma)\sum_{s,a}\tilde{d}_{\gamma}^{k}(s,a|\tilde{s})P_{s,a}(V^{*}-\underline{v})^{2}+\spn(V^{*})\sum_{s,a}\tilde{d}_{\gamma}^{k}(s,a|\tilde{s})(V^{*}(s)-\gamma P_{s,a}V^{*})\nonumber \\
 & \stackrel{(ii)}{\leq}\spn^{2}(V^{*})+\spn(V^{*})V^{*}(\tilde{s})\nonumber \\
 & \leq\spn^{2}(V^{*})+\frac{\spn(V^{*})}{1-\gamma}\nonumber \\
 & \stackrel{(iii)}{\leq}4\spn^{2}(h^{*})+\frac{2\spn(h^{*})}{1-\gamma},\label{eq:compute1}
\end{align}
where (i) and (ii) follows from \eqref{eq:d_discount_sum_f}, (iii) holds by Lemma~\ref{lemma:discount_approximate}.

The proof is completed by combining \eqref{eq:compute0} and \eqref{eq:compute1}.
\end{proof}

\subsubsection{Statement and Proof of Lemma~\ref{lemma:com2}}\label{sec:lemma_com2}

\begin{lemma}\label{lemma:com2} Recall that $\theta^{m-1}=\|V^{m-1}-V^*\|_{\infty}$. It then holds that
$$\sum_{s,a}\tilde{d}^k_{\gamma}(s,a|\tilde{s})\mathbb{V}(P_{s,a},V^{m-1}-V^m)\leq 4(\theta^{m-1})^2 + 2\theta^{m-1}(V^*(\tilde{s})-V^{\pi^k}(\tilde{s})) $$ 
for $2\leq m \leq k \leq K$ and any $\tilde{s}\in \mathcal{S}$.
\end{lemma}

\begin{proof}
Direct computation gives that 
\begin{align*}
 & \sum_{s,a}\tilde{d}_{\gamma}^{k}(s,a|\tilde{s})\mathbb{V}(P_{s,a},V^{m-1}-V^{m})\\
 & =\sum_{s,a}\tilde{d}_{\gamma}^{k}(s,a|\tilde{s})\left(P_{s,a}(V^{m-1}-V^{m})^{2}-(P_{s,a}(V^{m-1}-V^{m}))^{2}\right)\\
 & =\sum_{s,a}\tilde{d}_{\gamma}^{k}(s,a|\tilde{s})\left(P_{s,a}(V^{m-1}-V^{m})^{2}-(V^{m-1}(s)-V^{m}(s))^{2}+(V^{m-1}(s)-V^{m}(s))^{2}-(P_{s,a}(V^{m-1}-V^{m}))^{2}\right)\\
 & \stackrel{(i)}{\leq}-(V^{m-1}(\tilde{s})-V^{m}(\tilde{s}))^{2}+(1-\gamma)\sum_{s,a}\tilde{d}_{\gamma}^{k}(s,a|\tilde{s})P_{s,a}(V^{m-1}-V^{m})^{2}\\
 & \qquad+\left\Vert V^{m-1}-V^{m}\right\Vert _{\infty}\sum_{s,a}\tilde{d}_{\gamma}^{k}(s,a|\tilde{s})\max\{V^{m-1}(s)-V^{m}(s)-P_{s,a}(V^{m-1}-V^{m}),0\}\\
 & \leq(\theta^{m-1})^{2}+\theta^{m-1}\sum_{s,a}\tilde{d}_{\gamma}^{k}(s,a|\tilde{s})\max\{V^{m-1}(s)-V^{m}(s)-P_{s,a}(V^{m-1}-V^{m}),0\}\\
 & \stackrel{(ii)}{\leq}(\theta^{m-1})^{2}+\theta^{m-1}\sum_{s,a}\tilde{d}_{\gamma}^{k}(s,a|\tilde{s})\max\{V^{m-1}(s)-V^{m}(s)-\gamma P_{s,a}(V^{m-1}-V^{m}),0\}\\
 & \stackrel{(iii)}{\leq}(\theta^{m-1})^{2}+\theta^{m-1}\sum_{s,a}\tilde{d}_{\gamma}^{k}(s,a|\tilde{s})\max\{V^{m-1}(s)-r(s,a)-\gamma P_{s,a}V^{m-1},0\}\\
 & \leq(\theta^{m-1})^{2}+\theta^{m-1}\sum_{s,a}\tilde{d}_{\gamma}^{k}(s,a|\tilde{s})(V^{m-1}(s)-r(s,a)-\gamma P_{s,a}V^{m-1})\\
 & \stackrel{(iv)}{\leq}(\theta^{m-1})^{2}+\theta^{m-1}\left(V^{m-1}(\tilde{s})-\sum_{s,a}\tilde{d}_{\gamma}^{k}(s,a|\tilde{s})r(s,a)\right)\\
 & \leq2(\theta^{m-1})^{2}+\theta^{m-1}(V^{*}(\tilde{s})-V^{\pi^{k}}(\tilde{s})).
\end{align*}
Here $(i)$ and $(iv)$ follow from \eqref{eq:d_discount_sum_f}, $(ii)$ holds because $V^{m-1}\geq V^m$, and $(iii)$ hold because $V^{m'}(s)\geq r(s,a)+ \gamma P_{s,a}V^{m'}$ for $ 1\leq m'\leq K$ and any $s\in \mathcal{S}$.

\end{proof}

\subsubsection{Proof of Lemma~\ref{lemma:as2}} \label{sec:lemma_as2}

\begin{proof}[Proof of Lemma~\ref{lemma:as2}] From the LHS of \eqref{eq:www3} and Lemma~\ref{lemma:discount_approximate}, we have 
$V^k(s)\geq V^*(s)\geq \frac{\rho^*}{1-\gamma}-\spn(h^*)$. Together with the RHS of \eqref{eq:www3}, we have 

\begin{align}
    V_{T}^{\pi^k}(\tilde{s})  & = \sum_{s,a}d_{T}^k(s,a|\tilde{s})r(s,a) \nonumber
    \\ & \geq \sum_{s,a}d_{T}^k(s,a|\tilde{s})(V^k(s)-\gamma P_{s,a}V^k - \lambda^k(s,a)) \nonumber
    \\ & \stackrel{(a)}{\geq} (1-\gamma)\sum_{s,a}d_{T}^k(s,a)V^k(s) +\gamma (V^k(\tilde{s}) - \mathbb{E}_{\pi^k}[V^k(s_{T+1})|s_1=\tilde{s}] ) - \sum_{s,a}d_{T}^k(s,a|\tilde{s})\lambda^k(s,a) \nonumber
    \\ & \geq  T\rho^* -T(1-\gamma)\spn(h^*) - \spn(V^k)- \sum_{s,a}d_{T}^k(s,a|\tilde{s})\lambda^k(s,a) \nonumber
    \\& \geq  T\rho^* -T(1-\gamma)\spn(h^*) -\spn(V^*) - \|V^k-V^*\|_{\infty} - \sum_{s,a}d_{T}^k(s,a|\tilde{s})\lambda^k(s,a) .\nonumber
    \\ & \stackrel{(b)}{\geq} T\rho^* -T(1-\gamma)\spn(h^*) -2\spn(h^*) - \theta^{k} - \sum_{s,a}d_{T}^k(s,a|\tilde{s})\lambda^k(s,a),\nonumber
\end{align}
where $(a)$ follows from the definition of $d_{T}^k(s,a|\tilde{s})$ in \eqref{eq:def_d_finite}, $(b)$ follows from the fact $\spn(V^*)\leq 2\spn(V^*)$ and the definition of $\theta^k$.

Recall that $\theta^k\leq \frac{c_5\epsilon_k}{1-\gamma}$ by the RHS of \eqref{eq:www3}.  Together with Lemma~\ref{lemma:finite_approximate}, we obtain that 
\begin{align}
    V_{T}^*(\tilde{s}) & \leq T\rho^* +\spn(h^*) \nonumber
    \\ & \leq  V_{T}^{\pi^k}(\tilde{s})   + T(1-\gamma)\spn(h^*) +2\spn(h^*) +\frac{c_5\epsilon_k}{1-\gamma}+ \sum_{s,a}d_{T}^k(s,a|\tilde{s})\lambda^k(s,a).\label{eq:us1}
\end{align}

We complete the proof of \eqref{eq:561}. 

To prove \eqref{eq:562}, it suffices to bound $\sum_{s,a}d^{k}_{T}(s,a|\tilde{s})\lambda^k(s,a)$. Define $\Lambda_{m}^k(\tilde{s}) = \sum_{s,a}d^k_{T}(s,a|\tilde{s})\lambda^m(s,a)$ and $\overline{\Lambda}_m^k = \max_{\tilde{s}}\Lambda_m^k(\tilde{s})$ for $1\leq m\leq k\leq K$. By the definition of $\lambda^k(s,a)$, we have that
\begin{align}
    \Lambda_k^k(\tilde{s}) & \leq c_1\underbrace{ \sum_{s,a}d_{T}^{k}(s,a|\tilde{s}) \sqrt{\frac{ \mathbb{V}(P_{s,a},V^{k-1})\iota}{T}}}_{\mathbf{Term.5}} + c_2 \underbrace{\sum_{s,a}d_{T}^{k}(s,a|\tilde{s})\sqrt{\frac{\mathbb{V}(P_{s,a},V^{k-1}-V^{k})\iota}{T_1}} }_{\mathbf{Term.6}} \nonumber
    \\ & \quad \quad + c_3 \underbrace{\sum_{s,a}d_{T}^{k}(s,a|\tilde{s}) \frac{(\|V^{k-1}-V^k\|_{\infty}+\|V^{k-1}-V^*\|_{\infty})\iota}{T_1} }_{\mathbf{Term.7}} +c_4\underbrace{\sum_{s,a}d_{T}^{k}(s,a|\tilde{s})\epsilon_k }_{\mathbf{Term.8}}.\label{eq:decomposition1} 
\end{align}

Using Cauchy's inequality, and noting that $\sum_{s,a}d_{T}^k(s,a|\tilde{s})=T$, we have that
\begin{align}
    \mathbf{Term.5}& \leq \sqrt{\frac{\sum_{s,a}d_{T}^k(s,a|\tilde{s})}{T}}\cdot \sqrt{\sum_{s,a}d_{T}^k(s,a|\tilde{s})\mathbb{V}(P_{s,a},V^{k-1})\iota} \nonumber
    \\ &= \sqrt{\sum_{s,a}d_{T}^k(s,a|\tilde{s})\mathbb{V}(P_{s,a},V^{k-1})\iota}.\nonumber
\end{align}
Using Lemma~\ref{lemma:bdvariance1} (statement and proof deferred to Section~\ref{sec:lemma_bdvariance1}), and noting that $T(1-\gamma)\geq 1$, we have 
\begin{align}
&\mathbf{Term.5}\nonumber
\\& \leq         \sqrt{\left(4\theta^{k-1}\left(V^{*}_{T}(\tilde{s}) - V^{\pi^k}_{T}(\tilde{s}) \right) +6T(1-\gamma)(\theta^{k-1})^2 + 8T(1-\gamma)(\theta^{k-1}+ 4\spn^2(h^*))  +  4\spn(h^*)T\rho^*\right)\iota} \nonumber
\\ &  \stackrel{(a)}{\leq} \sqrt{\left( 4\theta^{k-1}\Lambda_k^k(\tilde{s}) + 6T(1-\gamma)(\theta^{k-1})^2   +\frac{2c_5\theta^{k-1}\epsilon_k}{1-\gamma}+ 8T(1-\gamma)(\theta^{k-1}+4 \spn^2(h^*))  +  4\spn(h^*)T\rho^*\right)\iota} \nonumber
\\ & \stackrel{(b)}{\leq} \sqrt{ \frac{8c_5\epsilon_k}{1-\gamma}\Lambda_k^k(\tilde{s}) + \frac{24c^2_5T\epsilon_k^2}{1-\gamma}+\frac{20c_5^2\epsilon_k^2}{(1-\gamma)^2}
+4Tc_5\epsilon_k+8c_5T\epsilon_k\spn(h^*)+8T(1-\gamma)\spn^2(h^*)+4T\rho^*\spn(h^*)  }\cdot \sqrt{\iota},\label{eq:term5}
\end{align}
where $(a)$ holds due to \eqref{eq:us1} and the definition of $\Lambda_k^k$,  $(b)$ follows from $\theta^{k-1}\leq \frac{c_5\epsilon_{k-1}}{1-\gamma}=\frac{2c_5\epsilon_{k}}{1-\gamma}$ by \eqref{eq:www3}.

Similarly, we have that
\begin{align}
    \mathbf{Term.6}& \leq \sqrt{\frac{\sum_{s,a}d_{T}^k(s,a|\tilde{s})\iota}{T_1}}\cdot \sqrt{\sum_{s,a}d_{T}^k(s,a|\tilde{s})\mathbb{V}(P_{s,a},V^{k-1}-V^k)}\nonumber
    \\ & = \sqrt{\frac{T\iota}{T_1}}\cdot \sqrt{\sum_{s,a}d_{T}^k(s,a|\tilde{s})\mathbb{V}(P_{s,a},V^{k-1}-V^k)}.\nonumber
\end{align}
Using Lemma~\ref{lemma:bdvariance2} (statement and proof deferred to Section~\ref{sec:lemma_bdvariance2}), we further have that
\begin{align}
 & \mathbf{Term.6}\nonumber \\
 & \leq\sqrt{\frac{T\iota}{T_{1}}}\cdot\sqrt{2(\theta^{k-1})^{2}+2\theta^{k-1}\spn(h^{*})+2\theta^{k-1}T(1-\gamma)(\theta^{k-1}+\spn(h^{*}))+2\theta^{k-1}\left(V_{T}^{*}(\tilde{s})-V_{T}^{\pi^{k}}(\tilde{s})\right)}\nonumber \\
 & \stackrel{(a)}{\leq}\sqrt{\frac{T\iota}{T_{1}}}\cdot\sqrt{2(\theta^{k-1})^{2}+6\theta^{k-1}\spn(h^{*})+2\theta^{k-1}T(1-\gamma)(\theta^{k-1}+2\spn(h^{*}))+2\theta^{k-1}\left(\Lambda_{k}^{k}(\tilde{s})+\frac{c_{5}\epsilon_{k}}{1-\gamma}\right)}\nonumber \\
 & \leq\sqrt{\frac{T\iota}{T_{1}}}\cdot\sqrt{5(\theta^{k-1})^{2}+3\spn^{2}(h^{*})+2\theta^{k-1}T(1-\gamma)(\theta^{k-1}+2\spn(h^{*}))+2\theta^{k-1}\left(\Lambda_{k}^{k}(\tilde{s})+\frac{c_{5}\epsilon_{k}}{1-\gamma}\right)}\nonumber \\
 & \leq\sqrt{\frac{T\iota}{T_{1}}}\cdot\sqrt{\frac{20c_{5}^{2}\epsilon_{k}^{2}}{(1-\gamma)^{2}}+3\spn^{2}(h^{*})+2T\frac{4c_{5}^{2}\epsilon_{k}^{2}}{(1-\gamma)}+2T\cdot2c_{5}\epsilon_{k}(2\spn(h^{*}))+\frac{4c_{5}\epsilon_{k}\Lambda_{k}^{k}(\tilde{s})}{1-\gamma}+\frac{4c_{5}^{2}\epsilon_{k}^{2}}{(1-\gamma)^{2}}}\nonumber \\
 & \leq\sqrt{\frac{T\iota}{T_{1}}}\cdot\sqrt{\frac{24c_{5}^{2}\epsilon_{k}^{2}}{(1-\gamma)^{2}}+3\spn^{2}(h^{*})+\frac{8c_{5}^{2}\epsilon_{k}^{2}T}{(1-\gamma)}+8c_{5}T\spn(h^{*})\epsilon_{k}+\frac{4c_{5}\epsilon_{k}\Lambda_{k}^{k}(\tilde{s})}{1-\gamma}}.\label{eq:term6}
\end{align}

For $\mathbf{Term.7}$ and $\mathbf{Term.8}$, noting that $\sum_{s,a}d_{T}^k(s,a|\tilde{s})=T$, we have that
\begin{align}
    & \mathbf{Term.7}\leq \frac{4c_5T\epsilon_k\iota}{T_1(1-\gamma)} \label{eq:term7}
    \\ & \mathbf{Term.8}\leq T\epsilon_k.\label{eq:term8}
\end{align}

Combining inequalities \eqref{eq:term5}-\eqref{eq:term8} for \eqref{eq:decomposition1}, and solving the inequality for $\Lambda_{k}^k(\tilde{s})$, we have that
 \begin{align}
     \Lambda_k^k(\tilde{s}) & \leq O\left( \frac{\epsilon_k\iota}{1-\gamma}+ \frac{T\iota \epsilon_k}{T_1(1-\gamma)} \right) \nonumber
     \\&\quad  + O\left( \sqrt{\iota}\sqrt{   \frac{\epsilon_k^2}{(1-\gamma)^2} + \frac{T\epsilon_k^2}{1-\gamma}+\spn(h^*)T\epsilon_k +T(1-\gamma)\blue{\spn^2(h^*)}+Tc_5\epsilon_k+T\rho^*\spn(h^*)      } \right)\nonumber
     \\ &\quad  + O\left( \sqrt{\frac{T\iota}{T_1}}\sqrt{  \frac{\epsilon_k^2}{(1-\gamma)^2} + \spn^2(h^*)
 + \frac{\epsilon_k^2T}{(1-\gamma)}+ T\spn(h^*)\epsilon_k   } \right)\nonumber
 \\ & \quad  + O\left(\frac{T\epsilon_k\iota}{T_1(1-\gamma)}+T\epsilon_k \right).\nonumber
 \end{align}
By the condition $T_1(1-\gamma)\geq 320(c_1+c_2)^2\iota $, and the fact $T\geq \max\left\{\frac{\sqrt{\iota}}{1-\gamma},\spn(h^*)\iota\right\}$ and $\epsilon_k \geq \Omega((1-\gamma)\spn(h^*))$ for $1\leq k \leq K$, we have thatt

\begin{align}
    \Lambda_k^k(\tilde{s})\leq O(T\epsilon_k).\label{eq:final}
\end{align}
 The proof is completed.
 
 \end{proof}

\subsubsection{Statement and Proof of Lemma~\ref{lemma:bdvariance1}} \label{sec:lemma_bdvariance1}

\begin{lemma}\label{lemma:bdvariance1}
    Assume the conditions in Lemma~\ref{lemma:unify}. For $1\leq m \leq k \leq K$, it  holds that 
    \begin{align}
    & \sum_{s,a}d_{T}^k(s,a|\tilde{s}) \mathbb{V}(P_{s,a}, V^{m-1})\nonumber
    \\ & \leq 2T(1-\gamma)(\theta^{m-1})^2 +3(\theta^{m-1})^2 +   2\theta^{m-1}\spn(h^*) + 4\theta^{m-1}\left(V^{*}_{T}(\tilde{s}) - V^{\pi^k}_{T}(\tilde{s}) \right) \nonumber
    \\ & \quad \quad \quad \quad + 8T(1-\gamma)(\theta^{m-1}+ 4\spn^2(h^*))  +  4\spn(h^*)T\rho^*. \nonumber 
    \end{align}
\end{lemma}
\begin{proof}[Proof of Lemma~\ref{lemma:bdvariance1}]
By definition of $d_{T}^k(\cdot,\cdot|\tilde{s})$, we have that
\begin{align}
    \sum_{s,a}d_{T}^k(s,a|\tilde{s})(\textbf{1}_{s}-P_{s,a}) = \mathbf{1}_{\tilde{s}} - \mu_{T+1}^{k}(\tilde{s}),\label{eq:wxq1}
\end{align}
where $\mu_{T+1}^k(\tilde{s})$ denote the state distribution of $s_{T+1}$ following $\pi^k$ with initial state $\tilde{s}$.

Noting that $\mathbb{V}(P_{s,a},V^{m-1})\leq 2 \mathbb{V}(P_{s,a},V^{m-1}-V^*) + 2 \mathbb{V}(P_{s,a},V^*)$, it suffices to bound these two terms. 
Direct computation gives that 
\begin{align}
 & \sum_{s,a}d_{T}^{k}(s,a|\tilde{s})\mathbb{V}(P_{s,a},V^{m-1}-V^{*})\nonumber \\
 & =\sum_{s,a}d_{T}^{k}(s,a|\tilde{s})\left(P_{s,a}(V^{m-1}-V^{*})^{2}-(P_{s,a}V^{m-1}-P_{s,a}V^{*})^{2}\right)\nonumber \\
 & =\sum_{s,a}d_{T}^{k}(s,a|\tilde{s})\left(P_{s,a}(V^{m-1}-V^{*})^{2}-(V^{m-1}(s)-V^{*}(s))^{2}+(V^{m-1}(s)-V^{*}(s))^{2}-(P_{s,a}V^{m-1}-P_{s,a}V^{*})^{2}\right)\nonumber \\
 & \stackrel{(i)}{\leq}(\theta^{m-1})^{2}+2\theta^{m-1}\sum_{s,a}d_{T}^{k}(s,a|\tilde{s})\max\{V^{m-1}(s)-V^{*}(s)-P_{s,a}(V^{m-1}-V^{*}),0\}\nonumber \\
 & \stackrel{(ii)}{\leq}(\theta^{m-1})^{2}+2\theta^{m-1}\sum_{s,a}d_{T}^{k}(s,a|\tilde{s})\max\{V^{m-1}(s)-P_{s,a}V^{m-1}-r(s,a)+(1-\gamma)P_{s,a}V^{*},0\}\nonumber \\
 & \stackrel{(iii)}{\leq }(\theta^{m-1})^{2}+2\theta^{m-1}\sum_{s,a}d_{T}^{k}(s,a|\tilde{s})\left(V^{m-1}(s)-\gamma P_{s,a}V^{m-1}-r(s,a)\right)\nonumber \\
 & \stackrel{(iv)}{\leq}(\theta^{m-1})^{2}+2T(1-\gamma)(\theta^{m-1})^2+2\theta^{m-1}\left(V^{m-1}(\tilde{s})-\mathbb{E}_{\pi^{k}}[V^{m-1}(s_{T+1})|s_{1}=\tilde{s}]-\sum_{s,a}d_{T}^{k}(s,a|\tilde{s})r(s,a))\right)\nonumber \\
 & \stackrel{(v)}{\leq}3(\theta^{m-1})^{2}+2T(1-\gamma)(\theta^{m-1})^2+4\theta^{m-1}\left(\spn(h^{*})+V_{T}^{*}(\tilde{s})-V_{T}^{\pi^{k}}(\tilde{s})\right)+2\theta^{m-1}T(1-\gamma)\spn(h^{*}) \nonumber \\
 & \stackrel{(vi)}{\leq}3(\theta^{m-1})^{2}+2T(1-\gamma)(\theta^{m-1})^2+4\theta^{m-1}\left(V_{T}^{*}(\tilde{s})-V_{T}^{\pi^{k}}(\tilde{s})\right)+4\theta^{m-1}\left(T(1-\gamma)+3\right)\spn(h^{*}). \label{eq:ww01}
\end{align}
Here $(i)$ holds by \eqref{eq:wxq1}, $(ii)$ holds because $V^*(s)\geq r(s,a)+\gamma P_{s,a}V^*$ for all $a\in \mA$, $(iii)$ is true due to the fact that $V^{m-1}(s)-r(s,a)-\gamma P_{s,a}V^{m-1}\geq 0$ for all $a\in \mathcal{A}$ and $V^* \leq V^{m-1}$,  $(iv)$ holds by \eqref{eq:wxq1} and  the fact that $V^*(s)\leq \frac{\rho^*}{1-\gamma} + \spn(h^*)$ (Lemma~\ref{lemma:discount_approximate}), $(v)$ follows as $\spn(V^{m-1})\leq \spn(\|V^{m-1}-V^*\|_{\infty})+2\spn(h^*)$ and $T\rho^*\leq V^*_T(\tilde{s})+\spn(h^*)$ (Lemma~\ref{lemma:finite_approximate}), and $(vi)$ is true because $\spn(V^*)\leq 2\spn(h^*)$ (Lemma~\ref{lemma:discount_approximate}).

 On the other hand, recalling $\underline{v}= \min_{s}V^*(s)$, using similar arguments we obtain that 
 \begin{align}
      & \sum_{s,a} d_{T}^k(s,a|\tilde{s}) \mathbb{V}(P_{s,a},V^*) \nonumber 
      \\ &  = \sum_{s,a}d_{T}^k(s,a|\tilde{s}) \left( P_{s,a}(V^* -\underline{v} \cdot \textbf{1}) )^2  - (P_{s,a}V^* - \underline{v})^2\right)\nonumber
      \\ &  =\sum_{s,a}d_{T}^k(s,a|\tilde{s})\left( P_{s,a} (V^* -\underline{v} \cdot \textbf{1}) )^2 - (V^*(s)-\underline{v})^2 +   (V^*(s)-\underline{v})^2  - (P_{s,a}V^* - \underline{v})^2 \right) \nonumber
      \\ & \leq 4\spn^2(h^*) + 4\spn(h^*) \sum_{s,a}d_{T}^k(s,a|\tilde{s}) \max\{ V^*(s) - P_{s,a}V^*       ,0\} \nonumber
      \\ & \leq 4\spn^2(h^*) + 4\spn(h^*) \sum_{s,a}d_{T}^k(s,a|\tilde{s}) \max\{ V^*(s) -r(s,a)-\gamma P_{s,a}V^*+r(s,a)        ,0\} \nonumber
      \\ & = 4\spn^2(h^*) + 4\spn(h^*) \sum_{s,a}d_{T}^k(s,a|\tilde{s}) \left( V^*(s) -\gamma P_{s,a}V^* \right) \nonumber
      \\ & \leq  4\spn^2(h^*) + 4\spn(h^*)T(1-\gamma)\left(\frac{\rho^*}{1-\gamma}+\spn(h^*)\right) + 4\spn^2(h^*)\nonumber
      \\ & \leq 8\spn^2(h^*)+ 4\spn(h^*)T\rho^* + 4\spn^2(h^*)T(1-\gamma).\nonumber
  \end{align}
The proof is completed.

\end{proof}

\subsubsection{Statement and Proof of Lemma~\ref{lemma:bdvariance2}}\label{sec:lemma_bdvariance2}

\begin{lemma}\label{lemma:bdvariance2}
Assume the conditions in Lemma~\ref{lemma:unify}. Then the following holds 
for $1\leq m\leq  k \leq K$:
\begin{align}
   &   \sum_{s,a}d_{T}^k(s,a|\tilde{s})\mathbb{V}(P_{s,a},V^{m-1}-V^m)\nonumber
   \\ & \leq \left(2T(1-\gamma)+3\right)(\theta^{m-1})^{2}+2\theta^{m-1}\left(T(1-\gamma)+3\right)\spn(h^{*})+2\theta^{m-1}\left(V_{T}^{*}(\tilde{s})-V_{T}^{\pi^{k}}(\tilde{s})\right). \nonumber
\end{align}

\end{lemma}

\begin{proof}[Proof of Lemma~\ref{lemma:bdvariance2}]
Direct computation gives that 
\begin{align*}
 & \sum_{s,a}d_{T}^{k}(s,a|\tilde{s})\mathbb{V}(P_{s,a},V^{m-1}-V^{m})\\
 & =\sum_{s,a}d_{T}^{k}(s,a|\tilde{s})\left(P_{s,a}(V^{m-1}-V^{m})^{2}-(P_{s,a}(V^{m-1}-V^{m}))^{2}\right)\\
 & =\sum_{s,a}d_{T}^{k}(s,a|\tilde{s})\left(P_{s,a}(V^{m-1}-V^{m})^{2}-(V^{m-1}(s)-V^{m}(s))^{2}+(V^{m-1}(s)-V^{m}(s))^{2}-(P_{s,a}(V^{m-1}-V^{m}))^{2}\right)\\
 & \leq(\theta^{m-1})^{2}+2\theta^{m-1}\sum_{s,a}d_{T}^{k}(s,a|\tilde{s})\max\left\{ V^{m-1}(s)-V^{m}(s)-P_{s,a}(V^{m-1}-V^{m}),0\right\} \\
 & \stackrel{(i)}{\leq}(\theta^{m-1})^{2}+2\theta^{m-1}\sum_{s,a}d_{T}^{k}(s,a|\tilde{s})\max\left\{ V^{m-1}(s)-V^{m}(s)-\gamma P_{s,a}(V^{m-1}-V^{m}),0\right\} \\
 & \stackrel{(ii)}{\leq}(\theta^{m-1})^{2}+2\theta^{m-1}\sum_{s,a}d_{T}^{k}(s,a|\tilde{s})\max\left\{ V^{m-1}(s)-r(s,a)-\gamma P_{s,a}V^{m-1},0\right\} \\
 & \stackrel{(iii)}{=}(\theta^{m-1})^{2}+2\theta^{m-1}\sum_{s,a}d_{T}^{k}(s,a|\tilde{s})\left(V^{m-1}(s)-r(s,a)-\gamma P_{s,a}V^{m-1}\right)\\
 & =(\theta^{m-1})^{2}+2\theta^{m-1}\sum_{s,a}d_{T}^{k}(s,a|\tilde{s})\left(V^{m-1}(s)-P_{s,a}V^{m-1}+(1-\gamma)P_{s,a}V^{m-1}-r(s,a)\right)\\
 & \stackrel{\text{(iv)}}{\leq}(\theta^{m-1})^{2}+2\theta^{m-1}\spn(V^{m-1})+2\theta^{m-1}\sum_{s,a}d_{T}^{k}(s,a|\tilde{s})\left((1-\gamma)P_{s,a}V^{m-1}-r(s,a)\right)\\
 & \stackrel{(v)}{\leq}(\theta^{m-1})^{2}+2\theta^{m-1}(\theta^{m-1}+\spn(V^{*}))+2\theta^{m-1}(1-\gamma)T\left(\theta^{m-1}+\frac{\rho^{*}}{1-\gamma}+\spn(h^{*})\right)-2\theta^{m-1}V_{T}^{\pi^{k}}(\tilde{s})\\
 & \stackrel{(vi)}{\leq}(\theta^{m-1})^{2}+2\theta^{m-1}(\theta^{m-1}+3\spn(h^{*}))+2\theta^{m-1}T(1-\gamma)(\theta^{m-1}+\spn(h^{*}))+2\theta^{m-1}\left(V_{T}^{*}(\tilde{s})-V_{T}^{\pi^{k}}(\tilde{s})\right)\\
 & =\left(2T(1-\gamma)+3\right)(\theta^{m-1})^{2}+2\theta^{m-1}\left(T(1-\gamma)+3\right)\spn(h^{*})+2\theta^{m-1}\left(V_{T}^{*}(\tilde{s})-V_{T}^{\pi^{k}}(\tilde{s})\right).
\end{align*}
Here $(i)$ holds because $V^{m-1}-V^m\geq 0$, $(ii)$ holds because $V^m(s)\geq \max_{a} (r(s,a)+\gamma P_{s,a}V^m)$, $(iii)$ follows as $V^{m-1}(s) - r(s,a)-\gamma P_{s,a}V^{m-1}\geq 0$, $(iv)$ follows from \eqref{eq:wxq1}, $(v)$ is true as $V^{m-1}\leq \theta^{m-1}+V^*$ and $V^*(s)\leq \frac{\rho^*}{1-\gamma}+\spn(h^*)$ (Lemma~\ref{lemma:discount_approximate}), and $(vi)$ follows from $T\rho^* \leq V_{T}^*(\tilde{s})+\spn(h^*)$ (Lemma~\ref{lemma:finite_approximate}). The proof is completed.

\end{proof}

\end{document}